%% file: main.tex
\documentclass[11pt]{article}

\usepackage[letterpaper,margin=1in]{geometry}
\usepackage[utf8]{inputenc}
\DeclareUnicodeCharacter{0302}{\^{}}

\usepackage{amsmath,amssymb,amsfonts,amsthm}
\usepackage{mathtools}
\usepackage{bm}
\usepackage{bbm}
\usepackage{mathrsfs}

\usepackage{graphicx}
\usepackage{float}
\usepackage{subcaption}
\usepackage{booktabs}
\usepackage{multirow}
\usepackage{tabularx} 
\usepackage{array}

\usepackage{algorithm}
\usepackage{algorithmicx}
\usepackage{algpseudocode}

\usepackage{xspace}
\usepackage{xcolor}
\usepackage{soul}
\usepackage{cancel}

\usepackage{url}
\usepackage[colorlinks=true,
            linkcolor=blue,
            urlcolor=blue,
            citecolor=blue]{hyperref}

\usepackage[numbers,sort&compress]{natbib}

\usepackage{lineno}

\usepackage{bookmark} 
\newcommand{\Prob}{\mathbb{P}}
\newcommand{\Pm}{\mathbb{P}_{\mathrm{MFM}}}
\newcommand{\Po}{\mathbb{P}_{\mathrm{OA}}}

\newcommand{\aoamlong}{any-order autoregressive model\xspace}
\newcommand{\aoamlongs}{any-order autoregressive models\xspace}

\newcommand{\aoarlong}{any-order autoregressive\xspace}
\newcommand{\aoamshort}{AO-ARM\xspace}
\newcommand{\aoamshorts}{AO-ARMs\xspace}

\newcommand{\sftshort}{FT\xspace}

\newcommand{\eg}{\textit{e.g.}}
\newcommand{\ie}{\textit{i.e.}}

\newcommand{\ddg}{\Delta \Delta G}

\newcommand{\ourmethod}{\textsc{ProteinGuide}\xspace}
\newcommand{\ourmethodacro}{\textsc{ProteinGuide}\xspace}
\newcommand{\ourmethodICLR}{\textsc{Discrete Guidance}\xspace}
\newcommand{\ournewsamplingmethod}{DEG\xspace}
\newcommand{\tref}[1]{Table~\ref{#1}}
\newcommand{\fref}[1]{Fig.~\ref{#1}}
\newcommand{\supptref}[1]{Supplementary Table~\ref{#1}}
\newcommand{\suppfref}[1]{Supplementary Figure~\ref{#1}}

\theoremstyle{plain}
\newtheorem{theorem}{Theorem}
\newtheorem{proposition}[theorem]{Proposition}
\newtheorem{lemma}{Lemma}
\newtheorem{corollary}{Corollary}

\theoremstyle{definition}

\theoremstyle{remark}

\def\scititle{ProteinGuide: On-the-fly property guidance for protein sequence generative models}
\title{\scititle}

\author{
    Junhao~Xiong$^{1,\ast}$,
    Ishan~Gaur$^{1,\ast}$,
    Maria~Lukarska$^{3,4,5,\ast}$,
    Hunter~Nisonoff$^{1,\ast}$,\and
    Luke~M.~Oltrogge$^{3,4,5}$,
    David~F.~Savage$^{3,4,5,\dagger}$,
    Jennifer~Listgarten$^{1,2,\dagger}$\\
\and
\small$^{1}$Department of Electrical Engineering and Computer Sciences, University of California, Berkeley,\and
{\small Berkeley, CA 94720, USA}\and
\small$^{2}$Center for Computational Biology, University of California, Berkeley, Berkeley, CA 94720, USA\and
\small$^{3}$Department of Molecular and Cell Biology, University of California, Berkeley, Berkeley, CA 94720, USA\and
\small$^{4}$Innovative Genomics Institute, University of California, Berkeley, Berkeley, CA 94720, USA\and
\small$^{5}$Howard Hughes Medical Institute, University of California, Berkeley, CA 94720, USA\\
\and
\small$^{\ast}$These authors contributed equally to this work.\and
\small$^{\dagger}$Corresponding authors. Email: \url{savage@berkeley.edu}, \url{jennl@berkeley.edu}
}
\date{\today}

\begin{document}

\maketitle


\begin{abstract}
Sequence generative models are transforming protein engineering. However, no principled framework exists for conditioning these models on auxiliary information, such as experimental data, without additional training of a generative model. Herein, we present \mbox{\ourmethod}, a method for such ``on-the-fly'' conditioning, amenable to a broad class of protein generative models including Masked Language Models (\eg \ ESM3), any-order auto-regressive models (\eg \ ProteinMPNN) as well as 
diffusion and flow matching models (\eg \ MultiFlow). 
\ourmethod stems from our unifying view of these model classes under a single statistical framework. As proof of principle, we perform several {\it in silico} experiments. We first guide pre-trained generative models to design proteins with user-specified properties, such as higher stability or activity. Next, we design for optimizing two desired properties that are in tension with each other. Finally, we apply our method in the wet lab, using \ourmethod to increase the editing activity of an adenine base editor \textit{in vivo} with data from only a single pooled library of 2,000 variants. We find that a single round of \ourmethod achieves a higher editing efficiency than was previously achieved using seven rounds of directed evolution.

\end{abstract}

\newpage

Generative models are enabling protein engineers to create bespoke proteins for therapeutic, agricultural and environmental uses~\citep{dauparas2022robust, watson2023novo, ingraham2023illuminating}. The outcome of any AI-based protein engineering approach is to ultimately specify a protein sequence, or a set of such sequences, even when designed by way of structure. For example, one might first use a backbone structure generative model~\citep{watson2023novo,ingraham2023illuminating,hayes2025simulating}, followed by use of a backbone-conditioned generative model on sequences (also called an inverse-folding model)~\citep{dauparas2022robust,hsu2022learning}, to generate sequences likely to fold into the specified backbones. Alternatively, one might forgo generating a backbone structure intermediary and instead generate sequences directly from a sequence-only model~\citep{hayes2025simulating,madani2023large,alamdari2023protein,wang2024diffusion,marks2011protein,russ2020evolution}. In all cases, the final step employs a generative model of sequences.

In most practical settings one is interested in generating sequences with particular properties of interest---often properties that are rare in the distribution of natural sequences. For example, suppose one wishes to increase the catalytic activity of an enzyme for some substrate, then these desired proteins are not expected to be well-represented in most general-purpose, pre-trained generative models. Nor are they likely to be well-represented by conditioning on existing annotation labels in such models. Consequently, protein engineers often acquire experimental measurements to glean which proteins are more likely to express this property at the desired level. With these newly generated data, they can construct a predictive model for the measured property, with the hope of leveraging this model to generate new sequences to better achieve their design goal. However, current strategies for incorporating such a model or other auxiliary information into generation remain limited and often {\it ad hoc}. 

A common workaround to re-training a generative model is to use {\it post hoc} filtering---that is, to generate many candidate sequences from a general, pre-trained model and then to filter these sequences according to a predictive model~\citep{yeh2023novo, bennett2025atomically}. However, {\it post hoc} filtering is fundamentally constrained by the pre-trained model’s distribution, which makes  sampling of desired proteins increasingly inefficient as the desired sequences become more rare in the pre-trained generative model distribution~\citep{lisanza2024multistate}. Moreover, for filtering on more than one property, the sampling inefficiency problem is compounded~\citep{hong2024integrative}.
Alternatively, fine-tuning---which does not modulate the sampling distribution on-the-fly and does require re-training a generative model---is another commonly used approach~\citep{ziegler2019fine, ouyang2022training}. In fine-tuning, a set of proteins thought to most exhibit the desired property---such as those that scored most highly in the laboratory---are used to continue training the pre-trained generative model~\citep{ruffolo2024design, nijkamp2023progen2, ivanvcic2025discovery}.
Fine-tuning concentrates the generative model distribution on the fine-tuning dataset, thereby coercing it to sample from the relevant part of the existing pre-trained model distribution, rather than enabling extrapolation beyond that distribution~\citep{nijkamp2023progen2, munsamy2024conditional, ruffolo2024design}. Furthermore, fine-tuning risks overwriting important general domain information learned by the model~\citep{kotha2024understanding, luo2025empirical}, such as features related to the stability of a protein. The pre-trained model can learn about such general properties from data sets that are typically orders of magnitude larger than the fine-tuning data set. We shall see the practical difficulties that fine-tuning encounters in our experiments. See ``Alternative approaches for infusing generative models with auxiliary information'' and Section~\ref{sec:si-related-work} for more detailed discussions of other methods. 

The aforementioned limitations point to the need for a methodological framework that 
can infuse sequence generation from a general, pre-trained model with information about the desired properties, all-the-while retaining general knowledge absorbed into the pre-trained model. Development of such a framework should increase our chances of discovering parts of protein design space that are inaccessible to the generative model alone, or to the predictive model alone. We would also like this framework to handle multiple properties, which might be in tension with one another, such as enzyme specificity and activity.

%
Herein, we propose \ourmethod (\fref{fig:schematic}) to fill this gap. Specifically, \ourmethod enables “on-the-fly”  conditioning, such that
we do not train a new generative model; instead, we iteratively modulate, or “guide” the generation process of a pre-trained unconditional model with information from one or more property predictive models. Each sampling step in this guidance procedure represents a sampling step from the desired conditional generative model, without the need to re-train a generative model to explicitly capture the desired conditional distribution.

\ourmethod can be used, for example, to guide an inverse-folding model to produce proteins with greater solubility/expression/activity than the model naturally generates, or, to produce proteins more likely to bind a particular target, and so forth. Predictive models for generic properties can be used, such as those for stability/expression/solubility, or new, user-created models trained on newly acquired experimental measurements, such as catalytic activity for a desired reaction. Owing to the latter, \ourmethod enables easy integration of experimental data in a statistically coherent manner with an existing generative model, without re-training the existing model. Consequently, \ourmethod should be a key enabler in training-efficient iterative cycles of design, build, test, and learn---a process where experimental feedback guides subsequent AI-based design rounds.

Part of the technical underpinnings for \ourmethod rely on our earlier development of guidance for continuous-time diffusion and flow-matching models on discrete state-spaces, referred to herein as \ourmethodICLR~\citep{nisonoff2025unlocking}. Following on that work, others found that \ourmethodICLR performed favorably compared to fine-tuning for protein property design tasks when used to guide their family-specific generative models~\citep{yang2025steeringgenerativemodelsexperimental}.
Crucially, \ourmethodICLR did not address guidance for models other than discrete diffusion models (DDMs)~\citep{austin2021structured,campbell2022continuous,lou2023discrete} and discrete flow matching models (DFMs)~\citep{campbell2024generative, gat2024discrete}, thereby limiting its utility to the field of protein engineering and beyond. For example, our earlier work did not provide a way to guide ESM3~\citep{hayes2025simulating}, ProteinMPNN~\citep{dauparas2022robust} and many other models~\citep{madani2023large, nijkamp2023progen2, esm_cambrian_2024}. With \ourmethod, we extend guidance to also include masked language models with masking rates covering the range $0\%-100\%$, sometimes known as generative masked language models (MLMs)~\citep{ghazvininejad2019mask,wang2019bert,chang2022maskgit,li2024promises}, and \aoarlong models (\aoamshorts)~\citep{uria2014deep,yang2019xlnet,hoogeboom2021autoregressive}. In so doing, \ourmethod unlocks guidance to much larger set of models that are actively being developed by the protein engineering community.

To demonstrate the potential utility of \ourmethod, we deployed it on a range of problems in protein engineering, first evaluating the approach {\it in silico}, and then in the wet lab. Specifically, we demonstrate the applicability of \ourmethod \ {\it in silico} by guiding different pre-trained protein generative models, including ProteinMPNN, ESM3, and ESM C~\citep{esm_cambrian_2024}, to generate amino acid and structure token sequences, conditioned on several user-specified properties. These properties include higher stability and activity, functional annotations such as enzyme classes and CATH-labeled folds. We show results for generating new sequences with similar properties to the training data (interpolative), and also pushing beyond that (extrapolative). We also demonstrate that \ourmethod can be used on more challenging problems wherein we seek to design for two different desired properties that are in tension with one another (increased on-target, and decreased off-target activity). 
Finally, we demonstrate the utility of our method experimentally by using \ourmethod with two different inverse folding models to design new adenine base editors. We obtained experimental editing activity measurements in a high-throughput assay and train a predictive model to design sequences with higher editing activities. We show that editing efficiency previously achieved with seven rounds of directed evolution can now be surpassed with just a single round of guided design with \ourmethod.

\section*{Results}

We begin by introducing the concept of guidance and the theoretical framework underlying \ourmethod, from which we derive the guidance algorithms used throughout this work. We then place \ourmethod in the context of alternative strategies for incorporating auxiliary information into generative models, before presenting our \textit{in silico} and \textit{in vivo} experimental results.



\begin{figure*}[!htbp]
\centering
\includegraphics[width=1.0\textwidth]{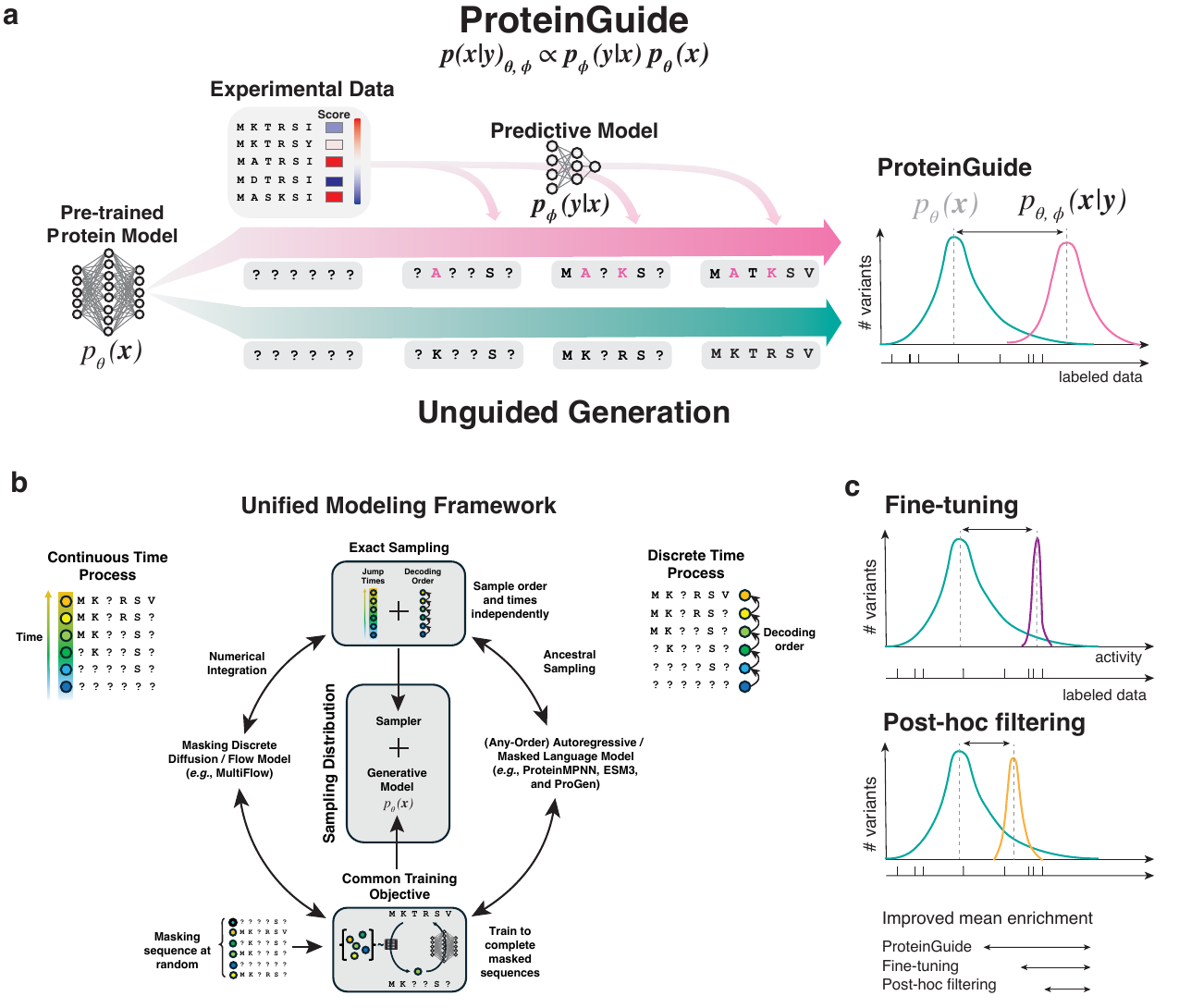}
\caption{
\textbf{Overview of \ourmethod.} 
    \textbf{a}) Contrasting unguided generation with guided generation by way of \ourmethodacro. A sequence is denoted, $x$, the pre-trained generative model is denoted $p_\theta(x)$; $y$ denotes a property of interest that we seek to guide toward using predictive model $p_\phi(y \mid x)$. When guiding, we blend information from the predictive model with that of the pre-trained model in a statistically correct manner, resulting in sequences sampled from the property-informed, reweighted distribution, $p_{\theta, \phi}(x \mid  y)$. Masking of one sequence position is denoted with the symbol ``?'', and pink letters indicate key mutations that are introduced only when guiding with the predictive model, rather than when using unguided generation. 
    \textbf{b}) Unifying view of training and sampling for masked diffusion/flow models, MLMs, and AO-ARMs enables guidance on a broad set of model classes. Colored dots visually tag individual sequences, each of which has a different amount of masked tokens. The sequence being masked during training is \texttt{M K T R S V}. The shared training objective involves randomly masking out parts of the sequence and then predicting the masked sequence given the remaining unmasked context over different masking rates. Once any type of model is trained, one can then use a sampling algorithm originally developed for any of the other models. Consequently, we can obtain computationally tractable exact sampling for continuous time models and also accelerated approximate sampling.
    \textbf{c}) Schematic contrasting \ourmethod with {\it post hoc} filtering and fine-tuning. \ourmethod combines a pre-trained generative prior with assay-labeled data through correct statistical conditioning, enabling more reliable enrichment for high-property sequences than the other approaches shown.
}
\label{fig:schematic}
\end{figure*}


\subsection*{Intuitive overview of guidance}
\label{sec:technical_overview}

We here begin our intuitive overview of \ourmethod with the notion of guidance itself. Conceptually, performing guidance comprises taking i) a pre-trained generative model, $p(x)$, on sequences, $x$, and ii) a classifier/regression predictive model, $p(y | x)$ for property $y$, to enable sampling (\ie, generation) from the desired conditional generative model, \mbox{$x \sim p(x | y) \propto p(y|x) p(x)$}. Notably, guidance does so without training/re-training an underlying generative model on $x$---rather, it modulates the sampling process of the unconditional model. Specifically, we execute Bayes' Rule at each sampling step to compute a property-conditioned distribution for the next amino acid using the unconditional model and the predictive model. This makes the guided sequence an effective sample from a property-conditioned generative model without ever needing to materialize the conditional generative model itself. We therefore refer to guidance as ``on-the-fly" conditioning. It does not require training a generative model---it only modulates the sampling process of the unconditional model.


We first develop guidance for discrete-space diffusion and flow-matching models, a method we earlier dubbed \ourmethodICLR~\cite{nisonoff2025unlocking}, which builds on prior work for continuous-space diffusion models. In continuous spaces, the key insight enabling guidance was that the generative process uses a series of gradient steps, each of which can be tractably converted to a conditional gradient with Bayes' Rule. Consequently, the corresponding sequence of gradient steps enables sampling from the desired conditional distribution~\citep{sohl2015deep,song2020score,dhariwal2021diffusion}. However, for discrete-space generative models, such as discrete-space diffusion and flow-matching models, these gradients are not defined. In fact, in discrete spaces, such models need to be set up completely differently even without guidance, to rely on \textit{transition rates} instead of gradients~\citep{campbell2022continuous,sun2022score,lou2023discrete, campbell2024generative,gat2024discrete}. In an analogous manner to deriving conditional gradients for continuous-space diffusion guidance, we derived conditional transition rates to unlock guidance for diffusion and flow-matching models on discrete spaces such as sequences (Section~\ref{sec:method-background}). 

Our derivation shows that to sample from the desired conditional generative model, one 
modulates the unguided transition rates with a likelihood ratio of the desired property condition before and after each possible transition (Section~\ref{sec:method-background}). 
The key to enabling practical application of guidance with these conditional rates is to leverage aspects of the continuous-time process to enable what is naively an intractably expensive computation, to one that is manageable. Specifically, we show how to change naive guidance from being combinatorial in the sequence length at each sampling step, to just linear, by noticing that it is sufficient to assume that only one position mutates in infinitesimal time (Section~\ref{sec:method-background}).
We also show how to speed up this \textit{exact guidance} even further by approximating repeated calls to the predictive model with a first order approximation to the predictive model. This even faster, approximate guidance, called \textit{Taylor-approximate guidance (TAG)}, often matches the sample quality of exact guidance on many tasks~\citep{nisonoff2025unlocking}.

\subsection*{Overview of {\ourmethod} and theoretical results}


\ourmethod is a unifying framework that extends \ourmethodICLR to a broad range of discrete-space generative models by exploiting equivalences in their training objectives and sampling distributions shown herein (Section~\ref{sec:method-training} and~\ref{sec:method-sampling}). A key observation is that when diffusion and flow-based models employ a so-called masking noise process---a common choice for protein sequence modeling---their training objectives and sampling distributions reduce to forms that are mathematically equivalent to those of generative MLMs and \aoamshorts (Section~\ref{sec:method-training} and~\ref{sec:method-sampling}). As a result, four widely used model classes---DDMs, DFMs, MLMs, and \aoamshorts---can be unified under a single framework comprising two regimes: continuous-time models (DDMs and DFMs) and discrete-time models (MLMs and \aoamshorts) (\fref{fig:schematic}b). By establishing their training and sampling equivalences, \ourmethod allows guidance algorithms to be transferred across model classes, enabling any of the four to be guided using either a continuous-time (DFM-style) or discrete-time (\aoamshort-style) formulation.

Although these two strategies are theoretically equivalent, they have practical differences that give rise to two different concrete algorithms. 
The first is Taylor-approximate guidance (TAG), a fast, approximate method proposed in \ourmethodICLR that we use to guide pre-trained generative models as DFMs using the derived training equivalences. We used TAG in the majority of our experiments.
The second is the newly derived method that we called \textit{discrete-time exact guidance (\ournewsamplingmethod)}. This is a slower, but exact, guidance algorithm that uses our sampling equivalence to directly extend \ourmethodICLR to discrete-time models. We use \ournewsamplingmethod to guide models as \aoamshorts when either the numerical integration used to sample from a continuous time model is unstable or the first-order approximation used by TAG is insufficient. In particular, we used \ournewsamplingmethod only in our multi-property Pareto-extrapolative and enzyme class guidance experiments. 

While connections between the training objectives of aforementioned model classes have been previously noted in various forms~\citep{austin2021structured,hoogeboom2021autoregressive,shi2024simplified,sahoo2024simple,ou2025your}, the practical utility of these observations has been unclear. Herein, we show that this connection can be used to guide MLMs and \aoamshorts by treating them as discrete flow matching models so as to then apply \ourmethodICLR. We also note that our sampling equivalence can be used more generally to define equivalent \aoamshorts to exactly sample from any DFM or DDM, with or without guidance. This generalizes recent results for DDMs showing that efficient ``Gillespie'' samplers could be created for those models~\cite{peng2025path, chen2024fast, hoogeboom2021autoregressive, zheng2024masked, ou2025your, shi2024simplified, sahoo2024simple}, and shows, for the first time, that exact sampling from a DFM can be done tractably. Previously this was assumed to be intractable because it would require infinitesimally small numerical integration time steps~\cite{campbell2024generative,nisonoff2025unlocking,havasi2025edit}.

Altogether, our unifying view of training and sampling brings these four model classes under a single framework and enables \ourmethod to guide a broad set of discrete-space generative models, making it readily applicable to many popular protein sequence generative models commonly used by practitioners (\fref{fig:schematic}b). 


\subsection*{Alternative approaches for infusing generative models with auxiliary information}
\label{sec:related_work_main}

The simplest approach for generating sequences from pre-trained models with desired property criteria is to \textit{post hoc} filter the generated sequences to only those satisfying the user-specified criteria (\fref{fig:schematic}c). Such an approach can be viable when the pre-trained model by itself frequently generates sequences satisfying the criteria. However, more generally, such an approach is extremely inefficient because the properties of interest are not well-captured by general pre-trained models. We shall also see this in our experiments.
Any approach that stands a chance of outperforming \textit{post hoc} filtering must involve, in some fashion, ``adapting'' the generative models to infuse it with knowledge of the relationship between sequences and the desired properties, as \ourmethod does. 

Perhaps the most commonly used approach in protein engineering for infusing experimental data into a generative model is by way of so-called ``fine-tuning'' (FT, sometime referred to as supervised finetuning, see Section~\ref{si_sec:lora})~\citep{ziegler2019fine, ouyang2022training}. FT involves training the generative model further with its original objective on a smaller, curated set of examples that have the desired properties (\fref{fig:schematic}c).
In the context of protein sequence generative models, this set might comprise sequences from a particular family that one hopes to focus the model on~\citep{ruffolo2024design,nijkamp2023progen2,gordon2025protein, munsamy2024conditional,chen2025target}, or sequences that are categorized as being functional~\citep{blalock2025functional,stocco2024guiding}. 
\sftshort has the benefit of being immediately applicable to any class of generative models. However, it has two downsides. First, FT is a heuristic method for combining sources of information, for which the statistical outcome is not well-defined (as opposed to say \ourmethod which executes Bayes' Rule). Consequently, important information from the pre-trained model may be lost, a phenomenon often referred to as ``forgetting''~\citep{kotha2024understanding,luo2025empirical}. Second, the act of continuing to train the original objective on new data concentrates the existing pre-trained distribution onto the portion corresponding to the new data, limiting generation to sequences whose properties resemble those seen during training~\citep{nijkamp2023progen2, munsamy2024conditional, ruffolo2024design}. We see evidence of this in its relative inability to extrapolate in several of our \textit{in silico} experiments (see Section ``Extrapolative experiments: guiding ESM3 to generate sequences from a family for desired properties'' and ``Multi-property, Pareto-extrapolative experiment: guiding ESM C for on- and
off-target binding of PbrR'').
Nevertheless, \sftshort has been shown to be useful either on its own, or as a first step before applying other approaches~\citep{ziegler2019fine, ouyang2022training, blalock2025functional} and is likely sufficient for designing proteins with similar property values to the training data.

Fundamentally, guidance, and consequently \ourmethod, differs from both \textit{post hoc} filtering and fine-tuning in that it uses a stand-alone predictive model to infuse auxiliary information into the pre-trained model, and does so using the statistical operation of executing Bayes Rule. Consequently, \ourmethod provides a framework with which to understand how we are blending sources of information. Additionally, because the predictive model can have class probabilities, or real-valued predictions, it can provide more nuanced information than FT which is limited to using only a select set of samples without associated property strengths. Because \ourmethod executes Bayes' rule, it can coherently blend several independent predictive models to simultaneously condition on all properties (see Section ``Multi-property, Pareto-extrapolative experiment: guiding ESM C for on- and
off-target binding of PbrR'').

Other statistically grounded frameworks have also been proposed. For example, Direct Preference Optimization (DPO)~\citep{rafailov2023direct} and related work~\citep{widatalla2024aligning, chennakesavalu2025aligning,wang2025finetuning,rector-brooks2025steering} enable Bayesian-style updating of a pre-trained model using ``preference-labeled" data or user specified ``reward function''. Notably, these methods all require training a new generative model before one can sample from it.
Up until recently, DPO has been applicable only to models with tractable likelihoods, thereby restricting them largely to autoregressive models~\citep{widatalla2024aligning}, although DPO was recently extended to diffusion models on continuous spaces by way of approximation~\citep{wallace2024diffusion}, and then to discrete spaces~\citep{borso2025preference}. Further research is needed to understand and improve the quality of such approximations. Herein, we focus our study on ``on-the-fly" guidance, rather than approaches that require training a generative model. However, we believe that such approaches will also prove useful. 

Our view is that guidance offers an approach to
statistical conditioning that is theoretically sound, practically well-behaved, flexible, and complementary to existing techniques. Importantly, guidance uniquely does not require access to the parameters of pre-trained models---it only requires predictions (logits) from the model. Accordingly, it can be used with closed-source models and large open source models for which fine-tuning might not be feasible for many users. Additionally, in settings where useful property-predictive models already exist, guidance allows one to directly leverage those models for the design task at hand. 
In recent work comparing guidance-based approach to fine-tuning
for protein property-focused design tasks, it was found that our \ourmethodICLR~\cite{nisonoff2025unlocking}
performed favorably compared to fine-tuning based approaches, and that guidance requires less hyperparameter tuning and computation, making it more practically accessible to everyday users than other approaches examined~\cite{yang2025steeringgenerativemodelsexperimental}. 
For further discussion of related work, including that of Markov Chain Monte Carlo (MCMC) methods~\citep{emami2023plug,wu2023practical,li2024derivative,singhal2025a}, we refer the readers to Section~\ref{sec:si-related-work}.

\subsection*{Overview of \textit{in silico} experiments}
\label{sec:in_silico_overview}

To illustrate the potential of \ourmethod, we first applied it, {\it in silico}, to a suite of experiments spanning three design settings, using a set of representative, well-known protein generative models, including ProteinMPNN~\citep{dauparas2022robust}, ESM3~\citep{hayes2025simulating}, and ESM C~\citep{esm_cambrian_2024}, as illustrative examples. See \supptref{tab:design_tasks} for an overview of these experiments. 
Note that \ourmethod can also be applied to many other sequence generative models~\citep{alamdari2023protein, wang2024diffusion, campbell2024generative, madani2023large, nijkamp2023progen2}. 

The three design settings capture three valuable use-cases. The first is an {\it interpolative design setting}, wherein we sought to generate proteins with property values that were seen in the supervised training data, but with greater diversity than appeared there. In a second category, an {\it extrapolative design setting}, we sought to generate diverse proteins with property values that exceeded any seen in the supervised training data. Finally, in a third task category, \textit{multi-property, Pareto-extrapolative}, we sought to improve one property, while decreasing another, such that we could push beyond the optimality frontier of these two properties defined by the training data for the supervised model.

For the interpolative task, we guided the inverse-folding model ProteinMPNN with a stability predictive model to generate sequences. We also guided the multi-modal model ESM3 with function annotations to generate sequences and structures. Because the interpolative setting is a relatively easy one for design, we highlight the former in the main text, leaving the other to the Supplementary Information (Section~\ref{sec:method-enzyme},~\ref{sec:method-fold-class}, Supplementary Figure~\ref{si_fig:ec},~\ref{si_fig:fold_class}).

For the extrapolative tasks, we guided ESM3 to generate sequences for each of three proteins, with predictive models trained on assay-labeled data measuring fluorescence, ubiquitin ligase activity, or \textit{in vivo} functional complementation. For the extrapolative, multi-property Pareto task, we guided ESM C to generate PbrR sequences, that simultaneously bind the desired lead ion with higher affinity and the off-target, zinc ion with low affinity. 

The interpolative tasks had hundreds of thousands of data points with which to train the predictive model, while the extrapolative tasks had around 2,000, and the Pareto-extrapolative task had around 1,000. These numbers were chosen to be in line with our wet lab experiment where we used 2,000 assayed variants to train our predictive model.

In all of these {\it in silico} experiments, we compared \ourmethod to the most commonly-used approaches in practice, namely, fine-tuning and \textit{post hoc} filtering. We used LoRA~\citep{hu2022lora} when performing fine-tuning with ESM3 and ESM C. In addition to these two baselines, we compared to the pre-trained model without any guidance or other changes. This baseline demonstrates
how frequently the desired property would emerge by simple use of a general, unconditional model, reflecting the difficulty of the task for current models.
%
%
For meaningful comparison, we allotted each method the same fixed compute budget as measured by wall-clock time, and a sample budget of $k$ samples. Full experimental details, including datasets, predictive models, baseline configurations, sampling settings, as well as additional results, are provided in the Methods and Supplementary Information.  We now go through the \textit{in silico} experimental results.

\subsection*{Interpolative experiments: guiding ProteinMPNN for improved stability}
\label{sec:pmpnn_stability}

Here we illustrate how assay-labeled experimental data can be used to condition protein sequence generative models to improve their success in generating sequences with a desired function, as measured in those data. We used \ourmethod to guide ProteinMPNN, an \aoamshort, with a  predictive model trained on a large dataset of experimental folding stability measurements of mini-proteins~\citep{tsuboyama2023mega}. The goal was to generate sequences with greater stability than the wild-type (\fref{fig:pmpnn_stability}a). A similar experiment was previously conducted by re-training the inverse folding model, which \ourmethod lets us bypass~\citep{widatalla2024aligning}. 

We conducted these experiments on eight randomly chosen proteins from the stability dataset for which sequences generated by ProteinMPNN were less stable than the wild-type sequence on average. For each protein, the success rate was computed over $k=100$ generated sequences, where success was defined based on two criteria: (i) the generated sequence should fold into the desired structure, and (ii) the generated sequence should be at least as stable as the wild-type sequence, \mbox{$\ddg \leq 0$}. We used ESM3 to predict the structure of the generated sequences and compare the root mean square deviation (RMSD) between the predicted structure and the desired structure; we considered the predicted designed structure to be successful if the RMSD was less than or equal to $2$\r{A}, following previous evaluation criteria~\citep{campbell2024generative}. To evaluate the folding stability of the generated sequences, we used the stability oracle from~\citet{nisonoff2025unlocking}.

In addition to ProteinMPNN's default autoregressive sampling, through our equivalences, we could also more quickly sample from it as a flow-matching (FM) model, which we also compared to. 
In implementing the \sftshort baseline, we followed a similar procedure to \citet{widatalla2024aligning} in that we further trained ProteinMPNN on its original objective on all sequences more stable than its corresponding wild-type in the stability dataset (Section~\ref{sec:method-stability}).

Across all eight proteins, \ourmethod consistently yielded the highest success rates (\fref{fig:pmpnn_stability}b). In particular, the generated sequences from \ourmethod preserved the structure constraints of the backbone as measured by refolded RMSD (\fref{fig:pmpnn_stability}c), while simultaneously improving their predicted stability (\fref{fig:pmpnn_stability}d). As expected, \textit{post hoc} filtering worked best when 
the unguided model had at least a moderate success rate. \sftshort improved success rates of the unguided model in some cases, but fails to do so in others, for reasons that are not yet fully understood. In contrast to these two methods, \ourmethod worked well across the board. Note that diversity and novelty of generated sequences was similar across all methods (\suppfref{si_fig:stability_diversity_novelty}).

\begin{figure*}[h!]
\centering
\includegraphics[width=1.0\textwidth]{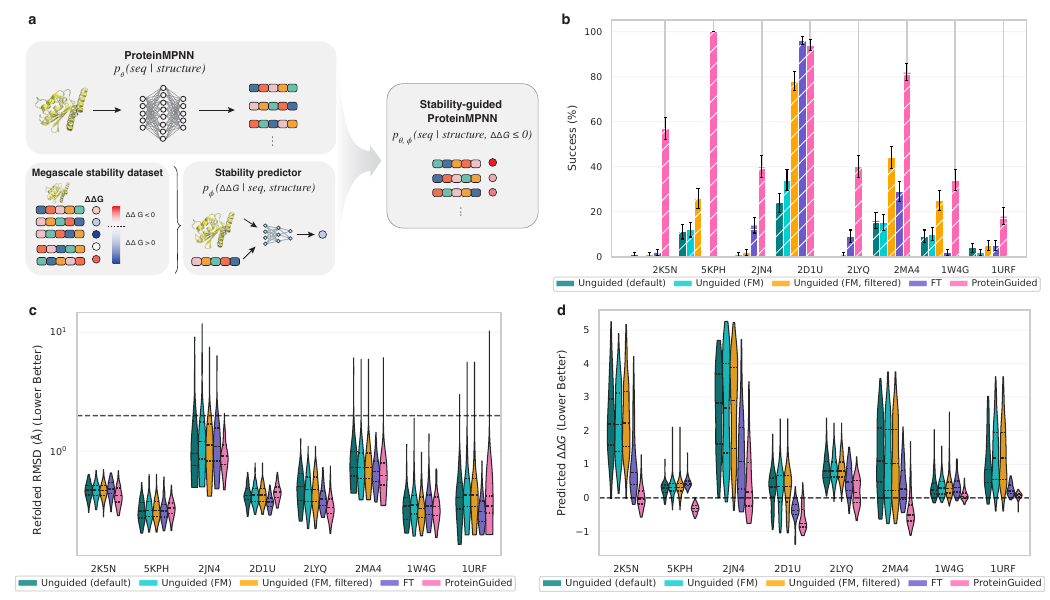}
\caption{
    \textbf{Guiding ProteinMPNN with experimental stability data to generate sequences with enhanced stability.}
    \textbf{a}) Schematics for guiding ProteinMPNN with large-scale experimental stability measurements. ProteinMPNN, an inverse folding model, is guided by a stability predictive model trained on the large-scale stability dataset from~\citet{tsuboyama2023mega} to generate sequences with stability better than the wild-type sequence. 
    \textbf{b}) Five methods for sampling (in legends) are assessed across eight proteins (horizontal axis) by their success rate on 100 generated samples. The height of each bar represents the percentage of sequences predicted to be at least as stable as wild type \mbox{($\ddg \leq 0$)}, while the hatched sub-portion indicates the fraction of those stable sequences that are also predicted to correctly fold into the desired structure (RMSD $\leq 2$\r{A}). In most cases the hatched sub-portion covers the full height of the bar. Error bars represent standard errors of proportions, calculated as $\sqrt{p(1-p)/n}$, where $p$ is the empirical proportion of success and $n$ is the sample size. 
    \textbf{c}) Distributions of refolded RMSD for the generated sequences (lower is better). The dashed horizontal line indicates the threshold below which a sequence is considered to fold into the desired structure (RMSD $\leq 2$\r{A}). 
    \textbf{d}) Distributions of oracle predicted stability for the generated sequences (lower is better). The dashed horizontal line indicates the threshold below which a sequence is predicted to be at least as stable as the wild-type \mbox{($\ddg \leq 0$)}.
}
\label{fig:pmpnn_stability}
\end{figure*}

\subsection*{Extrapolative experiments: guiding ESM3 to generate sequences from a family for desired properties}
\label{sec:fitness_single}

Having investigated the interpolative design setting, we next evaluated \ourmethod with the more challenging task of generating sequences with higher property values than those measured in the assay-labeled datasets (\fref{fig:fitness_single}a). We used assay-labeled datasets for three different proteins, each one having had a distinct property measured, and each data set including many variants beyond single mutations (\suppfref{si_fig:fitness_dataset}a). Specifically, we used data from i) the CreiLOV fluorescent protein, where the assay measures fluorescence intensity~\citep{chen2023deep}, ii) the ubiquitination factor E4B (UBE4B), where the assay measures E3 ubiquitin ligase activity~\citep{starita2013activity}, and iii) the poly(A)- binding protein Pab1 (PABP), where fitness is measured by way of an \textit{in vivo} functional complementation selection~\citep{melamed2013deep}. 

For our pre-trained model, we used ESM3 unconditionally. For each of the three proteins, we guided generation with a predictive model trained on 2,000 randomly chosen labeled sequences from that protein family. Following similar experimental set-up to prior work~\citep{blalock2025functional, yang2025steeringgenerativemodelsexperimental, emami2023plug}, we used an ``oracle'' model trained on all assay-labeled dataset as an \textit{in silico} proxy to assess the fitness of the generated sequences (\suppfref{si_fig:fitness_dataset}b). 
Since ESM3 is a ``pan-protein'' model, we also simultaneously guided ESM3 to generate sequences belonging to the desired protein family. This required training a predictive model to discriminate between sequences in the assay-labeled dataset and unguided generations from ESM3 (Section~\ref{sec:method-esm3-fitness}). 

For each protein design task, we allowed a budget of $k=100$ sequences. For fine-tuning in this set of experiments, we trained on the $M$ top variants from the assay-labeled data set. The value of $M$ was chosen by scanning through different $\%$ threshold on top-scoring variants as a hyperparameter, and choosing the one that had highest likelihood for held-out high-scoring variants.
%

Across all three proteins, \ourmethod consistently yielded the highest success rates. 
In particular, \ourmethod generated sequences with highest oracle property values compared to other methods (\fref{fig:fitness_single}b), and with higher diversity and novelty than those generated with FT (\fref{fig:fitness_single}c). Moreover, some of \ourmethod's variants had higher property values than any observed in the assay-labeled dataset (\fref{fig:fitness_single}d). As expected, \sftshort samples' oracle values recapitulated the range of those in its training dataset (\fref{fig:fitness_single}c,d). We hypothesize that this occured because \sftshort concentrates the distribution of the pre-trained model on the new training data---useful for interpolative design. 
\ourmethod, on the other hand, has the potential to extrapolate beyond the observed dataset to the extent that the predictive model used for guidance is able to extrapolate.

\begin{figure*}[h!]
\centering
\includegraphics[width=1.0\textwidth]{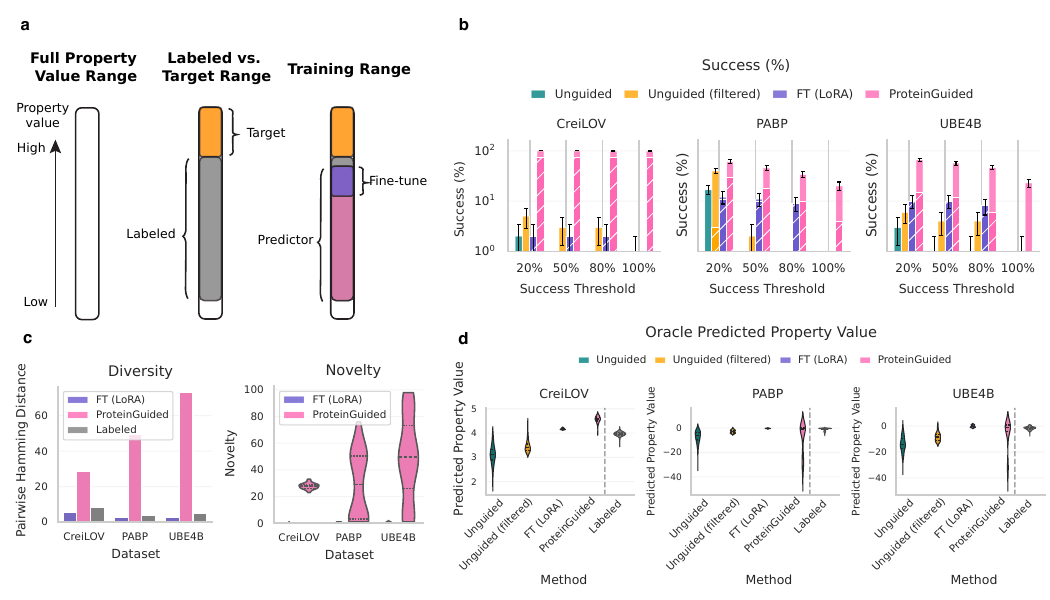}
\caption{
    \textbf{\ourmethod guides ESM3 with assay-informed predictive models for extrapolative design.} 
    \textbf{a}) Experimental set-up for the extrapolative setting, where the goal is to generate sequences with higher fitness compared to those in the labeled dataset. The set of sequences (around 100) with highest property values in the labeled dataset were excluded for both fine-tuning and predictive model training. Fine-tuning is performed on sequences on $M$ top variants from the assay-labeled data set. $M$ is chosen for each dataset scanning through different \% threshold on top-scoring variants as a hyperparameter, and choosing the one that had highest likelihood for held-out high-scoring variants. The predictive model used in guidance is trained on 2,000 sequences randomly sampled from the labeled dataset.
    \textbf{b}) For each dataset (subpanel title), each method is scored by its success rate on 100 samples. The height of each bar represents the percentage of sequences predicted by the oracle model to be above a particular threshold based on the quantile of the training set (horizontal axis) and are not identical to any sequence in the labeled set (\ie, novelty $> 0$), while the hatched sub-portion indicates the fraction of those sequences that are also predicted by ESMFold~\citep{lin2023evolutionary} to fold into a structure with pLDDT $\geq$ 0.8 (\suppfref{si_fig:fitness_structure}). Error bars represent standard errors of proportions, calculated as $\sqrt{p(1-p)/n}$, where $p$ is the empirical proportion of success and $n$ is the sample size.
    \textbf{c}) Diversity (left) and novelty (right) of the generated sequences. Higher diversity and novelty is desirable provided that the sequences also have high property value. Diversity is computed as the averaged pairwise hamming distance among the generated sequences. ``Labeled'' diversity is computed on 500 randomly sampled sequences from the labeled dataset for each protein. For each sequence, novelty is computed as its minimum hamming distance with sequences in the labeled dataset. The novelty of FT generated sequences is very low novelty and is not readily discernible.
    \textbf{d}) Distributions of oracle predicted property value for the generated sequences. ``Labeled'' consists of all sequences in the labeled dataset.
}
\label{fig:fitness_single}
\end{figure*}

\subsection*{Multi-property, Pareto-extrapolative experiment: guiding ESM C for on- and off-target binding of PbrR}
\label{sec:fitness_multi}

\begin{figure*}[htbp!]
\centering
\includegraphics[width=0.88\textwidth]{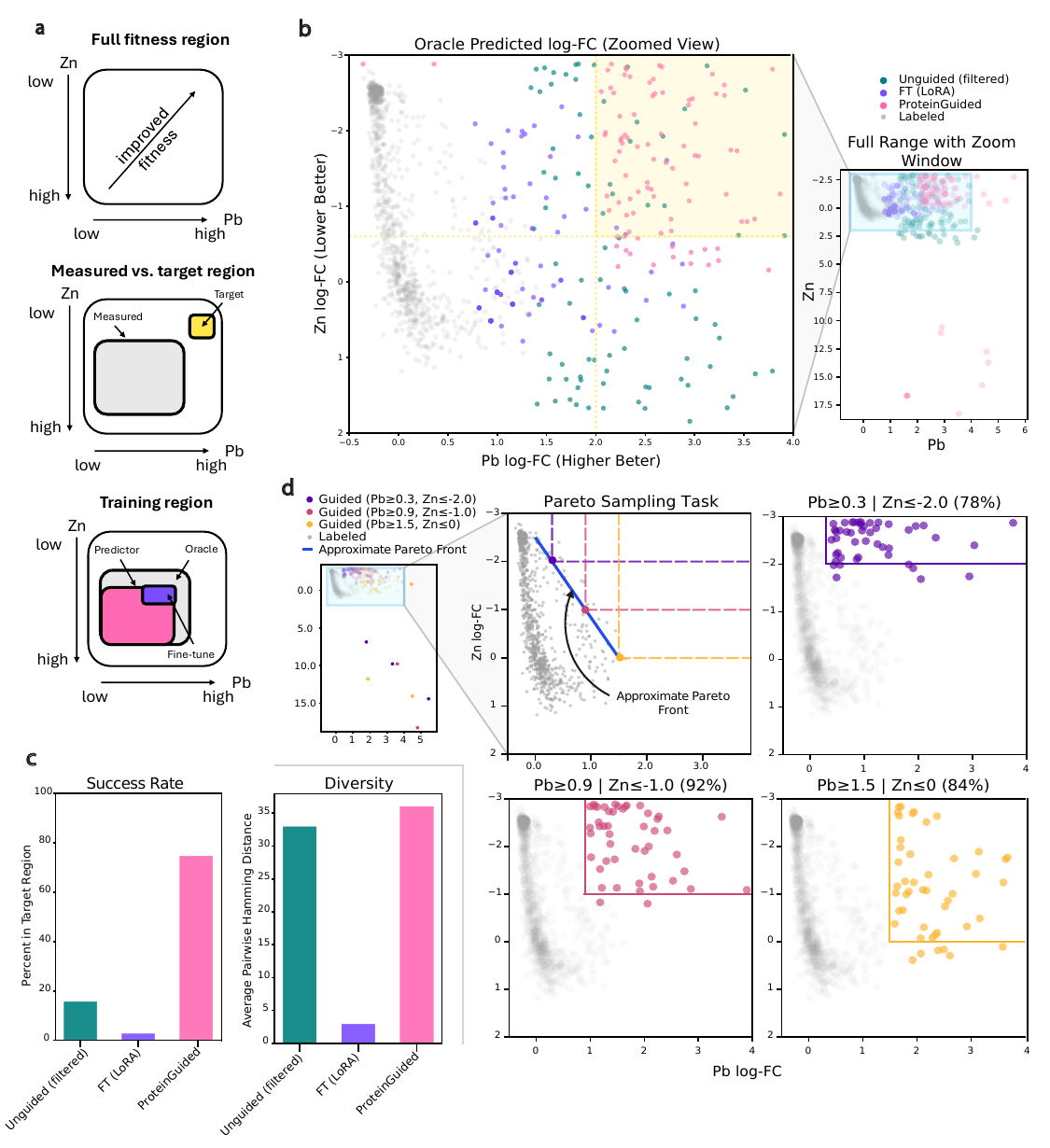}
\caption{
\textbf{Multi-property guidance with \ourmethod on PbrR to improve lead (Pb) binding and reduce off-target zinc (Zn) binding.}
    \textbf{a}) Schematic diagrams of the data splits used for training and evaluating models for multi-property guidance on PbrR. 
    \textbf{b}) Pareto-extrapolation task: scatterplot of oracle-predicted log-fold change for model generations for ESM C (Unguided), \sftshort, and \ourmethod. Experimental values of the initial library of variants from \cite{wang2025active} are also shown in gray. The target region requires a Pb log-FC of 2 compared to the wild-type and a Zn log-FC of -0.6 compared to wild-type.
    \textbf{c}) Left panel shows the proportion of sequences that were in the target region for each of ESM C (Unguided), \sftshort, and \ourmethod. Right panel shows the average Hamming distance between sequences generated by each model, as a measure of diversity.
    \textbf{d}) Controllable generation task: the top left panel shows outliers and the region corresponding to the remaining four plots (in yellow shading). Top middle panel shows each of three target regions for the Pareto sampling task as quadrants delineated by colored dashed lines. Each target region has it's bottom left vertex anchored on the approximate Pareto frontier. The remaining three plots (top right and bottom row) show sequences generated by guidance for each target region using the classifier thresholds in the subpanel titles. Subpanel titles also show success rate in parentheses.
}
\label{fig:fitness_multi}
\end{figure*}

Finally in our {\it in silico} series of experiments, we tackled a design problem comprising optimizing for two desired properties in tension with one another. This setting is of great interest to protein engineers who may want to, say, increase the activity of an enzyme while simultaneously keeping its stability above some threshold. Moreover, practitioners may want to turn a ``knob'' to trace out the optimality (\ie, Pareto) frontier in order to characterize it, or to optimize simultaneously for even more than two properties. 

Our multi-property optimization task uses data from a recent design campaign~\citep{wang2025finetuning} that optimized the PbrR protein for use in a cell-free biosensor to detect lead (Pb) contamination in water. PbrR is a transcription factor that forms a dimer when bound to metal ions---typically Pb (II)---allowing it to bind DNA and induce the transcription of a downstream reporter gene. The key challenge is that increasing Pb affinity typically also increases off-target binding to zinc (Zn). We therefore aimed to generate PbrR variants to simultaneously increase Pb binding while reducing off-target Zn (II) binding (\fref{fig:fitness_multi}a). That is, the goal was to push the Pareto frontier of the experimentally collected data outward past its limits of optimality (Figure \ref{fig:fitness_multi}d).

To simulate the early stage of a design campaign, we only used experimental data from the first library (of five) collected by~\citet{wang2025active}. This first library comprised of 1,098 variants arising from an alanine scan of the protein; a site saturation mutagenesis (SSM) at 49 positions; and a modest number of higher order mutations (up to degree 5) sampled from seven positions. We only designed positions that had SSM data available, leaving the rest at their wild-type amino acids. Binding affinities are reported as log fold change (log-FC) of fluorescence with respect to the wild-type protein in the presence of its metal-ion cofactor.

To use \ourmethod we trained two separate predictive models for each of Pb and Zn binding affinities, using all single site and pairwise terms in a linear additive model. We did so to explicitly mimic the common scenario in protein design where initial screens do not capture higher order epistatic interactions and predictive models cannot be expected to learn these interactions from the data. 
We used ESM C~\cite{esm_cambrian_2024} as the pre-trained model. Fine-tuning was done on variants lying on the experimentally observed Pareto frontier for on-/off-target activity (\suppfref{si_fig:multi_data}b, center panel).

We evaluated design for two different tasks.
In the first task---Pareto-extrapolation---we aimed to generate sequences in a target region that lay outside the (approximate) Pareto frontier of the experimental data. Specifically, we sought to produce sequences with Pb log-FC greater than $2.0$ and Zn log-FC less than $-0.6$ (\fref{fig:fitness_multi}b, yellow rectangle). In the second task---controllable generation---we selected three reference points along the Pareto frontier representing different Pb-/Zn-binding tradeoffs (\fref{fig:fitness_multi}d, top-left panel, points on the blue line). For each reference point, the target region corresponded to the part of function space with strictly better Pb and Zn binding characteristics than the reference point (see quadrants delineated by colored dashed lines in \fref{fig:fitness_multi}d, top-left panel). 

For the Pareto-extrapolation task, 75\% of the sequences generated using \ourmethod fell into the design target region (\fref{fig:fitness_multi}b), without losing any diversity compared to the unguided model (\fref{fig:fitness_multi}c). The unguided model's samples had high variance, partially covering the target region (\suppfref{si_fig:multi_main}a). Accordingly, \textit{post hoc} filtering the unguided sequences (top 100 sequences out of 1,000) yielded only a 16\% success rate. 
The fine-tuned model remained very close to the observed Pareto-optimal frontier in the dataset (\ie, its training data, see~\suppfref{si_fig:multi_data} center panel) and failed to extrapolate (\fref{fig:fitness_multi}b). Moreover, sequences generated from the fine-tuned model displayed poor diversity (\fref{fig:fitness_multi}c). In fact, the finetuned model actually predicts \textit{lower} log-likelihoods for sequences in the target region than the original unconditional model (\suppfref{si_fig:multi_main}f). 


Our next experiment---the controllable generation task---showcased the ability of \ourmethod to conditionally sample sequences according to a diverse range of property target values without retraining generative or predictive models. This is a key advantage of on-the-fly methods as part of design-build-test-learn (DBTL) cycles. As engineered proteins are evaluated in different contexts, the original design criteria may change. With \ourmethod, new sequences can then be sampled by simply inputting the new property constraints. For methods like fine-tuning, new datasets must be collected and new models must be trained and validated before sequences can be sampled.

In this experiment, we used \ourmethod to sample $k=50$ sequences for each of three target regions (\fref{fig:fitness_multi}d, top-right, bottom-left, and bottom-right subpanels). Each target region is anchored on a point from the Pareto frontiers. These anchor points represent different possible lower-bounds on the desired on-/off-target binding affinities for the generated sequences that could be specified by a user during a DBTL cycle. Overall, \ourmethod had an average success rate of 85\% across these three regions showing that it could readily handle a diverse set of user-specified design objectives. 

\subsection*{\textit{In vivo} experiments: engineering a base editor with enhanced activity}
\label{sec:tadA}

At last, we move to an \textit{in vivo} demonstration of \ourmethod to engineer an adenine base editor (ABEs)~\citep{rees2018base} for high editing activity (\fref{fig:tadA}a). Base editors are engineered CRISPR effectors that can make transition mutations in the DNA by directly deaminating cytosine (causing C-T mutations) or adenine (A-G mutations) in the single-stranded DNA exposed during target recognition by Cas9.

Our goal was to use experimental base editor activity data with \ourmethod to condition pre-trained generative models to produce sequences with improved activity above and beyond what the models had access to.  For pre-trained models, we used two inverse folding models, ESM3 and FMIF (Flow-Matching Inverse Folding)~\citep{nisonoff2025unlocking}---the latter, an inverse folding model that we trained with similar training data and architecture as ProteinMPNN, but having specifically removed all examples of engineered ABE variants from the training data. Our intent was to ensure that FMIF could not encode existing published information for engineering ABEs that is generally not available when engineering a new protein. 
We started our engineering campaign with the wild-type TadA sequence (an ABE) which is known to have very low, but detectable activity~\citep{ranzau2023wild}.
Then, we chose to only design the region of 86 residues on the C-terminal region of this variant due to constraints of gene synthesis; we also hypothesized that this region might be enriched for mutations important for increasing editing activity based on prior studies~\citep{gaudelli2017programmable, richter2020phage, xiao2024adenine}.

Activity levels for the base editors were assayed using an in-house bacterial selection assay based on the restoration of an inactive antibiotic resistance marker to measure the activity of the library, and activity was reported as log enrichments between the selective and non-selective conditions. 
Our ``pre-guidance'' first-round library was designed by sampling from each of the pre-trained models, ESM3 and FMIF, separately (1,000 samples per model). These sequences were then pooled into a library of 2,000 variants that was then assayed for activity. From this data set, we then built one classification model to be used for guidance, where  the classifier predicts the probability that a sequence’s LE is greater than zero. 

Next we could apply \ourmethod to design an improved second round library by guiding each of ESM3 and FMIF with the trained classifier. We generated 1,000 sequences by guiding each pre-trained model, yielding a total of 2,000 sequences, that we then assayed. These generated sequences were predicted to be active based on the classifier, while maintaining a similar level of diversity to the first round library (\fref{fig:tadA}b). 

We found that the majority of the \ourmethod variants in the second round had higher activities (mean log enrichment $0.66 \pm 0.02$) than the first round library (mean log enrichment $ -0.28 \pm 0.01$) (\fref{fig:tadA}c). Furthermore, the second round library revealed variants that were more active than any seen in the first round library, while having less than 80\% sequence identity to the wild-type TadA sequence within the designed region (\fref{fig:tadA}d). Intriguingly, we observed that while in the first round, more variants from ESM3 were predicted to be active than from FMIF (\fref{fig:tadA}b), the most active variants from the second round mostly came from FMIF (\fref{fig:tadA}d). Examining the distributions of sequences between the two rounds, we observed different enrichments of mutations from the two models, suggesting that 
they might be capturing different types of sequence constraints due to the difference in their training data (\fref{fig:tadA}e). Additionally, several enriched positions overlap with residues previously identified in directed evolution campaigns (\fref{fig:tadA}e). In particular, the D108N mutation---enriched in sequences generated by the guided FMIF model---has been reported in prior directed evolution campaigns, whereas most other beneficial mutations identified here differ from those previously observed (\fref{fig:tadA}e)~\citep{gaudelli2017programmable, richter2020phage}. 

We chose two of the variants with high enrichment values and low sequence similarity to the wild-type from our guided FMIF second library for further validation (\fref{fig:tadA}d, left panel). Specifically, we tested their activity individually in a titer plate-based assay and a pooled selection assay together with previously engineered base editors (BEs): ABE1.1, ABE7.10 and ABE8 (Section~\ref{sec:method-tadA})~\citep{gaudelli2017programmable, richter2020phage}. \ourmethod's variant with the highest log enrichment value, FMIF-1016, had activity between ABE7.10 and ABE8, which correspond to respectively seven and eight rounds of directed evolution (\fref{fig:tadA}f). Consequently, \ourmethod, in one round, generated base editors with activity higher than seven rounds of directed evolution. We speculate that further rounds could lead to still higher activity.

\begin{figure*}[ht!]
\centering
\includegraphics[width=0.9\textwidth]{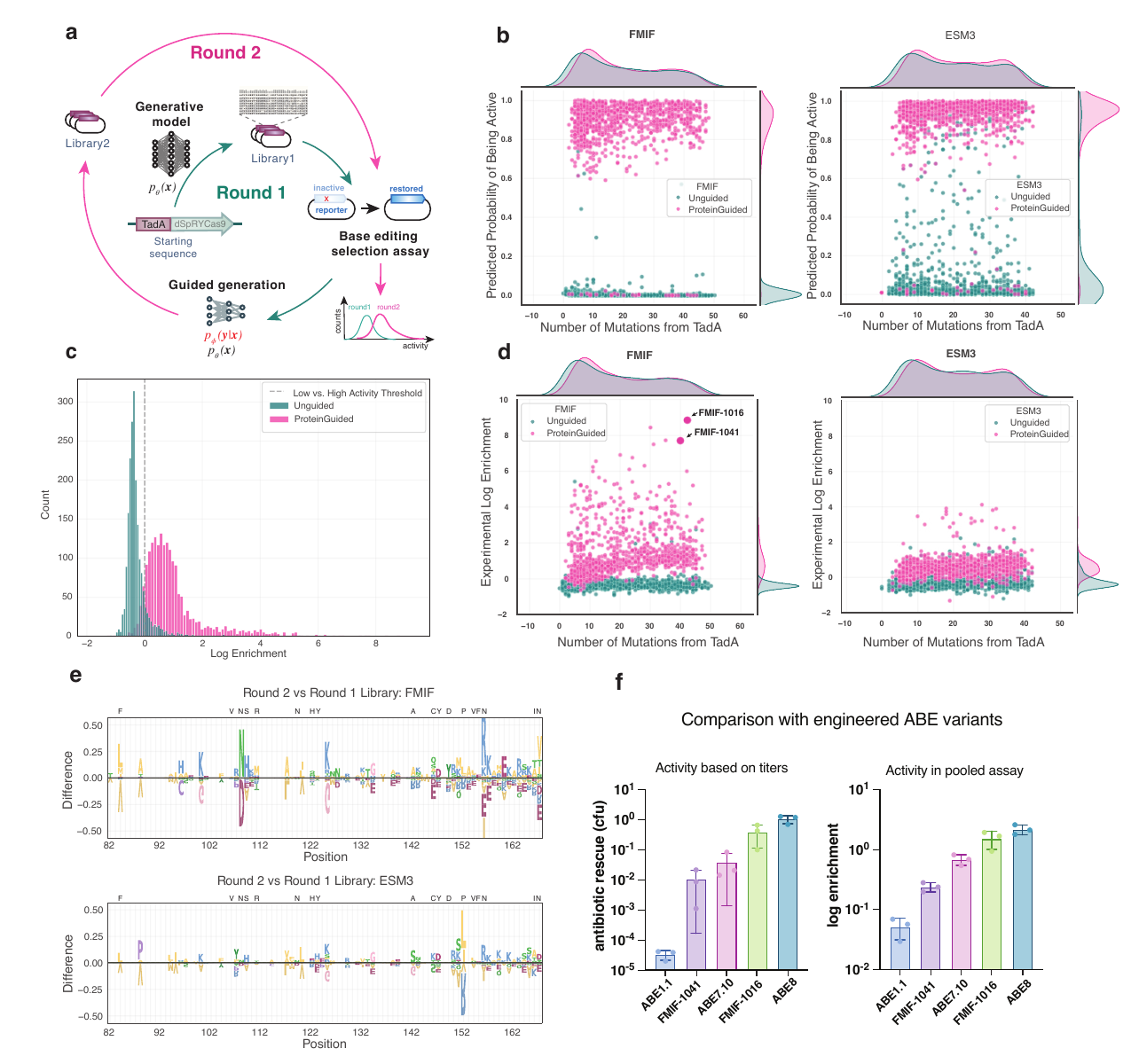}
\caption{
\textbf{Engineering adenine base editor with ProteinGuide.} 
    \textbf{a}) Schematic representation of the \ourmethod~ABE experiment. 
    \textbf{b}) Scatterplots of the number of mutations from TadA ($x$-axis) vs. the predicted probability of being active ($y$-axis) for every variant in library 1 (teal) and library 2 (pink) of FMIF (left) and ESM3 (right).
    \textbf{c}) A histogram of the experimentally measured log enrichments for both library 1 (teal) and library 2 (pink). \ourmethod~was used to generate the round 2 library such that variants would be enriched for having high activity (predicted activity in log enrichment units higher than the gray dashed line).
    \textbf{d}) Similar to \textbf{b}), but with experimentally measured log enrichments shown on the vertical axis.
    \textbf{e}) Sequence logo plots of the differences in the position-wise amino acids frequencies between library 1 and library 2 for FMIF (left) and ESM3 (right). Letters above the horizontal line at $y=0$ represent position/amino acids that are enriched in library 2 relative to library 1, and those below 0 represent positions/amino acids that are depleted in library 2 relative to library 1. Black letters above the plots show mutations obtained in the experimental evolution campaigns leading to ABE8. 
    \textbf{f}) Comparison of the activity of two of the top variants after \ourmethod with experimentally obtained variants with known activities (ABE1.1, ABE7.10, ABE8). Left: activity estimated based on the titers of individual transformations in the selection strain, shown are colony forming units (cfu) on selective normalized to non-selective condition. Right: activity calculated based on log-enrichment in a pooled selection assay measured by sequencing. Error bars represent standard deviations over three independent transformations for each experiment.
}
\label{fig:tadA}
\end{figure*}

\section*{Discussion}
\label{sec:discussion}

We introduced \ourmethod, a broadly applicable framework that enables conditioning of pre-trained protein generative models on any property of interest that can be encoded in a predictive model, without having to re-train the pre-trained model. An important application enabled by \ourmethod is to nimbly incorporate feedback from experimental measurements into the generative process, on-the fly. 
Through \textit{in silico} experiments in three distinct design settings---interpolative, extrapolative, and multi-property Pareto-extrapolative---we demonstrated the utility of \ourmethod~in a variety of application settings of interest to protein engineers, across a variety of pre-trained generative models, and spanning numerous protein properties. In our \textit{in vivo} experiment to engineer an adenine base editor with improved activity, we demonstrated that a single round of \ourmethod yielded variants with activity higher than that produced by seven rounds of directed evolution. 

Guidance, and hence \ourmethod, differs from the common practices of \textit{post hoc} filtering and fine-tuning in a few important ways. First, it uses a separate predictive model to modulate the sampling process from a pre-trained model, without having to train/re-train any generative model. Second, it uses the principle of Bayes' rule to statistically, coherently blend information from the pre-trained model with the predictive model, 
thereby enabling generation of variants that the pre-trained model would not otherwise produce. Moreover, the coherent blending mitigates worries about ``forgetting" that often emerge from fine-tuning.

Beyond straightforward application of our on-the-fly guidance approach, one might consider combining \ourmethod with other Bayesian-inspired frameworks that involve re-training a pre-trained model towards a reward-weighted posterior distribution~\citep{wang2025finetuning,rector-brooks2025steering,widatalla2024aligning,chennakesavalu2025aligning}, particularly in cases when the targeted region is significantly out of distribution for the pre-trained generative model. Pure guidance might be challenging in such a situation. Separately, it might prove promising to combine \ourmethod with MCMC-based methods, such as recent works for guiding generative models with Sequential Monte Carlo (SMC)~\citep{wu2023practical,li2024derivative,singhal2025a}, to improve the generation quality, albeit at the expense of added sampling computation time.

We anticipate future work will further extend and build upon our work. One interesting line of exploration might involve applying \ourmethod in conjunction with guidance methods for continuous spaces to guide multimodal generative models of both structure and sequence~\citep{campbell2024generative, jing2025generating}. 

\section*{Acknowledgments}
This work was funded in part by the the U.S. Department of Energy, Office of Science, Office of Biological and Environmental Research, Lawrence Livermore National Laboratory BioSecure SFA within the Secure Biosystems
Design program (SCW1710), and also the Office
of Naval Research (ONR) under grant N00014-23-1-2587. D.F.S. is an Investigator of the Howard Hughes Medical Institute, M.L. is a Life Science Research Fellow funded by Howard Hughes Medical Institute. 
For their assistance with Illumina sequencing we thank the Innovative Genomics Institute NGS facility, Netravathi Krishnappa and for nanopore and Sanger sequencing we thank the UC Berkeley DNA Sequencing Facility. We thank Hanlun Jiang, Antoine Koehl, Jianan Canal Li, and Flora Zhiqi Wang for helpful discussions.


\clearpage


\renewcommand{\thefigure}{S\arabic{figure}}
\renewcommand{\thetable}{S\arabic{table}}
\renewcommand{\theequation}{S\arabic{equation}}
\renewcommand{\thepage}{S\arabic{page}}
\setcounter{figure}{0}
\setcounter{table}{0}
\setcounter{equation}{0}
\setcounter{page}{1} 


\begin{center}
\section*{Supplementary Materials for\\ \scititle}


Junhao~Xiong$^{\ast}$,
Ishan~Gaur$^{\ast}$,
Maria~Lukarska$^{\ast}$,
Hunter~Nisonoff$^{\ast}$,\\
Luke~M.~Oltrogge,
David~F.~Savage$^{\dagger}$,
Jennifer~Listgarten$^{\dagger}$\\ 

\small$^{\dagger}$Corresponding authors. Email: \url{savage@berkeley.edu}, \url{jennl@berkeley.edu}\\
\small$^{\ast}$These authors contributed equally to this work.
\end{center}

\subsubsection*{This PDF file includes:}
Materials and Methods\\
Supplementary Text\\
Supplementary Tables\\
Supplementary Figures\\

\tableofcontents

\newpage


\include{materials_and_methods}
\newpage

\include{supplementary_text}
\newpage

\include{supplementary_figures}
\newpage

\bibliographystyle{unsrtnat}   
\bibliography{main}


\end{document}

%% file: materials_and_methods.tex
\section{Materials and Methods}
The Materials and Methods are organized as follows. We first provide an overview of the technical underpinning of \ourmethod (Section~\ref{sec:method-overview}), followed by more detailed technical background (Section~\ref{sec:method-background}). We proceed to show the equivalence of training objectives (Section~\ref{sec:method-training}) and sampling algorithms (Section~\ref{sec:method-sampling}) for a broad class of models, including masking-based diffusion/flow-matching models (MDMs/MFMs), generative Masked Language Models (MLMs), and \aoamlongs (\aoamshorts). The sections afterwards provide full experimental details for all experiments, including an overview for the \textit{in silico} experiments (Section~\ref{sec:method-insilico}), followed by dedicated sections for each experiment (Section~\ref{sec:method-stability}-\ref{sec:method-tadA}).

\subsection{Overview of \ourmethod}
\label{sec:method-overview}

Conceptually, the problem of guidance involves taking a pre-trained generative model, $p(x)$, and a classifier- or regression-derived likelihood, $p(y | x)$, to obtain \mbox{$p(x | y) = p(y|x) p(x) / p(y)$}, without re-training the underlying generative model. At bottom, designing guidance strategies amounts to exploiting local structure in the generative process so that this reweighting can be implemented tractably and efficiently during generation, while producing samples from the correct conditional distribution at the end of the process. As an example, diffusion models on continuous state-spaces generate data by learning to reverse a ``noising'' process, most commonly instantiated as gradually corrupting the data towards pure Gaussian noise~\citep{sohl2015deep,song2020score,ho2020denoising}. Intuitively, guidance in diffusion models is achieved by leveraging a key property of the underlying generative process---namely, that the model learns to map from a simple distribution (\eg, a standard normal distribution) to a complex data distribution through a series of infinitesimal steps. Each step ``points'' towards directions with higher likelihood under the data distribution, and can be ``nudged'', at generation time, towards directions that better satisfy the user-specified likelihood, so that the process ends with samples from the desired conditionals. These infinitesimal steps are formalized as gradients of noise-dependent distributions, known as \textit{score functions}, and Bayes' Rule can be simplified considerably when considering only these gradients, enabling guidance for these models~\citep{sohl2015deep,song2020score,dhariwal2021diffusion}.

For discrete-space generative models, however, these gradients are not defined, and such an approach cannot be readily adopted. Nevertheless, discrete-space diffusion and flow-based models also admit a natural local structure, but in a different form. These models learn to reverse a user-defined discrete noising process, such as introducing random mutations, or accumulating mask tokens~\citep{austin2021structured,hoogeboom2021autoregressive}. Such process can be formalized with continuous-time Markov chains (CTMCs), which can be described as a generative process wherein a state remains unchanged over some amount of holding time, then changes (jumps) to a new state. In practice, these two processes of holding, and jumping, are all jointly encoded in an object known as the (time-dependent) \textit{rate matrix}, which describes the time derivative of the probability of transitioning from any state to any other state at a given time. These rate matrices play a similar role to the score functions in continuous-space diffusion models, and sampling corresponds to integrating the CTMCs, for example, by straightforward Euler integration~\citep{campbell2024generative} or with more involved methods~\citep{gillespie1977exact, gillespie2001approximate, campbell2022continuous}. To make learning tractable, these models typically assume that the noising processes in each dimension are independent of one another (but not necessarily identical), with the important consequence that at any given time in the continuous-time noising process, the probability that two or more dimensions of the current state jump (transition) at exactly the same time is zero~\citep{campbell2022continuous,campbell2024generative}. Consequently, at any given time, the number of entries in the rate matrix that are non-zero is linear, rather than exponential, in the input space, making such models tractable.

It turns out that the same assumption which makes training these models tractable also makes guidance tractable for them. By reformulating Bayes' Rules in terms of the rates, we find that the \textit{conditional} rates at each sampling step can be obtained by modulating the unguided rates with a likelihood ratio describing how much the guidance signal increases (or decreases) in an infinitesimal time step. Since at any given time step, only transitions with a single mutation have non-zero rates, one can tractably compute the normalizing constant for any given time step~\citep{nisonoff2025unlocking}. Since this procedure affects only sampling and not training, one can sample from any desired conditional distribution for which a user-specified likelihood is available, without having to re-train the underlying pre-trained generative model. We refer to this work as \ourmethodICLR and the reader to~\citet{nisonoff2025unlocking} for more details.

While the framework just presented applies generally to discrete-space diffusion and flow-based models under a broad set of noising processes, one popular instantiation worth considering is a noising process based on masking, where tokens in a sequence are progressively replaced with mask tokens during the noising process and the model learns to recover the original tokens during the denoising process. These models share deep connections to two other popular classes of discrete-space generative models, Masked Language Models (MLMs) with masking rates covering the range $0\%-100\%$,, sometimes known as generative MLMs, and any-order autoregressive models (\aoamshorts). Although conceptually masked diffusion/flow-matching models (MDMs/MFMs), MLMs, and AO-ARMs seem to be quite different, it turns out that the loss functions used for training them are equivalent. While connections between these training objectives has been previously noted~\citep{austin2021structured,hoogeboom2021autoregressive,shi2024simplified,sahoo2024simple,ou2025your}, the practical utility of this observation has been unclear. Herein, we show that this connection enables a general framework for guiding all of these model classes. Moreover, we additionally show for the first time that sampling from a masked diffusion/flow-matching model can be accomplished more simply, yet exactly, by sampling from a corresponding, trivial-to-construct \aoamshort, which leads to a more efficient, exact guidance algorithm for MDMs, MFMs, MLMs, and \aoamshorts. Altogether, our unifying view of training and sampling of all four of these model classes under a single framework, enables \ourmethod to guide a broad class of discrete-space generative models, making it readily applicable to many popular protein sequence generative models commonly used by practitioners.

\subsection{Technical background }
\label{sec:method-background}

Here we give a more mathematical treatment of the overview above, providing the technical background important for understanding the results shown in later sections. 

\subsubsection{Discrete-space diffusion and flow models (DDMs and DFMs)}
While diffusion models were initially developed for continuous state-spaces~\citep{sohl2015deep,song2020score,ho2020denoising}, recent work has extended these frameworks to discrete domains such as text, molecular graphs, and protein sequences. These extensions enable generative modeling of inherently discrete data without requiring continuous relaxations. There are a number of works defining discrete-time diffusion processes in discrete spaces~\citep{austin2021structured, hoogeboom2021argmax}. They were later extended to continuous-time~\citep{campbell2022continuous,sun2022score, lou2023discrete} leveraging the formulation of continuous-time Markov chains (CTMCs), which describe stochastic processes where a system transitions between discrete states over continuous time. The dynamics of a CTMC are fully determined by an initial distribution $p(x_0)$ and transition rates $R_t(x, \tilde{x})$ that specify the instantaneous probability of transitioning from state $x$ to state $\tilde{x}$ at time $t$:
\begin{equation}
    p(X_{t+dt} = \tilde{x}|X_t = x) = \delta_{x,\tilde{x}} + R_t(x, \tilde{x})dt + O(dt^{1+\epsilon}).
\end{equation}
As an example, \citet{campbell2022continuous} defines a forward noising process by specifying a time-dependent ``forward'' rate matrix that gradually transforms the (unknown) data distribution into a reference distribution, such as a uniform distribution over amino acid sequences. A denoising model is then trained to reverse this process, learning rates that can transform the corrupted samples back to samples from the data distribution.

There are also a number of works that extend flow matching, which were originally developed for continuous state-spaces~\citep{albergo2022building, lipman2022flow, liu2022flow}, to discrete state-spaces~\citep{campbell2024generative, gat2024discrete}, also leveraging CTMCs, but differ from diffusion models in how the CTMCs are defined, which also induce different training objectives. These models specify a conditional flow that defines how probability mass should move from a noise distribution to the data distribution. The model learns to generate samples by following the flow trajectories. \citet{campbell2024generative} show that the rates used in generative sampling can be obtained by training a neural network with parameters $\theta$, $p_\theta(x_1|x_t, t)$, to approximate the true denoising distribution using the standard cross-entropy loss,
\begin{align}\label{eq:fmloss}
\mathcal{L}_{\text{FM}} = \mathbb{E}_{x_1 \sim p_{\text{data}}(x_1),\, t\sim p(t),\, x_t \sim p_{t|1}(x_t|x_1)}\left[\log p_\theta(x_1|x_t, t)\right],
\end{align}
where $x_1$ is the clean data, $p(t)$, referred to as the noise schedule, is a distribution with full support on the interval $[0,1]$, and $x_t$ is the noised data sample from the forward noising process $p_{t|1}(x_t | x_1)$.\footnote{For all models, we follow the convention of flow-based models, where $t=1$ is the target (training) distribution, and $t=0$ is the noise distribution.} As shown in \citet{campbell2024generative}, the cross entropy loss can be understood as a simplification to the ELBO used to train continuous-time discrete diffusion models~\citep{campbell2022continuous}. Given such a trained denoising model $p_\theta(x_1|x_t, t)$, the (unconditional) rate matrix used in sampling is defined as
\begin{equation}\label{eq:dfm_rates}
    R^{\theta}_t(x_t, x_{t+dt}) = \mathbb{E}_{x_1 \sim p_\theta(x_1 | x_t)} \Big[\overrightarrow{R_t}(x_{t+dt}, x_{t} | x_1) \Big],
\end{equation} where $\overrightarrow{R_t}(x_{t+dt}, x_{t} | x_1)$ is the forward noising rate matrix induced by $p(x_t|x_1)$.

In practice, for computational tractability, an independent forward noising process is used for each dimension,  \mbox{$i \in \{1, 2, ..., D\}$}, of $x_1$. \citet{campbell2022continuous,campbell2024generative}~show that in this setting, the reverse process also factorizes per dimension. Thus, the denoising model, $p_\theta(x_1|x_t, t)$, need only learn a factorized distribution of the clean data and the loss becomes
\begin{align}
\mathcal{L}_{\text{FM}} &= \mathbb{E}_{x_1 \sim p_{\text{data}}(x_1), t\sim p(t),, x_t \sim p_{t|1}(x_t|x_1)}\left[\log\prod_{i=1}^D  p_\theta(x_1^i|x_t, t)\right].
\end{align}
Similarly, the rate matrix can be expressed per dimension as
\begin{equation}\label{eq:dfm_rates_per_dim}
    R^{\theta,i}_t(x_t, x_{t+dt}^i) = \mathbb{E}_{x_1^i \sim p_\theta(x_1^i | x_t, t)} \Big[\overrightarrow{R_t^i}(x_{t+dt}^i, x_{t}^i | x_1^i) \Big],
\end{equation} where $\overrightarrow{R_t^i}(x_{t+dt}^i, x_{t}^i | x_1^i)$ is the forward noising rate matrix applied to dimension $i$.

There are many methods for sampling from discrete state-space diffusion and flow models by integrating the learned rate matrices with numerical integration schemes~\citep{campbell2022continuous,campbell2024generative,gat2024discrete,peng2025path}. One simple approach is Euler integration with a discrete time step $\Delta t$, where transition probabilities are computed as $p(x_{t+\Delta t} = \tilde{x}|x_t = x) = \delta_{x,\tilde{x}} + R_t(x,\tilde{x})\Delta t$ and new states are sampled categorically at each step. Alternative methods include Gillespie's algorithm~\citep{gillespie1977exact,gillespie2001approximate} for exact simulation and $\tau$-leaping~\citep{campbell2022continuous} for approximate sampling that allows multiple dimension to transition simultaneously. The flow matching formulation also introduces a stochasticity hyperparameter $\eta$ that can be adjusted at inference time to control the randomness of the sampling path, which \citet{campbell2024generative} showed could improve sample quality.

A popular instantiation of both discrete state-spaces diffusion and flow models uses masking as the per-dimension noising process, where tokens in a sequence are progressively replaced with mask tokens during the forward process and the model learns to recover the original tokens during the reverse process. For example, the masking conditional flow can be defined as $p^{\text{mask}}(x_t^i|x_1^i) = \text{Cat}(t\delta\{x_1^i,x_t^i\} + (1-t)\delta\{?,x_t^i\})$, where $\delta\{i,j\}$ is the Kronecker delta which is $1$ when $i=j$, $0$ otherwise, and ``?'' denotes a mask token.
\citet{campbell2024generative}~showed that under the masking noise process, the unconditional rate matrix used in sampling defined in Equation~\eqref{eq:dfm_rates_per_dim} simplifies to

\begin{align}\label{eq:dfm_rates_mask}
    R^{\theta, i}_t(x_t, x_{t+dt}^i) = \frac{p_\theta(x_1^i \, | \, x_t, t)}{1-t}\delta\{x_t^i,?\}.
\end{align}

This formulation has connections to Masked Language Models and \aoamlong, which we will describe in detail next.

\subsubsection{Masked language models (MLMs)}

Masked Language Models (MLMs) loosely refer to a broad class of models that involves training a model to recover sequences of discrete tokens given the input where some of the tokens have been replaced by a special token known as the ``mask'' token. Herein we focus on a class of Masked Language Models (MLMs), that we will refer to as ``generative Masked Language Models''~\citep{ghazvininejad2019mask, chang2022maskgit}. The key distinction between generative Masked Language Models and their non-generative counterparts~\citep{devlin2019bert}, is that during training, generative MLMs mask tokens with a masking probability, $t \in [0,1]$, sampled from a distribution $p(t)$ that has full support on the interval $[0,1]$. This requirement is met by MLMs such as ESM3~\citep{hayes2025simulating}, but, importantly, is not met by BERT-style MLMs such as ESM2~\citep{lin2023evolutionary}, since these models use a fixed percentage of tokens that are masked at each step. Letting $x_t$ refer to the masked sequence with masking proportion denoted by $t$, we have that the generative masked language modeling objective function is

\begin{align}\label{eq:mlmloss}
    \mathcal{L}_{\text{MLM}} &= \mathbb{E}_{x_1 \sim p_{\text{data}}(x_1), t\sim p(t),\, x_t \sim p_{t|1}(x_t|x_1)}\left[\log\prod_{i=1}^D  p_\theta(x_1^i|x_t)\right],
\end{align}

There are various approaches that are used to sample from a generative MLM. Some common approaches include directly decoding all the tokens from the fully masked sequence, or decoding the sequence iteratively~\citep{hayes2025simulating}. For iteratively decoding, a commonly used heuristic to choose the decoding order which prioritizes the token to be unmasked according to some confidence~\citep{chang2022maskgit}. 

\subsubsection{Any-order autoregressive models (\aoamshorts)}
Any-order autoregressive models (\aoamshorts) differ from standard autoregressive models by training on arbitrary permutations of the sequence ordering~\citep{uria2014deep,hoogeboom2021autoregressive}. Consider a sequence $x = (x^1, x^2, ..., x^D)$ from our data distribution. In an \aoamlong, we sample a permutation $\sigma$ of indices $\{1, 2, ..., D\}$ and factorize the joint distribution as:
\begin{align}
p_\theta(x) = \prod_{i=1}^D p_\theta(x^{\sigma(i)} | x_1^{\sigma(<i)})
\end{align}
where $x^{\sigma(<i)}$ denotes $(x^{\sigma(1)}, x^{\sigma(2)}, ..., x^{\sigma(i-1)})$.
The \aoamshort\ objective is to maximize the expected log-likelihood over all possible permutations:

\begin{align}\label{eq:aoarm_loss}
    \mathcal{L}_{\text{\aoamshort}} &= \mathbb{E}_{x \sim p_{\text{data}}(x), \sigma \sim \text{Unif}(\text{Perm}(D))}\left[\sum_{i=1}^D \log p_\theta(x^{\sigma(i)} | x^{\sigma(<i)})\right]
\end{align}

Sampling from an \aoamshort~\cite{uria2014deep, yang2019xlnet, strauss2021arbitrary, hoogeboom2021autoregressive} involves iteratively constructing a sample, $x=(x^1, \ldots, x^D)$ of length $D$, starting from a fully masked sequence, by first sampling an ordering of positions uniformly at random, then sample iteratively from the conditional distribution of the model given this ordering.

\subsubsection{Conditional generation with guidance}
\label{sec:guidance}

Guidance is a useful technique that enables conditional generation with diffusion and flow models~\citep{sohl2015deep,song2020score,dhariwal2021diffusion}. Originally developed for continuous state-space models, guidance allows steering the generative process toward samples with desired properties without requiring conditional training. In continuous state-space diffusion models, the generation process involves sampling from a reverse stochastic differential equation (SDE) that gradually transforms noise into data~\citep{song2020score}. To condition this process on a desired property $y$, it has been shown that one can leverage Bayes' theorem to obtain the conditional score function~\citep{song2020score}:
\begin{equation}
    \nabla_{x_t}\log p_t(x_t|y) = \nabla_{x_t}\log p_t(x_t) + \nabla_{x_t}\log p_t(y|x_t)
\end{equation}

This formulation, termed \textit{classifier guidance}, combines the unconditional score $\nabla_{x_t}\log p_t(x_t)$ with the gradient of a predictor model $\nabla_{x_t}\log p_t(y|x_t)$. The strength of guidance can be controlled by introducing a \textit{guidance strength} parameter $\gamma$~\citep{dhariwal2021diffusion}:
\begin{equation}
    \nabla_{x_t}\log p^{(\gamma)}_t(x_t|y) = \nabla_{x_t}\log p_t(x_t) + \gamma \nabla_{x_t}\log p_t(y|x_t)
\end{equation}

This approach provides several advantages: (ii) it enables conditioning an unconditional model without retraining, (ii) the guidance strength $\gamma$ can be adjusted at inference time, and (iii) it allows composition of multiple guidance signals. \citet{ho2022classifier} later introduced \textit{classifier-free guidance}, which achieves similar results without an explicit classifier by combining a conditionally trained model and an unconditional model, which in practice are typically trained jointly. Classifier-free guidance has been shown to improve sample quality~\citep{ho2022classifier} compared to using only the conditionally trained model, and recent work on \textit{autoguidance} showed further improved sample quality by combining a trained model with a more inferior version of itself~\citep{karras2024guiding}. 

While guidance, either classifier-based, classifier-free or autoguidance, has proven extremely valuable for continuous state-space models, extending these concepts to discrete state-spaces poses unique challenges, since the score functions are not defined and these formulations cannot be directly applied. One can sidestep the need for a score function by instead returning to Bayes' theorem, which, in the context of diffusion on discrete state-spaces, dictates that we must effectively compute
\begin{equation}
\label{eq:transition_bayes}
p(x_{t+dt}|x_t, y) 
= 
\frac{
p(y|x_{t+dt}, x_t)
p(x_{t+dt}|x_t)
}{
\sum\limits_{x^{\prime}_{t+dt}}
p(y|x^{\prime}_{t+dt}, x_t)
p(x^{\prime}_{t+dt}|x_t)
},
\end{equation} where $x_{t+dt}$ and $x_{t}$ are separated by a infinitesimal time step $dt$, and $y$ is the desired property we wish to condition on. In a $D$-dimensional discrete state-space where each dimension has cardinality $S$, there are $S^D$ terms in the normalizing constant of Equation \ref{eq:transition_bayes}, rendering it generally intractable. Each term arises from the fact that, without constraints, any state can be reached from any other state. Intuitively, tractability can only be achieved if the number of terms is not exponential, thereby implying some constraints on the allowed state transitions. Next, we shall see how the continuous-time formulation of discrete state-space diffusion and flow models enables us to tractably compute the normalizing constant by imposing such constraints, without losing any model expressibility.

\subsubsection{Guidance in discrete state-spaces}
Here we provide an overview for \ourmethodICLR, a general method for guidance in discrete-space diffusion and flow models and refer the readers to~\citet{nisonoff2025unlocking} for detailed derivations and additional results. To enable guidance for discrete state-space models, we turn to continuous-time diffusion models and flow models in discrete state-spaces. As described above, these models are based on continuous-time Markov chains (CTMCs) and therefore parameterize transition rates $R_t(x, \tilde{x})$ rather than score functions. These models also assume that the noising processes in each dimension are independent of one another (but not necessarily identical), which has the important consequence that the
probability that two or more dimensions of the current state jump (transition) at exactly the same
time is zero, which makes these models tractable~\citep{campbell2022continuous,campbell2024generative}. It turns out that the same assumption also makes guidance tractable for these models~\citep{nisonoff2025unlocking}.

To perform predictor guidance in this setting, we derived a principled approach to adjust these rates based on a predictor model $p(y|x_t)$ that evaluates how likely a state $x_t$ is to have the desired property $y$~\citep{nisonoff2025unlocking}. Specifically, we showed that the conditional rates can be obtained as:
\begin{equation}\label{eq:conditional_rates}
    R_t(x, \tilde{x}|y) = \frac{p(y|\tilde{x}, t)}{p(y|x, t)} R_t(x, \tilde{x})
\end{equation}
for any state $\tilde{x}$ that differs from $x$ in at most one dimension. This formulation directly implements Bayes' theorem at the level of transition rates, modulating transitions toward states that have higher likelihood of possessing the desired property.

Similar to continuous state-space guidance, we can control the strength of guidance with a guidance strength parameter $\gamma$:
\begin{equation}\label{eq:conditional_rates_with_gamma}
    R_t^{(\gamma)}(x, \tilde{x}|y) = \left[\frac{p(y|\tilde{x}, t)}{p(y|x, t)}\right]^{\gamma} R_t(x, \tilde{x})
\end{equation} or derive predictor-free guidance discrete state-spaces:
\begin{equation}
    R_t^{(\gamma)}(x, \tilde{x}|y) = R_t(x, \tilde{x}|y)^{\gamma}R_t(x, \tilde{x})^{1-\gamma}
\end{equation}

For computational efficiency, especially with large state-spaces, we introduced Taylor-approximated guidance (TAG)~\citep{nisonoff2025unlocking}, which approximates the likelihood ratio using a first-order Taylor expansion inspired by \citet{grathwohl2021oops}:
\begin{equation}
    \log \frac{p_{\phi}(y|\tilde{x}, t)}{p_{\phi}(y|x, t)} \approx (\tilde{x} - x)^T \nabla_x \log p_{\phi}(y|x, t)
\end{equation}
requiring only one forward and one backward pass of the predictor model instead of $D \times (S - 1) + 1$ forward passes, making estimation of the guide-adjusted rates $O(1)$ rather than $O(D \times S)$. With multiple particles, TAG can be further improved by acting as a proposal distribution in a Sequential Monte Carlo (SMC) algorithm similar to~\citet{wu2023practical}.  While there are other recently proposed approaches for guiding discrete diffusion models~\citep{li2024derivative,schiff2025simple,singhal2025a}, \ourmethod~has the benefit of being both exact and general, applicable to both diffusion and flow-based models, making it a general and flexible framework for guiding a broad class of discrete state-space generative models (for discussion on related work, see Section~\ref{sec:si-related-work}).

\subsection{Unified connections of training objectives for masked-based generative models}
\label{sec:method-training}

Connections among masked discrete diffusion models, Masked Language Models (MLM) and \aoamlongs (\aoamshorts) were initially pointed out by \citet{austin2021structured} and \citet{hoogeboom2021autoregressive}. Recent works~\citep{shi2024simplified,sahoo2024simple,ou2025your,campbell2024generative,gat2024discrete} further clarified these connections, including the connections to CTMC-based formulation of diffusion and flow models~\citep{campbell2022continuous,lou2023discrete,campbell2024generative}, and showed that the continuous-time ELBO used to train masked diffusion models is essentially a weighted sum of cross entropy losses predicting the masked states from unmasked positions, where the proportion of masked positions ranges from 0 (complete masking) to 1 (no masking). For particular choices of the noise schedules, these training objectives can be shown to be equivalent to the training objective of \aoamlong s~\citep{uria2014deep}, (generative) Masked Language Models~\citep{austin2021structured}, and discrete flow-matching models~\citep{campbell2024generative,gat2024discrete} under a masking process. Since ESM3~\citep{hayes2025simulating} and ProteinMPNN~\citep{dauparas2022robust} were trained under the MLM and \aoamshort objectives respectively, and owing to the sampling equivalency between \aoamshorts and diffusion models, we can view ESM3 and ProteinMPNN as discrete state-space diffusion and flow models, and sample from them as such, leveraging the generative rate matrices from Equation~\eqref{eq:dfm_rates} induced by the trained denoising model, $p_\theta(x_1 | x_t)$. This formulation not only provides a principled way to sample from these models, but also offers additional sampling flexibility, including, but not limited to, the usage of different noise schedules during sampling from those used during training, and corrector-sampling~\citep{campbell2024generative,gat2024discrete}, which have been shown to sometimes improve sampling results compared to prior heuristics~\citep{chang2022maskgit} for sampling from generative Masked Language Models~\citep{gat2024discrete}. In addition, this formulation unlocks the capacity to condition the generative process via guidance, which we introduce in a later section.

Here, we provide a self-contained derivation of the equivalence between the loss functions of masked flow matching (MFMs), Masked Language Models (MLMs), and \aoamlongs (\aoamshorts).

\subsubsection{Equivalence between MFMs and MLMs}

Recall from Equation~\eqref{eq:fmloss} that the flow matching loss is
\begin{align*}
\mathcal{L}_{\text{FM}} &= \mathbb{E}_{x_1 \sim p_{\text{data}}(x_1), t\sim p(t),\, x_t \sim p_{t|1}(x_t|x_1)}\left[\log\prod_{i=1}^D  p_\theta(x_1^i|x_t, t)\right].
\end{align*}
Our goal is to demonstrate that this loss function is equivalent to the loss function of generative Masked Language Models (MLMs), which, recalling from Equation~\eqref{eq:mlmloss}, can be written as:
\begin{align*}
\mathcal{L}_{\text{MLM}} &= \mathbb{E}_{x_1 \sim p_{\text{data}}(x_1), t\sim p(t),\, x_t \sim p_{t|1}(x_t|x_1)}\left[\log\prod_{i=1}^D  p_\theta(x_1^i|x_t)\right],
\end{align*}
where we can see that the key difference between Equation~\eqref{eq:mlmloss} and Equation~\eqref{eq:fmloss} is that the latter has a denoising neural network with a dependence on $t$. We will now show that when a masking noise process is used for flow matching, the denoiser's dependency on $t$ can be dropped, making the two loss functions equivalent.

Let \mbox{$x_1 = (x_1^1, x_1^2, ..., x_1^D)$} be our original sequence from the data distribution. For a masking noise process we define binary masking variables $m_t^i \in \{0,1\}$ for each position $i$ at time $t$, where $m_t^i = 1$ indicates token $i$ is masked at time $t$ and $m_t^i = 0$ indicates token $i$ is not masked at time $t$.
Using `?' to denote a masked state, the state $x_t$ is then defined as:
   \begin{align}
   x_t^i = \begin{cases} 
   x_1^i & \text{if } m_t^i = 0 \\
   \text{?} & \text{if } m_t^i = 1.
   \end{cases}
   \end{align}
Let us additionally define:
\begin{enumerate}
\item $M_t = \{i : x_t^i = \text{?}\}$ as the set of masked positions at time $t$.
\item $U_t = \{i : x_t^i \neq \text{?}\}$ as the set of unmasked positions.
\end{enumerate}
For any position $i \in U_t$, we know with certainty that $x_1^i = x_t^i$ by the definition of our forward process.
%
%
The key observation is that the time variable $t$ only influences which tokens are masked (the masking pattern), but once we observe $x_t$, we have complete information about which positions are masked and which are not. The time variable, therefore, provides no additional information needed to predict the clean data.
Using a mathematical argument first proposed by~\citet{gat2024discrete}, we can prove this as follows:
\begin{align}
p(x_1 \mid x_t, t) &= \frac{p(x_t, t \mid x_1)p(x_1)}{\sum_{\tilde{x}_1}p(x_t, t \mid \tilde{x}_1) p(\tilde{x}_1)}\\
&= \frac{\cancel{\big[\prod_{i\in M_t} (1-t)\prod_{i\in U_t}t\big]}\big[\prod_{i\in U_t} \delta\{x_t^i, x_1^i\}\big] p(t) p(x_1)}{\sum_{\tilde{x}_1}\cancel{\big[\prod_{i\in M_t} (1-t)\prod_{i\in U_t}t\big]}\big[\prod_{i\in U_t} \delta\{x_t^i, \tilde{x}_1^i\}\big] p(t) p(\tilde{x}_1)} 
 \label{eq:mlm_time_indept_line2}\\
&= \frac{\big[\prod_{i\in U_t} \delta\{x_t^i, x_1^i\}\big] p(x_1)}{\sum_{\tilde{x}_1}\big[\prod_{i\in U_t} \delta\{x_t^i, \tilde{x}_1^i\}\big] p(\tilde{x}_1)} \label{eq:mlm_time_indept_line3}\\
&= p(x_1 \mid x_t) \label{eq:mlm_time_indept_line4}.
\end{align}
where Equation~\eqref{eq:mlm_time_indept_line2} uses the factorization over the dimensions and analyze separately the two cases of whether a position is masked or unmasked, and Equation~\eqref{eq:mlm_time_indept_line3} shows that $p(x_1 \, | \, x_t, t)$ does not depend on $t$, giving us Equation~\eqref{eq:mlm_time_indept_line4}. Thus, we can rewrite the masking flow matching loss function as:
\begin{align}
\mathcal{L}_{\text{FM}} &= \mathbb{E}_{x_1 \sim p_{\text{data}}(x_1),\, t\sim p(t),\, x_t \sim p_{t|1}(x_t|x_1)}\left[\log \prod_{i=1}^D p_\theta(x_1^i|x_t,t)\right]\\*
&= \mathbb{E}_{x_1 \sim p_{\text{data}}(x_1),\, t\sim p(t),\, x_t \sim p_{t|1}(x_t|x_1)}\left[\log \prod_{i=1}^D p_\theta(x_1^i|x_t)\right]\\*
&= \mathcal{L}_{\text{MLM}}.
\end{align}

\subsubsection{Equivalence between MFMs and \aoamshorts}

Recall that from Equation~\eqref{eq:aoarm_loss}
The \aoamshort\ objective is to maximize the expected log-likelihood over all possible permutations:
\begin{align*}
\mathcal{L}_{\text{\aoamshort}} &= \mathbb{E}_{x \sim p_{\text{data}}(x), \sigma \sim \text{Unif}(\text{Perm}(D))}\left[\sum_{i=1}^D \log p_\theta(x^{\sigma(i)} | x^{\sigma(<i)})\right]
\end{align*}

We will show that this objective can be equivalently expressed as a masking discrete flow matching loss:
\begin{align}
\mathcal{L}_{\text{\aoamshort}} &= \mathbb{E}_{x_1 \sim p_{\text{data}}(x_1), \sigma \sim \text{Unif}(\text{Perm}(D))}\left[\sum_{i=1}^D \log p_\theta(x_1^{\sigma(i)} | x_1^{\sigma(<i)})\right]\\
    &= D \mathbb{E}_{x_1 \sim p_{\text{data}}(x_1), \sigma \sim \text{Unif}(\text{Perm}(D)), i \sim \text{Unif}(\{1, ..., D\}) }\left[\log p_\theta(x_1^{\sigma(i)} | x_1^{\sigma(<i)})\right]\\
    &= D \mathbb{E}_{x_1 \sim p_{\text{data}}(x_1), i \sim \text{Unif}(\{1, ..., D\}),\sigma \sim \text{Unif}(\text{Perm}(D))}\left[\log p_\theta(x_1^{\sigma(i)} | x_1^{\sigma(<i)})\right]\label{eq:perm}\\
    &= D \mathbb{E}_{x_1 \sim p_{\text{data}}(x_1), t \sim \text{Unif}([0,1]),x_t \sim p_{t|1}(x_t|x_1)}\left[\log p_\theta(x_1 \mid x_t)\right]\label{eq:permtodfm}\\
    &= D \mathcal{L}_{\text{FM}},
\end{align}
where we go from Equation~\eqref{eq:perm} to Equation~\eqref{eq:permtodfm} by recognizing that randomly picking the number of positions to mask and sampling a random permutation is equivalent to picking a random time and sampling the forward noising process for that time. Thus, we have shown that the \aoarlong\ modeling loss is equivalent, up to a constant, to a masking discrete flow matching loss with a uniform noise schedule.

\subsection{Simplified and faster sampling for flow matching and diffusion models with masked noise processes}
\label{sec:method-sampling}


In previous sections, we showed an equivalence between the loss functions used to train discrete flow-matching models (DFMs), discrete diffusion models (DDMs), \aoamlongs (\aoamshorts) and masked language models (MLMs), specifically when the DFMs and DDMs used a masked noise process, which we referred to as masked flow matching models (MFMs) and masked diffusion models (MDMs). In this section, we show a relationship between sampling algorithms used for these models.

\subsubsection{Intuitive overview}
We first provide an intuitive overview of the main technical result. We will show that instead of sampling from a CTMC---the required operation for sampling from a CTMC-based DFM---by way of integration in time using algorithms such as Euler, Gillespie, or $\tau$-Leaping (of which only Gillespie is exact), that one can instead use a much simpler algorithm to exactly sample from the CTMC. Specifically, one can use a simpler sampling algorithm that is normally used to sample from an \aoamshort. In particular, we will be able exactly sample from a CTMC underlying a DFM by independently sampling all transition times up-front---meaning before any state transitions have been sampled---and only then sampling all the state transitions. This lies in stark contrast to general CTMC sampling in which sampling time and state transitions must be interleaved. However, it is exactly how one would sample from an \aoamshort, where one first samples a "decoding order``, and then autoregressively unmasks each position in turn according to this order, conditioning on all previously unmasked positions. Consequently, there is a constructive proof below for which \aoamshort \ to sample from so as to sample from the desired CTMC associated with a DFM. As a corollary, we also show that one only needs the \emph{ordering} of the transition times to sample the state transitions. 

In order to achieve correctness of this simplified sampling algorithm, we will need to make use of commonly used modeling choices made for CTMC-based DFMs, such as that these CTMCs follow an interpolation schedule as prescribed by the flow matching set-up. These choices enable us to derive a joint probability distribution for the CTMC over time and sequences that factorizes in the necessary way, namely into a term containing solely transition (holding) times, and another containing solely state transitions.

So far we have discussed only DFMs, but through other work that connects DFMs as a generalization of DDMs (see section below on the related work specific to sampling equivalances), these diffusion models inherit all the work of sampling from a DFM. We will discuss related work for DFMs below. 
Because we showed earlier that the training objective for MLMs is the same as for \aoamshort \ and MDMs, we can similarly appeal to our earlier reasoning of interpreting an MLM as either an \aoamshort \ or DFM, and correspondingly also sample from it in the same as described herein.

In addition to showing simplified and faster sampling for CTMC-based DFMs, we also, to the best of our knowledge, show for the first time how to use this new viewpoint to provide an alternative derivation for the guidance results for CTMC-based diffusion and flow matching models presented in~\citet{nisonoff2025unlocking}. 

\subsubsection{Related work on sampling algorithms}

A number of existing works have discussed the topics herein, but these works are for the specific case of diffusion models, namely CTMC-based MDMs. We prove similar results for the case of flow matching models, namely CTMC-based MFMs.
We also synthesize these results from prior work to give a coherent, constructive proof for the equivalence of paths sampled by MFMs and \aoamshort.

As mentioned above, it has been previously shown that sampling exactly from the CTMC of a MDM is equivalent to sampling unmasking times for each position upfront and then using this to define the decoding order to sample from an \aoamshort~\cite{peng2025path, chen2024fast, hoogeboom2021autoregressive, zheng2024masked, ou2025your, shi2024simplified, sahoo2024simple}. The discussions in \citet{ou2025your} and \citet{peng2025path} are most similar to this section in spirit, but again, analyze the equivalence for MDMs specifically. Furthermore, neither works give a unified proof of the equivalence in sampling which allows complete removal of the notion of time when sampling from an MFM, showing equivalence for the full path and not just marginal distributions. Appendix H of \citet{campbell2024generative} further shows a correspondence between DDMs and DFMs where DFMs define a large class of rate matrices for which DDMs correspond to a specific choice. We believe this holds in the general case, such that our results for MFMs should be reducible to the previously cited results for MDMs.

Previous works by \citet{gat2024discrete, campbell2024generative} on MFMs have mentioned that the conditional distribution parameterized by the neural network is independent of time (see \cite{sahoo2024simple, shi2024simplified, peng2025path} for the analogous MDM results). 
This is a key insight that we use in this section as well. However, as used by prior work, we do not believe this insight suffices to show the AO-AR sampling method for MDMs samples the desired CTMC, nor does it directly extend the MDM proofs to MFMs. Herein, we provide a proof of the exact equivalence, which additionally gives a method to sample the transition times. Integrating these findings, Proposition~\ref{prop:main} gives us a constructive proof to sample from the desired CTMC using either an MFM or an AO-ARM.

Our proofs hold for the common MFM formulation from \citet{campbell2024generative, gat2024discrete}, which we define again in the section below. Many of the results can readily be extended to other noise distributions, such as uniform substitutions, or to the use of stochasticity and corrector sampling. However, in those cases, a full reduction to AO-ARM sampling is not possible. Instead, the calculations yield an efficient Gillespie sampling algorithm, which we do not cover in detail.

\subsubsection{Background and notation}
\label{sec:sampling-background}

In this section, we briefly review the relevant concepts for \aoamshort s, CTMCs, and flow-matching models. Although some of these were covered earlier, here we quickly remind the reader so as to also introduce notation convenient specifically to this section.


Sampling from an \aoamshort~\cite{uria2014deep, yang2019xlnet, strauss2021arbitrary, hoogeboom2021autoregressive} involves iteratively constructing a length-$D$ sample $X=(X^1, \ldots, X^D)$ starting from a fully masked sequence, in two steps. First, a permutation, $\sigma$, of positions, \mbox{$1, 2, \ldots, D$} is sampled uniformly at random. This permutation dictates the order of positions to unmask as the sample is constructed over $D$ iterations (one for each position), sometimes called the ``decoding order". Given the permutation ordering, $\sigma(1,2, \dots, D)$, the model progressively unmasks each position according to this ordering, one at a time, autoregressively. 
That is, at each step, the model samples from the conditional distribution, \mbox{$X^{\sigma(i)} \sim p(X^{\sigma(i)} | x^{\sigma(<i)})$}. These distributions are the core object learned when training an \aoamshort, typically by a neural network. We denote the set of progressively unmasked sequences generated during the autoregressive sampling as $\bm{X}=(X_1, \ldots, X_D)$, where bold indicates a ordered set of random variables, and subscript denotes the $i^\text{th}$ unmasking iteration. Note that both the dimension of $X$ is $D$, and that there are $D$ unmasking operations to generate it, which can be confusing. We sometimes refer to $\bm{X}$ as the ``path''. 
We can also represent the path compactly as a tuple of the decoding order and the final sequence, denoted  $\bm{X}=(\sigma, X)$.

Formally, we write the joint probability of the permutation ordering, $\sigma$, and the fully unmasked sequence, $X$  (together sometimes called the ``decoding path'')  as
\begin{equation}
    \Po(\sigma, X)=U[S_D](\sigma) \prod_{i=1}^D p \left( X^{\sigma(i)} | x^{\sigma(<i)} \right), \label{eq:p-oa}
\end{equation}
 where $S_D$ is the group of permutations over sequences of length $D$, $U[S_D](\sigma)$ indicates $\sigma$ is sampled uniformly at random from $S_D$, $X^{\sigma(i)}$ is the random variable for the value of position $\sigma(i)$, and $x^{\sigma(<i)}$ is the set of $i-1$ observed positions that are unmasked at the $i$-th decoding step.


Continuous time Markov chains (CTMCs) describe a sequence of Markovian transitions between discrete states as a continuous time process. For a detailed introduction, including path probability calculations like the ones done here, see \citet{del2017stochastic}. When a CTMC ``jumps'' from a state $x$ to state $\tilde{x}$ at time $\tau$, we say that the CTMC experienced a state transition from $x$ to $\tilde{x}$ at jump time $\tau$. Since the transition is instantaneous, we define the new state to be the state at time $\tau$: $X_\tau=\tilde{x}$. Accordingly, we say the previous state $x$ is the left limit of the CTMC at time $\tau$, meaning it is the value of the CTMC going right up to the jump time. We denote the left limit $X_{\tau_i}^{-}=\lim_{t\nearrow \tau_i} X_t = X_{\tau_{i-1}}$. 

We will often omit the random variables when symbol for the observed value is the corresponding lower-case letter. For example, $\Pm(x_t)=\Pm(X_t=x_t)$ refers to the probability of observing the random variable $X_t$ taking on value $x_t$. Sometimes, when the variable for the observed value does not correspond to the random variable it is assigned to, we will write both explicitly. For example, in writing $\Pm(x_t | X_t^-=x_s)$ we refer to the probability that the masked flow matching model (MFM) transitions from $x_s$, the state it reached at a previous time $s$, to $x_t$ at time $t$. Recall that $X_t^-=x_s$ indicates that the state $x_s$ is the left limit $X_t^{-}$ of the state at $t$.

Finally, the rate matrix from $x$ to $\tilde{x}$ at time $t$ is denoted $R_t(x, \tilde{x}$). The rate matrix is generally time-dependent even when conditioned on the state. In flow-matching models, transition rates between states generally accelerate in time to ensure the process terminates by time $t=1$. Accordingly, the CDF of the jump time, is an integral of time-dependent rates and does not always have a closed-form expression. It is written as
\begin{equation}
\mathbb{P}\left(T_k<t \mid T_{k-1}\right)=1-\exp \left(\int_{s=T_{k-1}}^{s=t} R_s\left(X_s^{-}, X_s^{-}\right) \mathrm{d} s\right). \label{eq:holding-time}
\end{equation}
As one can see, this integrates the rate of the CTMC not undergoing a state transition at time $s$.


Discrete flow matching models (DFMs) were formulated by \citet{campbell2024generative,gat2024discrete} to define a CTMC whose marginal distribution over states interpolates between a noise distribution $p_\text{noise}$ at time $0$ and a data distribution $p_\text{data}$ at time $1$.

Formally, this requires that the rates learned by our DFMs are solutions to the Kolmogorov Forward Equation (KFE) that satisfy the interpolation $X_t=\kappa_t X_1 + (1-\kappa_t)X_0$ subject to $X_1 \sim p_\text{data}$ and $X_0 \sim p_\text{noise}$. The \textit{interpolation schedule} $\kappa_t$ sets the probability of finding $X_t$ in the denoised state $X_1$ at each time $t$. The interpolation schedule can be freely defined by the user so long as $\kappa_0=0$ and it monotonically increases to $\kappa_1=1$. Note, for the MFMs analyzed in this section, this means we can interpret $\kappa_t$ as the CDF of the jump time for a position, $\kappa_t=\Prob(\tau<t)$. This, in turn, implies that the time derivative is the PDF of the jump time, $\dot{\kappa_t}=\Prob(\tau=t)$.

The CTMC underlying a DFM is used to sample transport paths to map samples from $p_\text{noise}$ to final samples from $p_\text{data}$, the distribution the DFM seeks to model. The transport is simulated by sampling a sequence $X_0 \sim p_\text{noise}$ (all masked in the case of a masked flow-matching model, MFM) and numerically integrating the CTMC by sampling
\begin{equation*}
X_{t+\Delta t} \sim p_{t+dt|dt}(\cdot|x_t) \approx \delta_{x_t}(\cdot) + R(x_t, \cdot) \Delta t    
\end{equation*}
until we reach $X_1\sim p_\text{data}$.

The formulation of DFMs from \citet{campbell2024generative, gat2024discrete} which was described above, makes three important assumptions:
\begin{enumerate}
    \item the conditional rates of the CTMC factorize per position, when conditioned on the denoised sequence $X_1$,
    \item the rates of multiple positions changing in the CTMC simultaneously is $0$,
    \item the interpolation schedule $\kappa_t$ is independent of the partially denoised sequence $X_t$.
\end{enumerate}
Given these conditions for the DFM, the rates for each position $i$ can be found analytically to be
\begin{equation}
    R_t^{i}(x_t^{i}, \tilde{x})=\frac{\dot{\kappa}_t}{1-\kappa_t} (p_{1|t}(X_1^{i}=\tilde{x}|X_t=x_t)-\delta_{x_t^{i}}(\tilde{x}^{i})), \label{eq:rates-param}
\end{equation}
where $p_{1|t}(X_1^{i}=\tilde{x}|X_t=x_t)$ is the key object learned by a machine learning model. Note that this expression for the rates factors into two terms. The first is the probability of a transition occurring at time $t$, \mbox{$\Pm(\tau=t|\tau \geq t)=\frac{\dot{\kappa}_t}{1-\kappa_t}$}. The second, $p_{1|t}(X_1^{i}=\tilde{x}|X_t=x_t)-\delta_{x_t^{i}}(\tilde{x}^{i})$, ensures that if there is a jump (when $\delta_{x_t^{i}}(\tilde{x}^{i})=0$), the state will be sampled from $p_{1|t}$. This will be discussed in more detail in the proofs of the lemmas.

Similar to AO-ARMs, where we represent the sequence of partially unmasked states as the tuple $(\sigma, X)$, we represent the sequence of transitions in a sampled MFM path as the tuple $(\bm{\tau}, \bm{X})$. Here, $\bm{\tau}=(\tau_0=0, \tau_1, \ldots, \tau_D<1)$ is the sequence of transition times sampled by the CTMC, in increasing order, from $0$ to $1$. As before, $\bm{X}=(X_0, X_{\tau_1}, \ldots, X_{\tau_D})$ is the sequence of states at each of those times. Note that $X_0$ is the all masked state and the fully unmasked state is $X_1=X_D$. 

As in the AO-ARM case, we say that $\sigma$ is the permuation that maps from transition times to the position of the sequence that changed. Concretely, the $i$-th transition occurs at time $\tau_i$ and causes the $\sigma(i)$-th position to be unmasked. Like AO-ARMs, this means we can also represent $\bm{X}$ as the tuple $(\sigma, X_{\tau_D})$. We can also use $\sigma$ to sort the transition times of each position $\bm{T}=(T_1, \ldots, T_D)$ into the transition times of the CTMC: $\sigma(\bm{T})=\bm{\tau}$.

Below, we have included an expression for the probability of a CTMC path sampled using an MFM. This particular formula applies to all CTMCs and does not readily admit an algorithmic implementation, we just include it here for completeness:
\begin{align}
&\Pm\left(\bm{\tau}, \bm{X}\right) \notag \\
= &\Pm(\tau_0, x_{\tau_0})
    \times \prod_{i=1}^D \Pm(\tau_i, x_{\tau_i} | \tau_{i-1}, x_{\tau_{i-1}}) 
    \times \Pm(\tau_{D+1} > 1 | \tau_D, x_{\tau_D}). \notag
\end{align}

\subsubsection{MFM CTMCs use the minimal number of transitions}
\citet{campbell2024generative, gat2024discrete} shows that the rates defined above (Eq.~\ref{eq:rates-param}) make the minimal number of jumps needed to denoise the sequence. In the case of an MFM, this means we have exactly $D$ transitions, one to unmask each position. However, there is an infinite class of rate matrices which also solve the KFE and DFM boundary conditions, which do not have this minimal jumps property. 

Non-minimal rate matrices can be formulated as those with non-zero stochasticity as in \citet{campbell2024generative}, and can be equivalently formulated through corrector sampling \citet{gat2024discrete}. Our proofs explicitly handle the minimal case, but the calculations can be readily extended to the case of non-zero stochasticity. We do note, however, that when doing \textit{exact} sampling, rate matrices with non-zero stochasticity yield the same time-marginals as the minimal rates (by construction). Generally, non-zero stochasticity is only useful when $\Delta t$ is large enough during numerical integration that multiple positions are sampled at once. In such settings, stochasticity is believed to correspond to implicit ``equilibration'' steps that reduce drift in the simulated CTMC.

\subsubsection{Statement of exact MFM sampling using \aoamshorts}
We show that paths from a masked flow matching model's (MFM's) CTMC can be sampled \textit{exactly} by sampling transition times uniformly at random and sampling state transitions from an \aoamshort\ independently. The key fact underlying our proof of Prop.~\ref{prop:main} is that the transition times and states for an MFM are conditionally independent of each other given a decoding order. In particular, the distribution over paths factorizes into two terms. The jump time component reduces to a product of $D$ independent uniform distributions, one for each position; the state transition component yields a product of denoising terms that correspond to the conditional distributions sampled by an \aoamshort . These preceding two statements are proved in Lemmas~\ref{lemma:transition-time-decomposition} and~\ref{lemma:transitions-are-same}, respectively, in the section below. Finally, Cor.~\ref{corollary:sample-permutation} extends the claim of Prop.~\ref{prop:main} to shows how to construct a sampler where the decoding order is sampled upfront and the transition times and states are truly sampled independently. The detailed derivations and proofs are included in Supplementary Information section~\ref{sec:si-proofs} ``Detailed derivations and proofs''. Below we state the main results.

\begin{proposition} \label{prop:main}
The probability of sampling a path $(\bm{\tau}, \bm{X})$ from an MFM is equivalent to the probability of
\begin{enumerate}
    \item sampling transition times for each position in the sequence $\bm{T}=(T^1, \ldots, T^D)$, uniformly at random (u.a.r), from the interval [0,1],
    \item sorting $\bm{T}$ in increasing order to obtain a permutation $\sigma$, which maps the unmasking steps to the position to be unmasked at that step,
    
    \item setting $\bm{\tau}=\sigma(\bm{T})$, where $\sigma$ acts as $\tau_i=T^{\sigma(i)}$
    
    \item sampling the sequence of states $\bm{X}$ from a \aoamshort\ using order $\sigma$ with the MFM denoiser: \mbox{$p(X^{\sigma(i)}=\tilde{x}|X^{\sigma(<i)}=x) :=p_{1|t}(X_1=\tilde{x}|X_t^-=x)$}.
\end{enumerate}
Below, $\Po$ refers to the distribution induced by the resulting \aoamshort\ in point (4) above. We call this resultant model the ``equivalent'' \aoamshort\ for our MFM. Note that we can also obtain the equivalent MFM for an \aoamshort\ by reversing the assignment of conditional distributions.

Altogether, these four points correspond to following equation:
\begin{equation}
\mathbb{P}_{\mathrm{MFM}}\left(\bm{\tau}, \bm{X}=(\sigma, X_{\tau_D})\right) \\
=U^D_{[0, 1]}(\bm{T}=\sigma^{-1}(\bm{\tau}))\mathbb{P}_{\mathrm{OA}}\left(X=X_{\tau_D} | \sigma\right),  \notag
\end{equation}
where $\sigma^{-1}$ acts as $T^{j}=\tau_{\sigma^{-1}(j)}$ and $U^D_{[0, 1]}$ is the product of $D$ independent uniform distributions on the interval $[0,1]$.
\end{proposition}
We use this notation to show that the density of sampling $\bm{\tau}$ and $\bm{X}$ from an MFM is the same as the probability of (i) sampling $\bm{T}$ u.a.r, inducing a associated decoding order $\sigma$ and $\bm{\tau}$ as a result, and (ii) using that decoding order to recover $\bm{X}$ from an \aoamshort, as described at the beginning of the proposition.

\begin{corollary}
\label{corollary:sample-permutation}
An MFM can be exactly sampled by first sampling a permutation $\sigma \sim U[S_D]$ and then sampling the sequence of states for each transition from an \aoamshort. The transition times can then be sampled uniformly at random and assigned to positions according to the sampled permutation.
\end{corollary}

Together with Prop.~\ref{prop:main} this completes our claim of conditional independence of the times and states given the unmasking ordering. In Cor.~\ref{corollary:sample-permutation} we can see plainly that the density of paths under the MFM CTMC is the same as sampling from the times and states independently once we've sampled the permutation upfront. Furthermore, the corollary shows that to get the likelihood of an output sequence under the CTMC's final distribution, the specific jump times can be marginalized out easily, but the decoding order still makes getting likelihoods intractable.

\subsubsection{Implications for guidance}\label{si_sec:exact_guidance_derivation}
In this section, we will use the equivalence between \aoamshort\ and MFMs to show a simplified proof of our guidance method in \citet{nisonoff2025unlocking}. The calculation done here pertains specifically to MFMs, but can be readily extended to DFMs more broadly. The core underlying intuition is that the jump times and state transitions depend on the current sequence only through $p_{1|t}^\theta$. This is because the interpolation schedule is independent of the state. Therefore, when using guidance, only this distribution needs to be modified.

When guiding an MFM, we condition the CTMC such that it samples from $p_\text{data}(x|y)$ instead of $p_\text{data}(x)$. This corresponds to setting $\Pm'(X_1)=\Pm(X_1|y)$ and solving for the appropriate rate matrices $R'_t$. Due to Prop.~\ref{prop:main}, we can instead analyze how to condition the equivalent \aoamshort\ and then push this change back into our MFM.

\begin{proof}[Proof Sketch]
The equivalent \aoamshort\ samples the unmasked values of positions using $p^\theta(x^{\sigma(i)}|x^{\sigma(<i)}):=p_{1|t}^\theta \left(X_1^{\sigma(i)}=x^{\sigma(i)}|X_t^{-}=x^{\sigma(<i)} \right)$. To condition, we desire to sample from $p(x^{\sigma(i)}|x^{\sigma(<i)}, y)$ instead. All we require is an expression for this distribution using our existing conditional model, which we get using Bayes rule as follows:
\begin{align*}
p'^\theta(x^{\sigma(i)}|x^{\sigma(<i)}) &=p^\theta(x^{\sigma(i)}|x^{\sigma(<i)}, y) \\
&= \frac{p(y|x^{\sigma(i)}, x^{\sigma(<i)}) p^\theta(x^{\sigma(i)}|x^{\sigma(<i)})}{p(y|x^{\sigma(<i)})} \\
&= \frac{p(y|x^{\sigma(i)}, x^{\sigma(<i)})}{p(y|x^{\sigma(<i)})} p^\theta_{1|t}(X_1^{i}=x^{\sigma(i)}|X_t=x^{\sigma(<i)})
\end{align*}
Above, $p(y|x^{\sigma(i)}, x^{\sigma(<i)})$ is identical to the noisy predictor trained in regular classifier guidance for an MFM. This suffices to recover the exact guidance methodology in \citet{nisonoff2025unlocking}. Using the conditioned \aoamshort, we not write the equivalent MFM. There is one subtlety in this translation. Because we are now finding the general form of the rates absent an already sampled set of transition times, and therefore decoding order, we must choose an appropriate order to use the guided \aoamshort's conditional distribution as the guided MFM's probability denoiser. 

Say we are looking at the rate for the position $j$ at a time leading up to the $i$-th jump, such that $\tau_{i-1}<t\leq \tau_i$. We choose $\sigma$ such that $j=\sigma(i)$ and $\sigma(<i)$ corresponds to the positions currently unmasked in $x_t$. With this construction, we find
\begin{align*}
{R'}_t^{j}(x_t^j, \tilde{x}^j) = \frac{\dot{\kappa}_t}{1-\kappa_t} \bigg(&p'^\theta_{1|t}(X_1^{j}=\tilde{x}^j|X_t=x_t) - \delta_{x_t^{j}}(\tilde{x}^{j})\bigg) \\
= \frac{\dot{\kappa}_t}{1-\kappa_t} \bigg(&p'^\theta(X^{\sigma(i)}=\tilde{x}^j|X^{\sigma(<i)}=x_t) - \delta_{x_t^j}(\tilde{x}^j)\bigg) \\
=\frac{\dot{\kappa}_t}{1-\kappa_t} \bigg(&\frac{p(y|X^{\sigma(i)}=\tilde{x}^j, X^{\sigma(<i)}=x_t)}{p(y|X^{\sigma(<i)}=x_t)} p^\theta_{1|t}(X_1^{i}=\tilde{x}|X_t=x_t) \\
&-\delta_{x_t^j}(\tilde{x}^j)\bigg). \\
\intertext{In the case that $x_t^i = \tilde{x}^i$, $\frac{p(y|X_1^{i}=\tilde{x}, X_t=x_t)}{p(y|X_t=x_t)}$ simplifies to $1$, and we can factor out the added ratio of conditional probabilities to get }
=\frac{\dot{\kappa}_t}{1-\kappa_t} &\frac{p(y|X^{\sigma(i)}=\tilde{x}^j, X^{\sigma(<i)}=x_t)}{p(y|X^{\sigma(<i)}=x_t)} \\
&\times \bigg(p^\theta_{1|t}(X_1^{i}=\tilde{x}|X_t=x_t) -\delta_{x_t^j} (\tilde{x}^j)\bigg). \\
\end{align*}
Altogether, we have
\begin{equation*}
{R'}_t^{j}(x_t^j, \tilde{x}^j) = {R}_t^{j}(x_t^j, \tilde{x}^j|y) = \frac{p(y|X^{\sigma(i)}=\tilde{x}^j, X^{\sigma(<i)}=x_t)}{p(y|X^{\sigma(<i)}=x_t)} R_t^{j}(x_t^{j}, \tilde{x}^j).
\end{equation*}
\end{proof}

This result is exactly the one derived in \citet{nisonoff2025unlocking} and has two important properties we'd like to point out here. 

First, this conditional rate is zero for any transition involving the change of two or more positions, just like the unconditional rates. Therefore, exactly $D\times (S-1)+1$ rates need to be computed, where $D$ is the number of positions, $S$ is the size of the alphabet, $D\times (S-1)$ are all the 1-Hamming distance jumps for each position in the sequence, and the extra ``$+1$'' at the end comes from the rate of the CTMC staying at the current state.

Second, this shows that we can sample from the guided \aoamshort\ to get exact predictor guidance, and therefore only need to evaluate the classifier $S$ times: once for each of the possible transitions at the position being decoded. When $S$ is small enough, this lets us sidestep the need for the Taylor-approximate guidance, which was derived by \citet{nisonoff2025unlocking}. They developed TAG due to the intractability of evaluating a classifier $D\times (S-1)+1$ times to run exact guidance when sampling by numerical integration. Using \aoamshort\ sampling, we show exact guidance can be done with only a constant dependence in the dimension and linear dependence in the alphabet size. Therefore, running TAG with the \aoamshort\ sampling approach only adds a constant time cost for using guidance.

\subsection{Common set-up across \textit{in silico} experiments}
\label{sec:method-insilico}

We now provide an overview for the Method sections that are dedicated to each \textit{in silico} experiment. Since the main text presents a high-level overview, for each \textit{in silico} experiment, we include a dedicated section to provide the full experimental details, including datasets, predictive models, baseline configurations, and sampling settings for all methods, as well as additional results. Here, we will first address several considerations that are shared across different experiments, including baseline configurations and hyperparameter selections.

\medskip\noindent
\textbf{Post hoc filtering.} 
In all experiments, we compared to generation from the pre-trained model \textit{unguided}, with and without post hoc filtering. To ensure a meaningful comparison, we allowed the pre-trained model to generate unguided for the same wall clock time as with guidance, then post hoc filter to only the top-$k$ samples using the predictor, where $k$ is the number of samples allocated to \ourmethod. In other words, all methods get to generate sequences for the same amount of time and they are all scored according to their top-$k$ sequences, as judged by the predictor. 

\medskip\noindent
\textbf{Fine-tuning.}\label{si_sec:lora}
Although this work centers on on-the-fly guidance, and fine-tuning is not on-the-fly, we nonetheless
included a comparison to fine-tuning (\sftshort) as it is a widely used technique for adapting pre-trained generative models. Note that FT is sometimes referred to as ``supervised fine-tuning'' because one accesses labels in order to curate the set---however, since the procedure does not use the labels, we choose not to invoke this terminology.

For each task, we performed fine-tuning by continuing training the pre-trained model under its original training objective on a task-specific set of selected
sequences thought to most exhibit the desired property~\citep{ziegler2019fine,ouyang2022training,nijkamp2023progen2,stocco2025guiding}. When the pre-trained models are very large (\eg, ESM3 and ESM C), we performed fine-tuning with LoRA, a standard parameter-efficient fine-tuning method that injects low-rank trainable matrices into selected linear layers while freezing most pretrained weights~\citep{hu2022lora}. Unless otherwise stated, we used a commonly-used LoRA configuration with low-rank adapters of rank 16 (scaling factor 32; dropout 0.05) applied to the model’s attention transformations and the sequence output head, training only these added adapter parameters while keeping the pretrained weights frozen. These choices are consistent with prior large-scale LoRA studies that explicitly sweep ranks including 16 and identify 0.05 as a useful dropout value in smaller-model regimes (the $\sim$1 billion parameters of the public ESM3 model is comparatively small when compared to large language models where fine-tuning has been more extensively studied)~\citep{hu2022lora,dettmers2023qlora}. We optimized the adapters with a learning rate of $10^{-4}$ (no weight decay) using a linear warmup over the first 10\% of updates. We found the number of necessary training steps to be task-dependent and typically trained for at least 500-1000 steps (with effective batch size 64-128) to ensure the model is well-trained. Checkpoints were typically selected based on the model loss on a small hold-out set. We did not further tune these optimization hyperparameters beyond this fixed configuration; task-specific tuning and model selection were performed only through the experiment-specific data selection rule and hold-out validation protocol (detailed below). For smaller model (\eg, ProteinMPNN), we performed full fine-tuning of all trainable parameters. 

In practice, \sftshort also requires the user to make several additional design choices beyond specifying the target property, including (i) which subset of sequences should constitute the fine-tuning set, and (ii) how to perform model selection and hyperparameter tuning, which typically necessitates holding out data for validation and choosing optimization settings (\eg, learning rate schedules, weight decay, early stopping, and the number of epochs) based on some validation signal. Empirically, even in well-studied large-model settings, fine-tuning model selection is performed using held-out validation criteria and can exhibit rapid overfitting~\citep{ouyang2022training}; more broadly, fine-tuning performance can be sensitive to optimization hyperparameters and is commonly selected using validation set performance~\citep{devlin2019bert,jiang2020smart, dodge2020fine, mosbach2021on}. Accordingly, for each experiment we specified a task-appropriate data selection rule and validation protocol (described in experiment-specific sections), and used these held-out signals to tune hyperparameters for the \sftshort baselines.

\medskip\noindent
\textbf{Guidance strength.}
The guidance strength $\gamma$ in Equation~\eqref{eq:conditional_rates_with_gamma} is a hyperparameter that can affect the quality of generation. Conceptually, Bayes Rule implies $\gamma = 1$, but one might desire to bias the generation more towards the conditioning signal, and others have found $\gamma > 1$ to result in higher quality samples for continuous-space diffusion models~\citep{dhariwal2021diffusion, ho2022classifier}.  In our experiments, we tested $\gamma \in \{1, 10, 100\}$, and didn't systematically tune it further, although more extensive tuning might improve generation further. We have listed the final value of $\gamma$ for each experiment in the sections below. In practice, we recommend the user to start with $\gamma = 1$ and monitor whether $p(y | x_t)$ increases over the generation trajectory. Depending on where the unconditional model started the generation trajectory, and the distribution of $p(y | x_t)$ over the state spaces, $\gamma = 1$ might be not sufficient to increase the guidance signals. In such cases, one can increase the guidance strength to either $\gamma = 10$ or $\gamma = 100$, or to perform more systematic tuning on the metrics on generation quality one cares about.

\medskip\noindent
\textbf{Sampling speed.}
For both unguided and guided generation, the primary factor affecting the speed of generation is the number of sampling steps taken. In \aoamshort sampling, the number of steps corresponds to the number of positions in the sequence. In flow-matching sampling with Euler integration, the number of sampling steps correspond to the amount of discretization, with fewer number of steps (\ie~larger $\Delta t$) typically resulting in faster generation at the cost of lower sample quality. For MLMs, before its connection to discrete-space diffusion and flow models were made, practitioners used a variety of sampling methods. Some examples include directly decoding all the tokens from the fully masked sequence, or decoding the sequence iteratively, with a commonly used heuristic to choose the decoding order which prioritizes the token to be unmasked according to some confidence~\citep{hayes2025simulating}. To illustrate the benefit of our unification of training and sampling across multiple model classes, in each experiment, we used both the default sampling methods of the pre-trained model, and when it is slower than guidance, we leveraged the unified view to use other more efficient sampler. For example, in the stability-guided inverse folding experiment, we showed that one can leverage the unified perspective and perform flow-matching sampling with ProteinMPNN, an~\aoamshort, to produce similar quality samples with twice the speed. Since guidance can be perform with either ``continuous-time'' (\eg, Euler integration) or ``discrete-time'' sampling methods (\eg, any-order autoregressive sampling) (\fref{fig:schematic}b), we used the number of sampling steps that roughly corresponds the number of steps taken by unguided generation (which might depend on the length of the generated protein). 

In addition to the number of sampling steps, when performing guidance, the sampling speed is additionally affected by whether exact or approximate guidance is used, and the cost of performing forward passes on the predictor (and backward passes in the case of TAG). We also illustrate the effectiveness of both our exact and approximate guidance (TAG) methods under different settings. We used exact guidance for the more challenging multi-property, Pareto-extrapolative design and the \textit{in vivo} experiment, where performing forward passes with the predictive model is also relative cheap. We also used exact guidance for the enzyme guidance experiment, where we found TAG to struggle. In other experiments, we used TAG to show that despite being approximate, it can perform quite well with high efficiency on a variety of tasks, including for fold class guided generation, where the large size of the ESM3 structure token alphabet ($>$4k) makes exact guidance fairly inefficient.

\subsection{Stability-guided inverse folding}
\label{sec:method-stability}

We demonstrated how assay-labeled experimental data could be used to condition generative models for protein sequences, improving the success rates for such models in generating sequences with the desired function. To illustrate this use case, we guided ProteinMPNN with a large dataset of experimental folding stability measurements on mini-proteins~\citep{tsuboyama2023mega} to generate sequences with improved stability. Stability is reported in real-valued units of $\ddg$, which quantifies the difference in stability from the wild-type sequence. A protein is considered stable if it is at least as stable as the wild-type ($\ddg \leq 0$). To generate sequences with improved stability, we followed the experimental set-up from~\citet{nisonoff2025unlocking}. Specifically, we used ProteinMPNN~\citep{dauparas2022robust} as the inverse folding model (as opposed to the FMIF model used in~\citet{nisonoff2025unlocking}). As described above, the \aoamshort\ nature of ProteinMPNN means we can also treat the pre-trained model as a masked flow matching / diffusion model. Then we apply guidance with the stability oracle from~\citet{nisonoff2025unlocking} to obtain a stability-conditioned inverse folding model. For full details on the dataset used, the training procedure of the stability oracle, and the stability classifiers used for guidance, we refer the readers to \citet{nisonoff2025unlocking} (Appendix section F.4).  

\subsubsection{Set-up and baselines}
We applied this set-up on eight randomly chosen proteins for which the generated sequences from ProteinMPNN are not already on average more stable than the wild-type, some of which therefore differ from the ones used in~\citet{nisonoff2025unlocking}. For each generated sequence, we define success based on two criteria: (i) the generated sequence should fold into the desired structure, $\mathbf{x}$, and (ii) the generated sequence should be at least as stable as the wild-type sequence, \mbox{$\ddg \leq 0$}. We used AlphaFold 2~\citep{jumper2021highly} to predict the structure of the generated sequences and compare the root mean square deviation (RMSD) between the predicted structure and the desired structure $\mathbf{x}$. We consider a sequence to achieve the desired fold if the RMSD is less than or equal to $2$\r{A} following previous evaluation criteria~\citep{campbell2024generative}. To evaluate the folding stability of our sequences, we use the $\ddg$ regression model (\ie~the stability oracle) from \citet{nisonoff2025unlocking}.

For each protein, we compared the generations from all methods on 100 sequences. As baselines, we compared to the default sampling strategy of ProteinMPNN, which is autoregressive. However, the AO-ARM nature of ProteinMPNN means we can also treat it as a masked flow matching model, and use flow-matching sampling with Euler integration. Using Euler integration with $\Delta t = 0.01$, unguided generation was twice as fast as guided generation, so we also compared to unguided generation with the flow-matching sampler both with and without post hoc filtering.

To compare to fine-tuning, we followed a similar procedure to \citet{widatalla2024aligning} by training ProteinMPNN further on its original objective on all the sequences which are more stable than wild-type in the stability dataset, which are around 100,000 sequences out of around 669,000 sequences. Since ProteinMPNN only has modest number of trainable parameters, we performed full fine-tuning. To select hyperparameters for fine-tuning, we used the same validation split from~\citet{nisonoff2025unlocking} constructed by clustering examples by wild-type backbone (WT cluster) and holding out 20\% of backbones by cluster. We kept all other ProteinMPNN training settings consistent with the original recipe and tuned only the learning rate and weight decay of the AdamW optimizer~\citep{loshchilov2018decoupled}, with learning rate in $\{10^{-4},\, 5\times10^{-5},\, 2\times10^{-5},\, 10^{-5},\, 5\times10^{-6},\, 10^{-6},\, 5\times10^{-7}\}$, and weight decay in $\{10^{-2}, 10^{-3}, 10^{-4}\}$. For each pair of learning rate and weight decay value, we trained on the training clusters and selected the checkpoint with the lowest validation negative log-likelihood / perplexity on the held-out WT clusters. This protocol is designed to reduce overfitting to specific protein families and is possible since the stability dataset from \citet{tsuboyama2023mega} is very unique and contains the same type of measurements across many diverse proteins (in contrast to single-family landscapes where protein-level holdout is not possible). The chosen setting was learning rate of $10^{-6}$ and weight decay of $10^{-2}$ trained for around 25 epochs. After selecting hyperparameters on the validation split, we re-trained the model using the selected hyperparameters on all clusters for the chosen stopping point to produce the final FT baseline used in evaluations. 

For unguided generation and FT, sampling was performed with three different temperatures: $1.0$, $0.1$, and $0.01$ and we report the results for the temperature that achieved the highest success rate. For guided generation, we used a single temperature of 0.1 and Euler integration with a step size $\Delta t=0.01$. Guidance was performed with TAG with a single guidance strength $\gamma=100$. We did not specifically tune the sampling temperature or guidance strength. 

\subsubsection{Implementation note on decoding ordering in \aoamshorts}
While the mathematical equivalence between \aoarlong models (\aoamshorts) and discrete flow matching is exact in theory, a practical subtlety arises when implementing \aoamshorts are implemented with causal attention mechanisms. In practice, many \aoamshorts, including ProteinMPNN~\citep{dauparas2022robust}, are implemented using standard autoregressive architectures with causal attention masks. When computing the probability $p_\theta(x_1^i | x_t)$ for a masked position $i$ given a partially masked sequence $x_t$, the model must process the unmasked positions in $x_t$ in some specific order due to the causal nature of the attention mechanism.
Formally, for a position $i \in M_t$ (the set of masked positions), the probability $p_\theta(x_1^i | x_t)$ should theoretically be independent of how the unmasked positions in $U_t = {j : x_t^j \neq \text{?}}$ are ordered. However, in practice, the causal attention mechanism means that:
\begin{equation*}
p_\theta(x_1^i | x_t) = p_\theta(x_1^i | x_t^{\sigma(1)}, x_t^{\sigma(2)}, \ldots, x_t^{\sigma(|U_t|)})
\end{equation*}
where $\sigma$ is some permutation of the unmasked positions in $U_t$, and the probability depends on this specific ordering $\sigma$.
During training, the \aoarlong objective encourages the model to produce consistent predictions regardless of the permutation $\sigma$. The expectation over all possible permutations in the training objective should, in principle, ensure that $p_\theta(x_1^i | x_t)$ is invariant to the ordering of unmasked positions. However, in practice, this invariance may not be perfectly achieved due to imperfect optimization and finite-sample training.
To address this subtlety and ensure consistent behavior during guided sampling, for the protein stability experiment that guides ProteinMPNN, we fix the decoding ordering for $x_t$ based on the residue positions of the unmasked tokens. Specifically, we sort the unmasked positions in $U_t$ in ascending order of their sequence positions:
\begin{equation*}
\sigma(j) = \text{sort}(U_t)[j]
\end{equation*}
where $\text{sort}(U_t)$ returns the unmasked positions ordered by their indices in the original sequence. This approach provides a deterministic and reproducible ordering for any given masked state $x_t$, respects the natural left-to-right order of residues in the protein sequence, and requires no additional hyperparameters or complex ordering strategies. By using this fixed position-based ordering, we ensure that the conditional probabilities $p_\theta(x_1^i | x_t)$ are well-defined and consistent across different sampling runs, enabling reliable application of the guidance method to \aoamshorts trained with causal attention.

\subsection{Guiding ESM3 with assay-informed predictive models}
\label{sec:method-esm3-fitness}

For the extrapolative design task, we evaluated the ability of \ourmethod to generate diverse proteins with property values that exceeded any seen in the supervised training data. We used assay-labeled datasets for three diverse proteins which include many variants with more than single mutations: the CreiLOV fluorescent protein~\citep{chen2023deep}, the ubiquitination factor E4B (UBE4B)~\citep{starita2013activity}, and the poly(A)- binding protein (PABP)~\citep{melamed2013deep}. Below we provide details for the dataset used, the oracle model used for evaluation, the set-up for baselines and guidance, and additional results.

\subsubsection{Datasets}
We used assay-labeled datasets for three diverse proteins which include many variants with more than single mutations: the CreiLOV fluorescent protein \citep{chen2023deep}, the ubiquitination factor E4B (UBE4B) \citep{starita2013activity}, and the poly(A)- binding protein (PABP) \citep{melamed2013deep}~(\suppfref{si_fig:fitness_dataset}a) that have previously used in other works for evaluating protein fitness prediction and optimization~\citep{hsu2022learningfitness, emami2023plug, blalock2025functional, yang2025steeringgenerativemodelsexperimental}. 

\medskip\noindent
\textbf{CreiLOV.}
CreiLOV is a flavin-binding fluorescent protein (FbFP) derived from LOV photoreceptors; LOV-FPs were developed as genetically encoded reporters that work without molecular oxygen, making them attractive alternatives to GFP in hypoxic/anoxic (anaerobic) contexts. We used the dataset from the study of~\citet{chen2023deep} processed by \citet{yang2025steeringgenerativemodelsexperimental}, where higher fitness refers to brighter florescence. The dataset contains 167,530 variants, including (i) site-saturated mutagenesis covering all 118 residues (i.e., all single substitutions), and (ii) a higher-order mutations at 15 selected positions with beneficial single mutations, with up to 15 mutations from wild-type (\suppfref{si_fig:fitness_dataset}a, left).

\medskip\noindent
\textbf{UBE4B.}
UBE4B is an E3 ubiquitin ligase in the ubiquitin–proteasome system, where modulating ligase activity can change protein ubiquitination and downstream regulation; it’s a useful engineering target because small sequence changes can strongly tune enzymatic activity. We used the assay-labeled dataset from~\citet{starita2013activity}, where fitness refers to E3 auto-ubiquitination activity (quantified by high-throughput selection/enrichment of variants with higher activity). Specifically, we used the version of the dataset processed by~\citet{hsu2022learningfitness}, which includes 32,290 variants with mutations in the 102 residues in the C-terminal region of the proteins, with up to 6 mutations from wild-type (\suppfref{si_fig:fitness_dataset}b, right).  

\medskip\noindent
\textbf{PABP.}
PABP is a poly(A)-binding protein that regulates mRNA fate, and its RNA-recognition motifs (RRMs) are a classic system for studying/engineering RNA-binding function. We used the assay-labeled dataset from~\citet{melamed2013deep}, where fitness refers to in vivo functional complementation / growth-based enrichment of RRM2 variants (variants that better support cellular function are enriched) for the yeast version of PABP, usually referred to as Pab1. Specifically, we used the version of the dataset processed by~\citet{hsu2022learningfitness}, which includes 37,711 variants with single and double mutations in a 96 residues region of the protein (\suppfref{si_fig:fitness_dataset}b, middle).

\subsubsection{Oracle models}
Following similar experimental set-up to prior work~\citep{blalock2025functional, yang2025steeringgenerativemodelsexperimental, emami2023plug}, we used an ``oracle'' model trained on the full assay-labeled dataset as an \textit{in silico} proxy to measure the fitness of the generated sequences. For CreiLOV, we used the published fitness oracle from~\citet{yang2025steeringgenerativemodelsexperimental}, which is an ensemble of 10 two-layer MLP regressors trained on one-hot sequence encodings. \citet{yang2025steeringgenerativemodelsexperimental}~used all of the single, double, and triple mutants in the library
for training, with 10\% of the quadruple mutants for validation. The oracle achieves Spearman correlation of 0.93 on its validation set (\suppfref{si_fig:fitness_dataset}b, left).

For UBE4B and PABP, we trained oracle models with similar architectures to above. Each model is an ensemble of 5 two-layer MLPs regressors (hidden dimension size 400) trained on one-hot encoded sequences. For each dataset, we trained the model on 90\% of the sequences and validated on the remaining 10\% sequences, with learning rate of $10^{-3}$ (weight decay of $10^{-2}$) using linear warmup of 10\% of updates. Since the UBE4B and PABP datasets have fewer number of variants and lower-order mutants compared to the CreiLOV dataset, we did not specifically holdout higher-order mutants for testing so as to maximize the information the oracle could extract from the assay-labeled data. The oracle for PABP achieves Spearman correlation 0.92 on its validation set, and oracle for UBE4B achieves Spearman correlation of 0.59 on its validation set (\suppfref{si_fig:fitness_dataset}b, middle and right).

\subsubsection{Set-up and baselines}
We used ESM3 as the pre-trained model. We compared guided generation to unguided generation, and fine-tuning ESM3 with LoRA~\citep{hu2022lora}. For sampling from ESM3, either unguided or after fine-tuning, we used the default temperature of 0.7 and $L$ sampling steps where $L$ is the number of residues we are designing. Since the predictor used for guidance is lightweight, guided generation runs at roughly the same speed as unguided generation. Nevertheless, to be conservative, we generated twice as many samples for both unguided generation and \sftshort when comparing to post hoc filtering. We performed generation starting from fully masked sequences so as not to bias the generation towards the wild-type sequence, while note that conditioning on additional information, such as partially masked wild-type sequence or predicted backbone structure, can make the generation task substantially easier. 

For fine-tuning ESM3 with LoRA, we used the hyperparameter settings detailed in Section~\ref{si_sec:lora} ``Common set-up across \textit{in silico} experiments''. For validation, we holdout $\sim100$ sequences with the highest property value in the labeled dataset (roughly correspond to the top $0.1\%$ of CreiLOV dataset, and the top $0.5\%$ of the PABP and UBE4B datasets). We then experimented with performing fine-tuning on the sequences in the top $q \in \{50\%, 20\%, 10\%, 5\%, 2\%\}$ percentile of the labeled dataset. For each hyperparameter setting, we trained the model with a minimum of 1,000 steps with an effective batch size of 128. We selected the threshold and checkpoint with the lowest perplexity on the holdout sequences to report generation results. We found the perplexity is typically lowest for the models trained on the top $2-10\%$, for which the model typically converged after 1,000 training steps and the differences in holdout perplexities with the different thresholds are not significant. 

\subsubsection{Guidance}
To perform guidance, we trained predictors directly on partially masked sequences to align with the guided sampling distribution. For each dataset, we used the same architecture to the oracle for the predictor used for guidance, and trained them on 2,000 randomly sampled sequences from the labeled dataset to match the experimental budget in our later \textit{in vivo} experiment, with the same set of variants holdout as fine-tuning. Positions in each sequence are masked at random with a masking rate uniformly sampled from $(0, 1]$. We used the same architecture and hyperparameter as the oracle model for these predictive models. The predictors are very lightweight and took less than 5 minutes to train for each dataset.

\medskip\noindent
\textbf{Joint likelihood guidance.}
Because ESM3 is a pan-protein model, we additionally guided sampling to remain within the desired protein family. We therefore guided with a joint likelihood that combines (i) a property term encouraging high predicted property value and (ii) an in-family term encouraging membership in the target family:
\begin{equation}
p(y \ge y^*, \mathrm{fam}=1 \mid x_t)
=
p(y \ge y^* \mid \mathrm{fam}=1, x_t)\,p(\mathrm{fam}=1 \mid x_t).
\end{equation}
We parameterized $p(y \ge y^* \mid \mathrm{fam}=1, x_t)$ using an ensemble regressor: letting $\mu(x_t)$ and $\sigma(x_t)$ denote the mean and standard deviation of the ensemble predictions, we defined
\begin{equation}
p(y \ge y^* \mid \mathrm{fam}=1, x_t)
=
1-\Phi\!\left(\frac{y^* - \mu(x_t)}{\sigma(x_t)}\right),
\end{equation}
where $\Phi(\cdot)$ is the standard Gaussian CDF.
We parameterized $p(\mathrm{fam}=1 \mid x_t)$ with a Bernoulli likelihood whose probability is given by a binary classifier head. The classifier was trained to discriminate between (positive) sequences from the assay-labeled dataset and (negative) sequences sampled from unguided ESM3. For each protein, we trained the model on 500 sequences from ESM3's unguided generation in addition to the 2,000 randomly sampled sequences from labeled dataset. Regression and classification losses were weighted equally.

\medskip\noindent
\textbf{Choosing the property threshold $y^*$.}
Although one could set $y^*=y_{\max}$ (the best observed labeled variant) to directly target extrapolation, we found this yields a weak early guidance signal because predictors tend to output values near the training-set mean when the sequence is heavily masked. To keep the guidance signal non-degenerate throughout the trajectory, we set $y^*$ to be at or below the mean label value of the predictor training set (CreiLOV: $3.8$; PABP: $-1$; UBE4B: $-1.5$). A natural extension is to anneal $y^*$ during sampling---starting near the training-set mean early in decoding and increasing toward a more stringent threshold later---analogous in spirit to the progressively tightened conditioning distributions used in CbAS~\citep{brookes2019conditioning}. While such annealing could plausibly improve extrapolation further, we found the fixed-threshold heuristic to be sufficient and did not explore annealing here.

\medskip\noindent
\textbf{Guided sampling.}
For guided generation, we used TAG with $\gamma = 10$ and $\Delta t = 0.01$. We used the same temperature of 0.7 as with unguided generation. The joint likelihood provides two guidance terms that can compete early in generation, particularly before the partially unmasked sequence has entered the in-family region of ESM3's distribution. To stabilize generation, we optionally applied staged guidance: for PABP and UBE4B we first guided using only the in-family term $p(\mathrm{fam}=1 \mid x_t)$ and then switched to the full joint likelihood. Specifically, with $t=0$ denoting a fully masked sequence, we began joint guidance at $t=0.1$ for PABP and $t=0.5$ for UBE4B, and used joint guidance throughout the entire trajectory for CreiLOV. An alternative would be to use separate guidance strengths for the two terms; we did not pursue this to keep hyperparameters consistent across experiments.

\subsubsection{Additional results}
In addition to the main-text results (\fref{fig:fitness_single}), we report structure-based metrics computed from ESMFold predictions for generated sequences, including the distribution of pLDDT and the refolded RMSD to the predicted wild-type structure (\suppfref{si_fig:fitness_structure}). Unguided ESM3 generations frequently exhibit low pLDDT and large refolded RMSD; consequently, a fraction of guided samples also fall into this low-confidence range. Importantly, however, \ourmethod also produces sequences that simultaneously achieve high predicted property values (\fref{fig:fitness_single}d), remain diverse and substantially different from wild-type (\fref{fig:fitness_single}c), and exhibit favorable structure-based metrics, suggesting that the best guided generations are not merely exploiting the fitness oracle. In contrast, fine-tuned (FT) generations tend to remain very close to the wild-type sequence (\fref{fig:fitness_single}c), and accordingly often yield higher pLDDT and lower refolded RMSD.

\subsection{Multi-property guidance with \ourmethod}
\label{sec:method-multiproperty}

\subsubsection{Dataset}

In this experiment we use the data from the first library tested by~\citet{wang2025active}. This library comprised of the wild-type sequence (WT) and 1,098 variants arising from an alanine scan of the protein, a site saturation mutagenesis (SSM) at 49 positions, and a modest number of higher order mutations (up to degree 5) sampled from 7 positions. Because the assay values saturate to finite values for non-functional variants, we use an inverse soft-plus transformation on the data and transform predictive model outputs back with a softplus so that they saturate at the same non-functional level as the real data.

\subsubsection{Oracle and predictive models}
Both the oracle and predictive models in this experiment are regression models that use the kernel trick. We use a kernel engineered to allow meaningful use of mutational data even when the mutation being test on is not necessarily the same as the mutation observed in the data. Formally, for one-hot encoded variants $X_1, X_2$ and the one-hot encoded wild-type $X_\text{WT}$, the kernel for capturing first-order epistasis (point mutations in isolation) is $K_1(X_1, X_2)=(X_1-X_\text{WT})^T(X_2-X_\text{WT})$. Note that this regressor only weighs the effects of different mutations through the variants, not at the level of individual epistatic terms. Due to this structure, higher order kernels for $n$-th order simply need to count the $n$-th order terms that are shared between two variants. This is purely a function of the first-order matches, \ie{ }$K_n(X_1, X_2)=$ ${K_1(X_1, X_2)}\choose{1}$. The predictive model is second order, $K_1+K_2$, and the oracle is third order, $K_1+K_2+K_3$.

For a partially masked sequence $x_t$, predictive model of Pb and Zn activity $f(x)=(\hat{y}_\text{Pb},\hat{y}_\text{Zn})$, and a target region $T$, the noisy classifier was defined as $p(y|x_t)=\mathbb{E}_{p_\theta^*(x|x_t)}[\mathbf{1}[f(x)\in T]]$. Here $p_\theta^*(x|x_t)$ is the product-of-marginals distribution for the pretrained generative model $p_\theta(x|x_t)$. We also estimate the expectation with 50 samples from $p_\theta^*$.

Finally, due to mismatches in the predictor and oracle, samples that look good under the predictor might fail when scored by the oracle. In a real design setting, there is no oracle access, so we aimed to somehow estimate this mismatch. We reasoned that, especially with our regression models that used the kernel trick, mutations for outlier datapoints would cause extremely large fluctuations in the predicted values when seen in new mutational context. This is because the optimization problem when fitting the regression parameters is ill-conditioned. This can result in the regression coefficients for a few data points relevant for the prediction of the outlier being perfectly balanced against each other, with large values canceling each other out to match the true label. When a new variant is scored, a different subset of sequences may have high inner-product scores with it. Among this new subset, the unstable, large regression coefficients for an outlier sequence may suddenly dominate. 

To quantify this bias for variants far form the data, we generated new variants with the non-SSM positions fixed and up to 30 mutations selected uniformly at random. We then scored them with the trained predictor and additionally scored them with an ensemble of 10 predictors each fit on the training data with 20\% of the examples randomly dropped out. We observed that Zn binding affinity was typically underestimated by our model and that Pb predictions tended to not have a large systematic bias (\suppfref{si_fig:multi_main}e). We adjusted the definitions of the target regions for all experiments according to these heuristically observed residuals.

This method helps for our kernel predictors, but does not address the difference in epistatic degree modeled by the predictive model and the oracle. Accordingly, we do observe several outliers that were classifier as being in the target region by the 

The oracle model was trained on the full dataset whereas the predictor was only trained on a subset of 1060 sequences that excluded some high performing and higher mutational depth sequences. The excluded sequences were used for the validation set (\suppfref{si_fig:multi_data}a, left and right panels).

\subsubsection{Experimental set-up and baselines}

For all models we only designed the positions for which there was SSM data available.
All results were compute matched. The unguided model generated 1000 sequences, the finetuned model generated 500, and \ourmethod got 100. All models were evaluated on their top 100 sequences. We used the exact guidance algorithm for \ourmethod (Section~\ref{si_sec:exact_guidance_derivation}).

The finetuned model trained on the subset of the predictor training data that was in the more Pareto-optimal arm of the experimental data distribution (\suppfref{si_fig:multi_data}a, center panel, points above the dashed line). It was trained until the loss converged. We did not observed any overfitting on training data. The validation loss converged at a similar time as well. 

The original PbrR study fine-tuned a protein langauge model to iteratively redesign the best sequences in the dataset for the Pareto extrapolation task~\citep{wang2025active}. They call this method Multi-objective Conservative Extrapolation, which we refer to herein as MOCE. In addition to the baselines mentioned above, we also compare to this method and find that it does not generate any successes (\suppfref{si_fig:multi_main}d).
This is likely because their method is designed to conservatively limit extrapolation and does not move far from the training data.

For MOCE, we used the training and generation scripts in the project Github repository. We trained the model on the same data as the predictor, since MOCE uses pairs of data to learn the directional impact of mutations on sequences. We use the 30 sequences closest to the target region as the starting points for sampling. We also allow MOCE to generate 600 sequences of which there are 299 unique variants and 100 of these were used for evaluation.

\subsection{Enzyme class guided sequence generation}
\label{sec:method-enzyme}

In this experiment, we demonstrated how guidance can generate sequences from target Enzyme Commission numbers (ECs). ECs are hierarchical four-level codes describing different enzyme classes. We trained an enzyme class predictor that classifies protein sequences into their ECs based on their sequence and used it to guide the protein language model ESM3 \cite{hayes2025simulating}. We present results for guiding sequence generation towards EC numbers of varying prevalence in SwissProt. Our results focus on the top-10 most common ECs (\eg, Figure~\ref{si_fig:ec}c) but we include data for 10 random and rare EC classes (\eg, \tref{tab:enzyme_guidance_results_ec_sets_combined}) as well.


\subsubsection{Enzyme class predictor}

Our enzyme class predictor was a 3-layer multilayer perceptron (MLP) that took mean-pooled ESM3 sequence embeddings as input. The embeddings are of dimension 1536 and we used an MLP with hidden dimension ($2 \times 1536 = 3072$). The predictor was trained on the SwissProt dataset used by ProteInfer \cite{sanderson2023proteinfer}. Our model is trained on all SwissProt sequences, including non-enzymes. Sequences for each EC were randomly divided into training and test sets, with the test set additionally including the test sequences for CLEAN~\cite{yu2023enzyme} which we used for benchmarking our predictor. The test set contained sequences from ECs with at least 10 samples. Those with at least 30 were used for training. Finally, 15\% of the non-enzymes were included in the test set. The rest were used as training data. The training data contained 315,833 samples and the testing data contained 36,354 samples. The noisy predictor's dataset contained 10 randomly sampled noising times for each sequence.

The model was trained with a standard negative log-likelihood objective and an AdamW optimizer (learning rate of 0.001, weight decay coefficient of 0.01). We used a batch size of 15k and trained for two epochs on the full training data. The accuracy curves for each mask-level bin are plotted over the test set of the 100 most common level-4 EC numbers (\fref{si_fig:ec}b).

\subsubsection{\ourmethod set-up and baselines}

In this experiment, we implemented exact classifier guidance for discrete flow matching as described in \citet{nisonoff2025unlocking} using the exact sampling algorithm described in Section~\ref{si_sec:exact_guidance_derivation}. As mentioned in that section, this is similar to the Gillespie algorithm described in \citet{peng2025path}. The key difference is that their Gillespie algorithm is for unconditional discrete diffusion, whereas our proof gives a general method for converting any CTMC with state- and position-wise-independent transition probabilities. We used a guidance strength $\gamma=10$ on the classifier outputs and ESM3 sampling temperature of 0.7. We sampled all positions one at a time.

This experiment includes three baselines: post hoc filtering, a LoRA fine-tuned model, and ESM3 using natural language prompting. For post hoc filtering, we take 210 unconditional samples for each Enzyme Class so that the comparison is compute-matched between the two models. Exact guidance requires 21 unconditional model evaluations and the cost of the classifier ends up being negligible by comparison. We finetune ESM3 using LoRA according to setup described in Section~\ref{si_sec:lora}. A separate LoRA model is trained for each EC. Due to the additional training cost of fine-tuning of this experiment, we only allowed it the same sampling budget as guidance. We use an 80/20 split between training and testing. We train each LoRA model to convergence and take the checkpoint with lowest test loss (average cross entropy for masked positions). For natural language prompting, we curated a set of the keywords from the ESM3 IntroPro keyword vocabulary that were most similar to the official EC name using fuzzy string matching (see \tref{tab:enzyme_keywords} for full list). We then sample from ESM3 with these keyword prompts as additional conditioning information during sequence generation.

\subsubsection{Evaluation metrics}

We considered the generated sequences a success if i) the predicted class with the highest probability was the target EC and ii) if ESMFold predicted a median pLDDT across residues of at least 0.8. The ``oracle'' model for this experiment was trained the same way as the predictor, just on all of SwissProt.


\subsubsection{Detailed results}
In our main experiment, we had each method (unguided with and without post hoc filtering, finetuning, natural language prompting, and \ourmethod) generate 10 sequences for each of the 10 most common enzyme classes (\fref{si_fig:ec}). We did not observe a significant difference in the success rate between the finetuned model and \ourmethod, as was expected for an interpolation task. We did not observe any successes from the unguided model or the natural language prompting.

We additionally ran an experiment where we compared unguided sampling and \ourmethod's abilities to redesign partially masked enzyme sequences to recover another sequence of the same EC (\fref{si_fig:ec}d). We used several masking levels (50\%, 60\%, 70\%, 80\%, 90\%, and 100\%). We find that \ourmethod displays a consistent improvement over the unguided model. At high masking levels this is a substantial relative improvement over the unguided model, but a small relative improvement at low masking levels. It is worth noting that we find no detectable degradation in the pLDDTs of designed sequences due to the guidance process.

Finally, we report results for 10 randomly selected and the 10 most rare ECs (\tref{tab:enzyme_guidance_results_ec_sets_combined}). For these experiments, we start generation from real enzyme sequences with $80$\% of positions masked. We find that the \ourmethod performs similarly on the top 10 classes and 10 random classes. However, the 10 least represented classes are not in the training data for the predictor. Accordingly accuracy suffers. \ourmethod is still able to achieve success about one-third of the time. We believe this is because of its cross-entropy training objective, allowing it to guide to the target classes essentially by ``process of elimination'', in that guidance is really just avoiding the classes for which the predictor does have trained logits.

\subsection{Fold class guided backbone structure generation}
\label{sec:method-fold-class}

We first provide an overview of the experimental set-up for guiding ESM3 backbone structure generation with fold class labels from CATH~\citep{dawson2017cath}, a hierarchical protein structure classification scheme. Concretely, we trained a fold classifier, $p_\phi(\mathbf{c} \, | \, \mathbf{x})$ that predict the CATH labels, $\mathbf{c}$, from the discrete structure tokens outputted by the structure encoder of ESM3, $\mathbf{x}$, on the set of CATH-labeled protein structures from the Protein Data Bank (PDB). We combined this fold class predictor with the structure generation track of ESM3, $p_\theta(\mathbf{x})$, to obtain a fold class conditioned backbone structure generative model, $p_{\theta, \phi}(\mathbf{x} \, | \, \mathbf{c})$ (\suppfref{si_fig:fold_class}a). We designed this experiment to demonstrate that \ourmethod can be used to guide a large multimodal protein language model to generate structures. 

\subsubsection{Fold classifier}
Now we describe additional details of the fold class oracle and the fold classifier used in guidance. Since ESM3 operates on discrete structure tokens, we need a predictor which predicts the fold class label of a protein from its structure token representation. To this end, we first trained a fold class predictor that takes in discrete structure tokens as inputs. The model architecture consists of an encoder-only Transformer trunk. The final per-position embedding is mean-pooled, then fed into three parallel linear layers, each projecting the embedding to a vector with the dimension equal to the number of classes for each of the three levels of CATH hierarchies: C (``Class''), A (``Architecture'') and T (``Topology/Fold''). The C level describes the overall secondary structure-content of a structure (\eg, ``mainly $\alpha$'', ``mainly $\beta$'', and ``mixed $\alpha/\beta$), the A level roughly clusters structures by the overall arrangement of secondary structure in space (\eg ``beta barrel'', ``roll''), and the T level assigns structures by how the secondary structure elements are connected and arranged (\eg ``immunoglobulin-like'', ``TIM barrel'')~\citep{dawson2017cath}. The model was trained on the PDB dataset following the same pre-processing procedure as in \citet{geffner2025proteina} with a total of 214,564 structures, categorized into 5 C classes, 43 A classes, and 1,336 T classes. We fed each structure through the structure encoder of ESM3 to obtain its structure token representation. The dataset is randomly divided into training, validation, and test sets at a ratio of 8:1:1, ensuring that at least one protein from each class is included in the test set whenever possible. 

The model was trained with a standard cross-entropy loss applied to each of the CATH level. Some proteins have multiple CATH labels due to having multiple domains and CATH being a domain-level annotation. For such proteins, we followed~\citet{geffner2025proteina} and randomly sampled one domain label as the ground truth for each training iteration, encouraging the model to predict equal probabilities for all domain labels of a protein. During evaluation, if the model predicts any of the correct labels, the prediction is considered correct. For each protein, the loss is calculated for each level of its available CATH label. The transformer trunk of the model has a hidden dimension size of 512, 16 attention heads and 5 encoder layers. In total, the model has around 19 million parameters. The model is trained using the AdamW optimizer~\citep{kingma2015adam,loshchilov2018decoupled} with a learning rate of 0.0001, weight decay coefficient of 0.01, a dropout rate of 0.2, a batch size of 128 and a gradient accumulation step of 4. The model is trained for around 26,000 training steps on a single NVIDIA RTX 6000 GPU. We did not systematically tune the hyperparameters of the model. On its test set, the model achieves an aggregated Micro Accuracy of 0.9678 (A level) and 0.9490 (T level) and Macro Accuracy of 0.9702 (A level) and 0.9212 (T level)~\citep{grandini2020metrics}. 

The fold classifier used for guidance takes in partially masked inputs. We initialized the model weights of this ``noisy'' classifier with those of the unnoised classifier, and trained it on the same training dataset as the unnoised classifier. We trained the noisy classifier for around 28,000 additional training steps, with the same hyperparameters as the unnoised classifier.

\subsubsection{Set-up and baselines}
We used ESM3 as the pre-trained model (with its structure token track) and performed fold class conditioning on the A (``Architecture'') and T (``Topology'') levels of CATH, which provide a practically informative middle ground between the coarser C level (Class), which is often too broad to be challenging or discriminative, and the finer-grained H level (Homology), which is highly specific and can be overly restrictive for controlled generation. The A level roughly clusters structures by the overall arrangement of secondary structure in space (\eg ``beta barrel'', ``roll''), and the T level is a more fine-grained assignment based on how the secondary structure elements are connected and arranged (\eg ``immunoglobulin-like'', ``TIM barrel'')~\citep{dawson2017cath}. We compared the methods on 10 most common classes and 10 randomly selected classes for both A and T level (\supptref{tab:cath_labels_a_t_common_random}), evaluating on $k=100$ samples for each class. The generated structures are evaluated on two criteria: (i) whether the fold class oracle predicts the desired class as the most likely class for the generated structure, and (ii) whether the structures are considered designable, where designability is calculated following standard procedure~\citep{pmlr-v202-yim23a}: for each generated backbone, eight sequences are sampled using ProteinMPNN~\citep{dauparas2022robust} with a temperature of 0.1, then ESMFold~\citep{lin2023evolutionary} is used to predicted the structure of each generated sequence. The RMSD is calculated between the generated structure and the predicted structure, and a generated backbone is considered designable if the lowest RMSD---referred to as the self-consistency RMSD (scRMSD)---among the eight sampled sequences is $\le2$\r{A}.

We compared \ourmethod to unguided ESM3 sampling (with and without post hoc filtering). We additionally compared \ourmethod to ESM3’s built-in prompt conditioning, which was trained to accept keyword-style ``function tokens'' as conditioning signals. Following the procedure in~\citet{hayes2025simulating} and the ESM3 GitHub repository, we converted each target CATH label to a set of ESM3-compatible keywords via an intermediate mapping through InterPro. Specifically, for Architecture (A; two-part codes) and Topology (T; three-part codes), we first mapped the CATH code to one or more associated InterPro entries, and then converted each InterPro identifier to a (possibly multi-keyword) list using ESM3’s provided InterPro-to-keyword dictionary (e.g., tokens such as \texttt{beta}, \texttt{barrel}, \texttt{helix}). Each keyword was provided to ESM3 as a function annotation spanning the full sequence (positions 1 through the target length); when multiple InterPro entries mapped to multiple keywords, we supplied the union of the resulting annotations.

As for the comparison to fine-tuning, we performed LoRA-based fine-tuning of ESM3 on proteins from the desired classes. Because the targets here are discrete labels, we performed fine-tuning separately for each class, substantially increasing the overall time and compute cost for fine-tuning compared to guidance. We used the same LoRA setting as previously described, with the differences that we fine-tuned the structure output head instead of the sequence output head. For each class, we performed fine-tuning on 90\% of the examples, with 10\% holdout for validation, training for a minimum of 500 steps and a maximum of 5000 steps, with an effective batch size of 64 and early stopping imposed if validation loss stops decreasing for more than 5 epochs. We found setting a sufficient large minimum number of training steps to be important for fine-tuning, especially for the classes with smaller number of examples. Depending on the number of samples and how quickly it converges, fine-tuning on a single class can take between several hours on the small classes to more than half a day on the larger classes on a single GPU. Due to the additional training cost of fine-tuning of this experiment, we only allowed it the same sampling budget as guidance.

Following evaluation procedures in prior work~\citep{geffner2025proteina}, for each class, for every sample, we sampled its length uniformly among the lengths present for the proteins of that class in the AFDB, with a minimum length of 50 and maximum length of 275. For unguided generation, keyword prompting and fine-tuning, we used the default sampling setting of ESM3 with a sampling temperature of 0.7 and $L$ sampling steps, where $L$ is the length of the protein sequence~\citep{geffner2025proteina}. We also experimented with using a lower sampling temperature and found that it results in ESM3 producing predominantly mainly~$\alpha$ proteins. For guided generation, we used a sampling temperature of 0.3 and guidance strength $\gamma=100$. We used Euler integration with TAG with a step size $\Delta t=0.005$. We found using a sampling stochasticity $\eta=20$ to improve guided generation, but did not improve unguided generation.  Unguided sampling was around 1.2 times faster than guided sampling.

\subsubsection{Detailed results}
Since the goal of this task is to generate more samples similar to those observed in the training distribution, and fine-tuning separately for each class allows each fine-tuned model to fully concentrate its distribution to the target class, we expect fine-tuning to be a very strong baseline. Indeed, we observed that fine-tuning to perform quite well across different levels and classes (\suppfref{si_fig:fold_class}b,c,d). The \ourmethod generation also substantially enrich for the desired annotations compared to unguided generation and keyword conditioning, with success rates comparable to the strong baseline of class-specific fine-tuning (\suppfref{si_fig:fold_class}b,c,d). We observed keyword conditioning to be helpful for some classes but not the majority of cases. Also as expected, generation for more rare classes is more challenging (\suppfref{si_fig:fold_class}c,d). We also observed that across all generation methods, among the samples that were classified correctly, only a modest fraction is designable. This aligns with previous observations that ESM3 tends to produce backbones that are less designable, possibly due to the composition of its training set~\citep{geffner2025proteina}. We expect the quality of generated structures should improve concomitantly with improved backbone generative models, independent of guidance. Examples of the generated backbone structures by guiding ESM3 with some example of CATH classes are shown (\suppfref{si_fig:fold_class}e).


\subsection{\textit{In vivo} experiment: guiding base editor for improved editing activity}
\label{sec:method-tadA}

Base editors are engineered CRISPR effectors that can make transition mutations in the DNA by directly deaminating cytosine (causing C-T mutations) or adenine (A-G mutations) in the single-stranded DNA exposed during target recognition by catalytically inactivated Cas9 (dCas9). Adenine base editors (ABEs) can fix about half of the pathogenic point mutations leading to genetic diseases \citep{rees2018base}. They have been evolved from an \emph{E.coli} tRNA-deaminase TadA through multiple rounds of directed evolution \citep{gaudelli2017programmable, gaudelli2020directed, richter2020phage}. ABEs are compact single domain proteins (167 amino acids) and it has been shown that they can be evolved through at least two different evolutionary trajectories, leading to high DNA deamination activity \citep{gaudelli2017programmable, richter2020phage, xiao2024adenine}. Furthermore, both evolutionary paths have led to a high concentrations of mutations in the C-terminal region, suggesting that this region can be a promising target small enough to be chemically synthesized as an oligo pool, but functionally important for evolving DNA base editing activity.

We focused on ABEs because of their importance as an engineering target,  their small size and the availability of existing methods for measuring activity based on reversion of an inactivated antibiotic resistance \citep{gaudelli2017programmable, xiao2024adenine} that are suitable for high-throughput screening and detection. Additionally, ABEs (EC 3.5.4.33) are an example of deaminases, a large class of enzymes with important biomedical and biotechnological applications, including other base editors, RNA deaminases and purine and pyrimidine deaminases \citep{BUDZKO2023102062, DELARCO2024108473}. 

We explored how experimental data for base editor activity could be used to condition pre-trained generative models to produce sequences with improved activity. As an overview, we first designed a first-round library with pre-trained generative models. We collected experimental assay-labeled data on each of the designed sequences measuring their activities. Then, we trained a predictor for activity based on the first-round sequences, which we then used to guide the design of a second-round library with~\ourmethod. Our goal was to generate sequences in the second round with improved activities compared to those in the first round. 

\subsubsection{ABE Library cloning}
ABE libraries contained a diversified region of 86 amino acids (position 82-167) that was synthetically generated as an oligo pool and purchased from Twist Bioscience. The pool was amplified with KAPA polymerase (KAPA HiFi HotStart polymerase, Roche) using the following conditions: 1x KAPA HiFi 2x ReadyMix, 300 nM forward and reverse primer, 5 ng template DNA; 3 min initial denaturation at 95 \textdegree C,  20 cycles of 98 \textdegree C for 20 sec, 64 \textdegree C for 15 sec, 72 \textdegree C for 1 min and a final extension of 2 min at 72 \textdegree C. The resulting product was loaded on 2\% aragose gel and gel-extracted using DNA Clean-Up \& Concentration kit (Zymo Research). This pool was then cloned into a vector carrying the rest of the TadA sequence fused to dSpRYCas9 with an XTEN linker following the architecture of ABE0.1 (Supplementary sequence 1) ~\citep{gaudelli2017programmable}. The library was cloned via Golden Gate assembly using BsaI restriction enzyme and T7 ligase (New England Biolabs). A total of 0.45 pmol library was mixed with 0.15 pmol vector and cycled for 5 min at 37\textdegree C and 5 min at 16\textdegree C for 30 cycles, followed by 30 min at 37\textdegree C and 5 min at 65\textdegree C. The assembled library was purified using DNA Clean-Up \& Concentration kit (Zymo Research) and transformed in One Shot Top10 competent cells (Thermo Fisher Scientific) by electroporation. 950 $\mu$L SOC media were added and cells were incubated at 37\textdegree C in a shaking incubator at 200 rpm for 1 h. Samples were taken for titer plates to estimate transformation efficiency and the rest of the recovery was inoculated in 2XYT media supplemented with 36 $\mu$g/mL chloramphenicol and grown for 8h. Plasmid DNA was purified with QIAprep Spin Miniprep Kit (Qiagen) and used for selection experiments.
\subsubsection{Base editing selection assays}
To measure adenine base editing we generated a reporter carrying a carbenicillin resistance gene with a stop codon mutation (Supplementary sequence 2) and a corresponding guide RNA and kanamycin resistance gene for plasmid maintainance. The selection strain was produced by transforming Top10 \emph{E.coli} cells with the reporter plasmid and preparing electrocompetent cells. Briefly, the plasmid was transformed in One Shot Top10 (Thermo Fisher Scientific) cells and a single colony was inoculated in liquid culture with 2XYT media and 50 $\mu$g/mL kanamycin. Cells were grown to mid-log phase and spun down at 3500 g, washed twice with ice-cold MiliQ water and twice with 10\% glycerol (v/v). Each spin was performed for 10 min at 3500 g and after a final spin of 13 min the cell pellet was resuspended with 2 mL of 10\% glycerol for 1 L of starting culture, aliquoted and flash-frozen in liquid nitrogen and stored at -80\textdegree C. For the selection experiments 100 $\mu$L of cells were transformed with 200 ng of library plasmid, electroporated and 900 $\mu$L of SOC media was added. After 1 h recovery the cells were inoculated in 4 mL 2XYT media supplemented with 50 $\mu$g/mL kanamycin (non-selective antibiotic for the reporter plasmid), 36 $\mu$g/mL chloramphenicol (for the base editor plasmid) and 0.1 \% arabinose for induction of the base editor and cells were induced for 20 h (for library 1 experiments) and 16 hours (for library 2). After that cells were washed twice with 2XYT to remove inducer and split in two. Half were inoculated in 30 mL 2XYT with 50 $\mu$g/mL kanamycin and 36 $\mu$g/mL chloramphenicol (non-selective) and the other half inoculated 30 mL 2XYT with 400 $\mu$g/mL carbenicillin and 36 $\mu$g/mL chloramphenicol (selective). Cells were grown to mid to late log phase and spun down, plasmid DNA was isolated using QIAprep Spin Miniprep Kit (Qiagen). 100 ng of the plasmid DNA was used as template for a PCR1 reaction amplifying the variable base editor region and adding adapters for Illumina sequencing (Table~\ref{tab:illumina_primers}). The base editor portion was amplified using Q5 polymerase as follows: 100 ng plasmid DNA was used as input, 200 nM forward and reverse primer, 1x Q5 MasterMix. The reaction was carried at 98 \textdegree C for 30 sec and 20 cycles of 98 \textdegree C for 10 sec, 63 \textdegree C for 10 sec and 72 \textdegree C for 1 min, with a final extension of 5 min at 72 \textdegree C. The DNA was purified with DNA Clean-Up \& Concentration kit (Zymo Research) and sent to Innovative Genomics Institute NGS core facility for PCR2 amplification and sequencing. The library was sequenced with NextSeq 1000/2000 P1 Reagents (300 Cycles) or MiSeq Reagent Kit v2 (600 cycles).
Selection experiments for the comparison with ABE1.1, ABE7.10 and ABE8 were performed in the same selection strain as the library experiments. For the titer plate assays 25 uL of competent cells were transformed with 50 ng of each plasmid, cells were electroporated and incubated at 37\textdegree C for recovery in 475 $\mu$L SOC. They were added to 2 mL 2XYT media supplemented with 50 $\mu$g/mL kanamycin kanamycin, 36 $\mu$g/mL chloramphenicol and 0.1 \% arabinose and induced for 6h. Serial 10-fold dilutions were made on selective (400 $\mu$g/mL carbenicillin and 36 $\mu$g/mL chloramphenicol) and non-selective antibiotics (50 $\mu$g/mL kanamycin and 36 $\mu$g/mL chloramphenicol) and the ratio of colonies on selective over non-selective condition was caluclated.
For the pooled experiment with ABE1.1, ABE7.10 and ABE8 all plasmids were mixed in a 1:1 ratio and 100 ng of the plasmid mix was transformed in the selection strain. Transformation and induction were performed identically to the library experiments except that cells were induced for 6h. After the selection the plasmid DNA was extracted, the ABE region was amplified by PCR and sent for nanopore sequencing. Log enrichment was calculated identically as in the library experiments.  

\subsubsection{Processing sequencing data}
The paired-end $2\times 150$ base pair sequencing reads exhibited variable-length overlapping regions resulting from additional nucleotides incorporated between the Illumina adapters and the primer annealing regions. To process and merge these reads, we first aligned the paired reads to the primer binding sites at both the 5' and 3' termini. In the event of nucleotide call disagreement between the forward and reverse reads in the overlapping region of the two reads, the nucleotide call with the higher Phred score was retained. These merged reads were then matched against the designed library of DNA sequences, keeping only reads with exact matches. To calculate log-enrichments between the selective and non-selective conditions, we added a pseudocount of 1 to all read counts for both conditions. The log-enrichment, $\log e_i$, for each library member indexed by $i$ was computed as:

\begin{align}
\log e_i &= \log \left(\frac{n_i^{\text{s}}/{N^{\text{s}}}}{n_i^{\text{ns}} / N^{\text{ns}}}\right),
\end{align}
where $n_i^{\text{s}}$ and $n_i^{\text{ns}}$ represent the read counts for library member $i$ in the selective and non-selective conditions, respectively, and $N^{\text{s}}$ and $N^{\text{ns}}$ denote the total number of read counts in the selective and non-selective conditions, respectively.

\subsubsection{Pre-trained sequence generative models}
We hypothesized that the structure of the wild-type TadA protein contains important constraints for maintaining a baseline editing activity, from which we could seek to improve. Therefore, we chose to use sequence generative models that we are able to explicitly condition on the wild-type TadA backbone structure. Specifically, we explored the use of two different pre-trained sequence generative models, ESM3 and FMIF (Flow-Matching Inverse Folding), an inverse folding model which we trained with very similar training data and architecture as ProteinMPNN, but with the explicit flow-matching training objective. We selected FMIF over ProteinMPNN to demonstrate the applicability of our guidance methodology to flow-matching models, thereby establishing the broader generalizability of our approach. For ESM3, we used it as an inverse-folding model by providing the backbone structure as conditioning inputs to the model when generating sequences.

We trained the FMIF model using the same dataset and training protocol described in~\citet{nisonoff2025unlocking} with two exceptions: (i) we used 3 encoder and decoder layers instead of 4, as both hyperparameter settings performed equally well on a held-out test set and (ii) we removed all examples of engineered ABE variants from the 2021 multi-chain training dataset provided by~\citet{dauparas2022robust}. This was done in order to ensure that FMIF does not leverage helpful information for engineering ABEs that is generally not available when engineering a new protein. Specifically, we removed PDB entry 6VPC chains E and F from the training set.

\subsubsection{First-round library design}
We built the first-round library by designing two sets of sequences with ESM3 and FMIF separately, then combined the sequences into a pooled library. The wild-type TadA is a homo-dimer, and we sought to adhere to this constraint when generating the sequences. To this end, at each sampling step, we obtained the output logits of the models when conditioning on the structure of homo-dimer. That is, if the protein is of length $L$, the output logits is of size $2L \times S$, where $S$ is the alphabet size of the model. We then took the average between the logits which correspond to the same position on sequence of the monomer following~\citet{dauparas2022robust}. We used the averaged logits to parameterize the rate matrix used for sampling. As a result, at each sampling step, we only generate $L$ residues, one for each position on the monomer, while respecting the structural constraints of the homo-dimer.

In many practical protein engineering applications, it is often desirable to have a control over the number of mutations introduced to the wild-type. In our case, we hoped to have a range of mutation counts in the first-round library, since we do not know a priori how tolerant TadA is to mutations. To this end, we used an additional sampling hyperparameter used in ProteinMPNN~\citep{dauparas2022robust} which we called the \textit{wild-type weight}, $w$. At each sampling step, the output logits of the model is biased by the wild-type weight with $l^w_{ds} = l_{ds} + w$ for every position $d$ of the sequence, where $l_{ds}$ is the original output logits of the model for position $d$ and state $s$, with $s$ being the state of the wild-type sequence at position $d$. Setting $w=0$ recovers sampling from the original model, and increasing $w$ allows one to gradually decrease the number of mutations made to the wild-type sequence, where $w \rightarrow \infty$ will result in only the wild-type sequence being sampled. The biasing of the logits by the wild-type weight can be interpreted as sampling from a weighted mixture between the original distribution parameterized by the model, and a delta distribution centered on the wild-type sequence. We note that in principle, constraining the number of mutations can also be imposed by introducing a guidance term that encourages the desired number of mutations, but doing so in a computationally tractable way might be challenging, so we leave this investigation to future work.

In our first-round library, we designed around 1,000 sequences with ESM3 and FMIF, respectively, for a total library size of around 2,000 sequences. In addition, we kept three copies of the wild-type TadA sequences in the library as controls. We chose to only design the region of 86 residues on the C-terminal region of the TadA protein due to the cost constraints of gene synthesis, and we hypothesized this region might be enriched for mutations important for increasing editing activity based on prior studies~\citep{gaudelli2017programmable,richter2020phage,xiao2024adenine}. We fixed active site residues Pro86, Cys87, and Cys90 and allowed the model to design the rest of the region. For both ESM3 and FMIF, we used 8 different values of wild-type weight $w$ linearly spaced between 0 and 3.5. 
We sampled from both ESM3 and FMIF using Euler integration with a timestep $\Delta t = 0.002$, with a temperature of 0.5 and sampling stochasticity $\eta = 20$.

\subsubsection{Predicting ABE activity from the first-round library}
The log-enrichment data collected from the first round of designs was used to design a second library enriched in highly-active base editors. Specifically, we sought to generate sequences with an average first-round log-enrichment score of at least $0.0$ (\suppfref{si_fig:round1scatter}). To do this, we trained a classifier to predict, given a sequence, whether the sequence's log-enrichment is greater than zero or not. Guidance requires a classifier trained not only on ``clean'' / ``un-noised'' sequence-label pairs, $\{x_1^i, y^i\}$, but also trained on ``noised'' sequence-label pairs $\{x_t^i, y^i\}$, $t\in [0,1]$. This was done in two steps: first, a classifier was trained on just the ``un-noised'' data and then the classifier was further trained on the ``noised'' data.

First, the data was pre-processed by discarding sequences with highly-variable log-enrichments between the two replicates. Specifically, any sequences for which the absolute value of the difference in log-enrichments between the two replicates exceeded $0.25$ were discarded. This step filtered $114$ out of $2,000$ sequences. Next, to further mitigate the influence of experimental noise, we discarded sequences with average log-enrichment scores very close to the class boundary. Specifically, any sequences for which the average log-enrichment was between $-0.25$ and $0.25$ were discarded. This step filtered an additional $348$ sequences, resulting in a dataset of $132$ positive sequences (\ie~mean log-enrichment $\ge 0.25$) and $1,406$ negative sequences (\ie~mean log-enrichment $\le -0.25$).

Given the relatively small dataset size and significant class imbalance, we implemented a linear classifier augmented with pairwise interaction terms. This approach draws inspiration from Potts models, which are exponential family models that effectively capture pairwise dependencies between positions and have demonstrated robust empirical performance in protein sequence modeling~\citep{hopf2017mutation}. We trained the model using PyTorch~\citep{paszke2019pytorch,ansel2024pytorch}, minimizing the binary cross-entropy loss with the AdamW optimizer~\citep{kingma2015adam,loshchilov2018decoupled} with a regularization parameter of $10.0$ applied specifically to the pairwise interaction terms. Training proceeded for $400$ epochs using full-batch gradient descent until convergence of the loss function was observed. Across $10$ class-stratified 80\%/20\% train/validation splits, the model achieved an average validation AUROC of $0.92$.

Next, to adapt the classifier to also predict class labels for ``noisy'' data, we froze all single-site terms involving non-mask tokens as well as any pairwise interaction terms between two non-mask tokens (\ie~we only further trained single site parameters corresponding to mask tokens and trained pairwise interaction terms involving at least one mask token). For each gradient step, we first sampled the entire un-noised dataset, then for each sequence in the dataset we randomly sampled a time between $0$ and $1$, and masked the sequence by sampling from the forward noising process up to that time. The model was again trained for $400$ epochs without any regularization, at which point the loss function converged. Training the model in this way, wherein the parameters for the non-mask terms are frozen, has the nice property that the noisy model is gauranted to be equivalent to the non-noisy model at $t=1$.

\subsubsection{Second-round guided ABE library design}
We generated the second-round library by guiding ESM3 and FMIF respectively with the classifier we trained based on the experimental data collected on sequences from the first round. We used the same sampling hyperparameters for each model as we did in the first round. We used exact guidance with guidance strength $\gamma=100$. We designed 944 sequences with each model. All of the generated sequences had a predicted probability of being highly active greater than 0.5. We also included 112 sequences from the first-round library as controls so we could calibrate the log-enrichment measurements between the two rounds. We selected these 112 control sequences with by sampling uniformly across the range of observed log enrichment in the first round.

The 1,888 designed sequences and 112 control sequences were assayed and sequenced with the same protocols as were done in the first round (\suppfref{si_fig:round2scatter}). To enable direct comparison between the first- and second-round libraries, we included 112 sequences from the first round as controls in the second-round library (\suppfref{si_fig:tadA_mapping}a). These control sequences were selected to span the full range of observed log enrichment values from the first round, providing calibration points across the entire activity spectrum. This design allowed us to quantify and correct for batch effects between the two experimental rounds.
We used these control sequences to learn a linear transformation that maps log enrichment values from the second round to their corresponding values in the first round. Specifically, we fit a linear regression model:
\begin{equation}
\log e_{\text{round 1}} = \alpha \cdot \log e_{\text{round 2}} + \beta
\end{equation}
where $\log e_{\text{round 1}}$ and $\log e_{\text{round 2}}$ represent the log enrichment values for the same sequence in rounds 1 and 2, respectively, and $\alpha$ (slope) and $\beta$ (intercept) are the parameters of the linear transformation.

To assess the uncertainty in this calibration procedure, we employed a bootstrap approach using the 112 control sequences. Specifically, we performed 1,000 bootstrap iterations, each time randomly sampling 112 sequences with replacement from our control set and fitting a linear regression model. For each bootstrap sample, we derived a unique set of mapping parameters ($\alpha_i$, $\beta_i$) and used these to transform the entire second-round dataset. We then calculated the median log enrichment value of the transformed second-round sequences for each bootstrap iteration, generating a distribution of potential outcomes that reflects the inherent variability in our experimental measurements.
From this distribution of bootstrap outcomes, we identified the 5th and 95th percentile mappings based on the median log enrichment values. These percentiles represent conservative and optimistic estimates, respectively, of the transformation from round 2 to round 1 log enrichment values (\suppfref{si_fig:tadA_mapping}b).



%% file: supplementary_text.tex
\section{Supplementary Text}

\subsection{Detailed discussion on related work}
\label{sec:si-related-work}

Here we provide a more in-depth discussion on the related work to this work. Since \ourmethod brings technical innovation that are broadly applicable to generative models to the problems in protein engineering, we broadly categorized the related literature into two categories: (i) those that are more related to the technical underpinning of guidance (see subsection ``Related work to guidance''), and (ii) those that are more related to the use of pre-trained models and experimental data in protein engineering (see subsection ``Related work in protein engineering''). For related work that are more specific to the results showing the equivalence of sampling methods used for discrete flow matching and \aoamshorts, see the Methods section ``Simplified and faster sampling for flow matching and diffusion models with masked noise processes''.

\subsubsection{Related work to guidance}

Several recent work have proposed related but different techniques from this work for steering discrete state-space diffusion models. These approaches can be broadly categorized into on-the-fly approaches (which modify the sampling process at inference time) and training-based approaches (which update the parameters of the pre-trained model).

\medskip\noindent
\textbf{On-the-fly approaches.} 
Historically, the general problem of sampling from challenging distributions has also been tackled through Markov chain Monte Carlo (MCMC) methods, and some MCMC-based methods have also been proposed for guiding discrete-space generative models. \citet{emami2023plug} proposed PPDE, which uses the Metropolis--Hastings (MH) algorithm to sample from product-of-experts (PoE) distributions on protein sequence space, where the target is defined via an (unnormalized) sequence-level score combining an ``evolutionary density'' expert and one or more property experts. The key distinction between that approach to \ourmethod are that PPDE requires taking gradients through the pre-trained model, which may be prohibitively large; it requires computing (unnormalized) likelihood from the generative model, which is very expensive for most PLMs; and, as it can be quite sensitive to its initialization, potentially having a mixing time that scales exponentially in the sequence length. In contrast, \ourmethod does not require gradients (or optionally through the predictor, does not require likelihood computations, runs in linear time in sequence length, and does not require a wild-type sequence to initialize sampling. We do note that PPDE can be a useful technique to \emph{refine} generations from \ourmethod, particularly when one desires additional local improvements from a good starting point (e.g., a wild-type or a strong candidate) after generating more diverse candidates with \ourmethod.

Some recent works~\citep{li2024derivative, singhal2025a} have also proposed techniques for guiding discrete-space diffusion models using Sequential Monte Carlo (SMC), building upon works first proposed for continuous-space diffusion models~\citep{wu2023practical}. Importantly, \ourmethod is also complementary to SMC-based approaches as these techniques can be combined with~\ourmethod to improve generation with some additional inference-time costs, such as by way of using multiple particles with proposal distributions defined via TAG. Indeed, through the lens of MCMC, one can also view~\ourmethod as a way to design proposal distributions with minimal rejections by exploiting the structures of the underlying generative process, in a similar spirit to other discrete-space MCMC methods~\citep{grathwohl2021oops}. \citet{lee2025debiasing} observed that when applying stronger guidance, sampling can also be improved with SMC, aligning with challenges in low-temperature sampling noted for continuous-space diffusion models~\citep{du2023reduce,ingraham2023illuminating}. For autoregressive (AR) models, \citet{zhao2024probabilistic} proposed an SMC-based steering method that learns intermediate twist functions via contrastive learning, similar to learning time-dependent predictors in diffusion guidance.

\citet{vignac2022digress} propose DiGress, approximate guidance approach in a discrete-time,
discrete state-space diffusion (D3PM)~\citep{austin2021structured} framework, which \citet{wang2024diffusion} (DPLM) apply to generate protein sequences. In our earlier work~\citep{nisonoff2025unlocking}, we proposed methods for both exact and Taylor-approximate Guidance (TAG), the second of which bears resemblance to DiGress due to both of them using a Taylor series approximation. In our earlier work~\citep{nisonoff2025unlocking}, we discussed the difference between DiGress/DPLM and \ourmethodICLR in detail, and performed comprehensive comparisons to DiGress in three experiments across different domains, generally showing \ourmethodICLR to be superior. Here we briefly discuss the distinction and refer the readers to~\citet{nisonoff2025unlocking} (\eg, Appendix section E.2) for more details. DiGress is fundamentally a method for discrete-time, discrete-space diffusion models. It does not specify how to guide either continuous-time, discrete-space diffusion models or discrete-space flow-matching models. \ourmethod~naturally encompasses these model classes, and through our equivalence can be applied to any-order autoregressive models (\aoamshorts). Both theoretically and empirically, the continuous time training objective has been demonstrated to perform better~\citep{campbell2022continuous,lou2023discrete,campbell2024generative,gat2024discrete}. In fact, one could view DiGress as a special case of our general framework by taking a specific discrete time approximation to the continuous time process. We note that this approximate imposes a uniform time grid and is therefore distinct from our construction of \textit{equivalent} discrete time processes for continuous time models. Accordingly, their formulation would not directly be applicable to an \aoamshort without additional assumptions. Orthogonally, \citet{yang2021fudge} proposed FUDGE for guiding autoregressive models, which can be viewed as a special case of our exact guidance method in the case of AR model, with \aoamshorts generalizing AR models. 

As noted in \citet{schiff2025simple}, which also presented a discrete-time analog of the continuous-time guidance method in \citet{nisonoff2025unlocking}, in discrete time, the normalizing constant in Equation~\ref{eq:transition_bayes} is still intractable, necessitating additional assumptions to achieve tractability, whereas \citet{nisonoff2025unlocking} showed TAG does not require this assumption. In addition, the continuous-time formulation of discrete diffusion models provides exactly the tractability necessary to perform \emph{exact} guidance, whereas DiGress is fundamentally approximate. \ourmethod enables exact guidance both for general noising processes~\citep{nisonoff2025unlocking} and for masking noise process. In particular, we showed that the model classes discussed---masked diffusion/flow models, masked language models (MLMs), and any-order autoregressive models (\aoamshorts)---are not only equivalent in their training objective, but also the ways they can be sampled (see Methods section ``Simplified and faster sampling for flow matching and diffusion models with masked noise processes''). Through this same equivalence, we develop a simpler and faster algorithm for sampling sequences from diffusion and flow-matching models that use masking noise processes, leading to a more computationally efficient guidance algorithm than the one originally proposed in~\citet{nisonoff2025unlocking}, which we called \textit{discrete-time exact guidance (\ournewsamplingmethod)}. The original method was already shown to at times improve performance over TAG, but we further employed \ournewsamplingmethod to our \textit{in silico} experiments to illustrate its effectiveness (as seen in the enzyme and multi-property guidance experiments, where TAG was less effective). 

There is another line of work that model discrete data by embedding the data into continuous state-spaces, either on a probabilistic simplex~\citep{avdeyev2023dirichlet, stark2024dirichlet} or on a continuous latent space~\citep{li2022diffusion, dieleman2022continuous}. Some of the approaches proposed to apply guidance to diffusion models on the continuous state-space representations of discrete state-space objects in order to utilize the guidance approaches for continuous state-space diffusion. We note that approaches which perform guidance in continuous spaces lose the discrete structure of the data during generation, which can be important for settings where the discrete structures contain information that are useful for guidance~\citep{campbell2024generative}. For more detailed discussion for other inference-time techniques, we refer the readers to a recent review article by~\citet{uehara2025inference}.

\medskip\noindent
\textbf{Training-based approaches.} 
Recent works have also explored updating model parameters to incorporate conditioning rather than modifying the sampling process. \citet{wang2025finetuning} proposed to use the Gumbel-Softmax trick to make discrete diffusion trajectories differentiable and enable backpropagation of rewards through entire sampling paths. Concurrently, \citet{rector-brooks2025steering} proposed to frame the steering problem as sampling from a Bayesian posterior where the pre-trained model serves as the prior and the reward acts as the likelihood. Both fine-tuning approaches established connections to our earlier work~\citep{nisonoff2025unlocking} while offering complementary perspectives on solving the steering problem. Direct Preference Optimization (DPO~\citep{rafailov2023direct}) and related work~\citep{widatalla2024aligning,chennakesavalu2025aligning} enable updating of pre-trained models, although up until recently, these methods have been applicable only to models with tractable likelihoods, thereby restricting them largely to autoregressive models. Further research is needed to understand and improve the quality of such approximations.

Reinforcement Learning from Human Feedback (RLHF) offers an alternative approach to model alignment through methods like Proximal Policy Optimization (PPO)~\citep{schulman2017proximal}. PPO does not technically require the retrained model to have tractable likelihoods, but it requires an explicit reward model and on-policy training, making it often more complicated to train than more direct approaches like DPO~\citep{rafailov2023direct}. In addition, the standard RLHF pipeline typically begins with fine-tuning, for which the statistical outcome is not well-defined, and for which important information from the already-trained model may be forgotten. For a more detailed review of training-based methods, specifically those anchored in reinforcement learning and diffusion models, we refer the readers to a review article by~\citet{uehara2024understanding}.

\subsubsection{Related work in protein engineering}
With the increasing interests of using generative models---both for structures~\citep{watson2023novo, ingraham2023illuminating} and for sequences~\citep{dauparas2022robust, hayes2025simulating, alamdari2023protein, madani2023large, hsu2022learning}---in protein engineering, many of the methods mentioned above have seen adaptations to protein engineering problems, some of which are summarized in a recent review article by~\citet{stocco2025guiding}. Below we categorized these applications by the types of methods they employed. 

Recent work by \citet{yang2025steeringgenerativemodelsexperimental} benchmarked \ourmethodICLR~\citep{nisonoff2025unlocking} when used with family-specific pre-trained models on a set of protein fitness optimization tasks. They found that \ourmethodICLR~outperforms several fine-tuning and other ``plug-and-play'' guidance methods (\eg, methods that perform guidance in continuous latent spaces). In addition, they found that guidance-based approaches require less hyperparameter tuning and have lower computational costs, making them more practically accessible to everyday users. However, these earlier work did not address how to perform guidance for models other than diffusion and flow matching, thereby limiting its utility to the field of protein engineering and beyond. In contrast, by unifying a broad class of discrete-space generative models under a single framework, we are able to directly apply guidance to popular pre-trained models commonly used by practitioners. Accordingly, we do not compare to many of the methods benchmarked in~\citet{yang2025steeringgenerativemodelsexperimental} again, but instead focus on demonstrating the efficacy of \ourmethod on a wider breadth of challenging design scenarios and on a real design campaign.

\medskip\noindent
\textbf{On-the-fly approaches.}
In the context of protein sequence generative models, there are prior work that embed protein sequences into continuous state-spaces and thereby directly leveraging guidance techniques developed for continuous-space diffusion models. For example, \citet{gruver2024protein} propose to perform guidance in the hidden layers of the unconditional diffusion model by jointly training the unconditional denoising model and the classifier on one hidden layer of the denoising model. In another example, \citet{lisanza2024multistate} propose to fine-tune RoseTTAFold with Gaussian diffusion on the one-hot encoded sequence space in order to perform guidance in this continuous state-space. 

Another style of approaches that have been very recently developed in the context of large language models involve directly manipulating the intermediate outputs of the model (\eg, activations)~\citep{turner2023steering}, and it has only very recently been adapted to protein language models (PLMs) to steer generations toward targeted protein properties~\citep{huang2025steering}. Concretely, these methods compute a ``steering direction'' in representation space (\eg, from differences in hidden activations associated with an attribute) and then inject this direction into selected layers during the forward pass to bias the model’s next-token (or sequence) distribution, without updating model parameters. Because this line of work is quite recent and protein properties can depend on highly nonlocal and context-dependent constraints, its robustness and generality across objectives, proteins, and model families remains an active area of research.

\medskip\noindent
\textbf{Training-based approaches.}
DPO has been applied to a few autoregressive protein generative models. For example, \citet{widatalla2024aligning} applied DPO to ESM-IF1, an autoregressive inverse folding model, to incorporate experimental stability data. In another example, \citet{stocco2024guiding} applied DPO to an autoregressive model conditionally trained on enzyme classes to guide the generation towards various enzyme properties. As for RLHF-based approach, \citet{blalock2025functional} recently applied PPO to align protein language models (\eg, ESM2) with experimental measurements. Notably, this approach does not directly sample from generative models in the conventional sense; instead, it employs a position-wise scoring strategy where models evaluate mutation probabilities at individual positions sequentially, followed by an iterative selection process that combines high-scoring mutations. 

\medskip\noindent
\textbf{Zero-shot scoring-based approaches.}
A complementary line of work uses pre-trained protein language models primarily as \emph{scoring} functions to guide iterative mutation and selection, rather than sampling from an explicit conditional generative model. For example, EVOLVEpro~\citep{jiang2024rapid} and related methods use language-model-derived scores within iterative optimization or active-learning-style loops to propose and evaluate candidate mutations.
Similarly, \citet{hie2024efficient}~demonstrate efficient evolution of human antibodies using general protein language models, emphasizing iterative selection guided by model scores rather than direct conditional sampling from a guided generative process. These approaches are highly practical for local improvement starting from strong initial sequences, but are conceptually distinct from plug-and-play conditional sampling methods that directly reweigh the generative trajectory.

\medskip\noindent
\textbf{Supervised fitness optimization.}
A large body of work in machine-learning-assisted directed evolution trains supervised predictors \(p(y\mid x)\) on experimentally measured variants and then uses these predictors to select new variants to test, often in active learning or Bayesian optimization style loops. Examples include ML-assisted directed evolution with combinatorial libraries~\citep{wu2019machine} and subsequent work emphasizing the importance of training set design and ``holey'' fitness landscapes~\citep{wittmann2021informed}, and has explored in more data-efficient regime~\citep{biswas2021low}. These methods can be viewed as attempting to invert a learned predictor to propose improved sequences, but unlike generative-model-based approaches they do not provide an explicit conditional sequence distribution and typically rely on local search or proposal mechanisms.

\newpage

\subsection{Detailed derivations and proofs}
\label{sec:si-proofs}

Below we included the detailed derivations and proofs for the statements and lemmas in the Methods section ``Statement of exact MFM sampling using AO-ARMs''. 

\subsubsection{Proofs of exact MFM sampling using AO-ARMs}

\begin{proof}[Proof of Proposition~\ref{prop:main}]
Since CTMCs are, by definition memoryless (depend only on the previous time and state), the path probability density decomposes into a product of densities for each transition conditioned on the last. This can be written as 
\begin{align}
&\Pm\left(\bm{\tau}, \bm{X}\right) \notag \\
    = &\Pm(\tau_0, x_{\tau_0})
    \times \prod_{i=1}^D \Pm(\tau_i, x_{\tau_i} | \tau_{i-1}, x_{\tau_{i-1}}) 
    \times \Pm(\tau_{D+1} > 1 | \tau_D, x_{\tau_D}). \notag \\
    \intertext{The MFM always starts at time $\tau_0=0$ with the fully masked sequence \mbox{$x_0=\{M\}^D$} so $\Pm(\tau_0, x_{\tau_0})=1$. There are also exactly $D$ transitions because we assume $0$ stochasticity and no corrector sampling. This means $\Pm(\tau_{D+1} > 1 | \tau_D, x_{\tau_D})=1$ as well. The above now simplifies to}
    = &\prod_{i=1}^D \Pm(\tau_i, x_{\tau_i} | \tau_{i-1}, x_{\tau_{i-1}}) \notag \\
    = &\prod_{i=1}^D \Pm(\tau_i | \tau_{i-1}, x_{\tau_{i-1}}) \Pm(x_{\tau_i} | \tau_i, x_{\tau_{i-1}}). \notag \\
    \intertext{We can apply Lemma~\ref{lemma:transition-time-decomposition} to remove the state dependence from the conditional distribution over jump times, $\Pm(\tau_i | \tau_{i-1}, x_{\tau_{i-1}})$. Similarly, we can apply Lemma~\ref{lemma:transitions-are-same} to remove the time dependence from the conditional distribution over jump states, $\Pm(x_{\tau_i} | \tau_i, x_{\tau_{i-1}})$. This is the key step of the proof, allowing us to separate the time- and state-dependent distributions into separate terms as follows}
    = &\prod_{i=1}^D \Pm(\tau_i | \tau_{i-1}) \Pm(x_{\tau_i} | x_{\tau_{i-1}}) \notag \\
    = &\left(\prod_{i=1}^D \Pm(\tau_i | \tau_{i-1})\right) \left(\prod_{i=1}^D \Pm(x_{\tau_i} | x_{\tau_{i-1}})\right). \notag \\
    \intertext{In Lemmas~\ref{lemma:transition-time-decomposition} and~\ref{lemma:transitions-are-same} we also calculate closed-form expressions 
    for $\Pm(\tau_i | \tau_{i-1})$ and $\Pm(x_{\tau_i} | x_{\tau_{i-1}})$, which can be substituted in and  simplified as follows
    }
    = &\prod_{i=1}^D \dot{\kappa}_{\tau_i} \times \prod_{i=1}^D p_{1|\tau_i}(X_1=x_{\tau_i} | X_{\tau_i}^{-}=x_{\tau_{i-1}}). \notag \\
    \intertext{Now, we will rewrite these products to directly match the proposition statement. First, the interpolation schedule $\kappa_t$ is the CDF of the jump time under the assumption of minimal transitions. This means the PDF of the jump time for a given position is $\Pm(\tau=t)=\frac{d}{dt}\Pm(\tau<t)=\dot{\kappa}_t$. Generally, we choose $\kappa_t=t$ so the PDF of the jump times is uniform $\Pm(t=\tau)=\dot{\kappa}_t=1=U_{[0,1]}(t=\tau)$:}
    = &\prod_{i=1}^D U[0,1](\tau_i) \times \prod_{i=1}^D p_{1|\tau_i}(X_1=x_{\tau_i} | X_{\tau_i}^{-}=x_{\tau_{i-1}}). \notag \\
    \intertext{Finally, we recognize that the \aoamshort\ constructed for the sampling algorithm uses the MFM denoiser as its conditional distribution. This means the second product above is just the conditional probability of the sample under the \aoamshort\, and the expression simplifies to }
= &\prod_{i=1}^D U[0,1](\tau_i) \times \prod_{i=1}^D p(X^{\sigma(i)}=x_{\tau_i} | X^{\sigma(<i)}=x_{\tau_{i-1}}) \notag \\
= &U^D[0,1](\bm{T}=\sigma^{-1}(\bm{\tau})) \times \Po(X=X_{\tau_D} | \sigma) \notag
\end{align}
as claimed.
\end{proof}

\begin{proof}[Proof of Corollary~\ref{corollary:sample-permutation}]
This follows directly from Proposition~\ref{prop:main}. Originally, we found
\begin{align*}
\Pm(\tau, \bm{X}) &= \Po(X_{\tau_D} | \sigma) U^D[0, 1](\bm{T}=\sigma^{-1}(\bm{\tau})) \\
\intertext{To be able to sample both the times and permutations independently, upfront, our sampler must find the permutation $\pi$ that maps the sampled transition times $\bm{T}$ to sampled positions. Then applying $\sigma$ to $\pi(T)$ sorts the times to be $\bm{\tau}$. In practice, we would just assign the sampled times to positions in the order specified by $\sigma$. For the purposes of this proof, however, $\pi$ lets us rewrite this expression as:}
&= U[S_D](\sigma) \Po(X_{\tau_D} | \sigma) U^D[0, 1](\bm{T}=\pi^{-1}(\sigma^{-1}(\bm{\tau})))  \\
\intertext{Finally, we apply the definition of OA-AR sampling}
&= \Po(\bm{X}) U^D[0, 1](\bm{T})
\end{align*}
\end{proof}

\subsubsection{Statements and proofs of lemmas}
\label{sec:lemmas}

\begin{lemma}\label{lemma:transition-time-decomposition}
The conditional probability of the next jump time for an MFM depends only on the previous jump time and can be expressed in closed form as:
\begin{align}
\mathbb{P}_{\mathrm{MFM}}\left(\tau_i | \tau_{i-1}, x_{\tau_{i-1}}\right) 
&= \mathbb{P}_{\mathrm{MFM}}\left(\tau_i | \tau_{i-1} \right) \notag \\ 
&=(D-(i-1)) \frac{\dot{\kappa}_{\tau_{i}}}{1-\kappa_{\tau_{i-1}}} \left(\frac{1-\kappa_{\tau_{i}}}{1-\kappa_{\tau_{i-1}}} \right)^{D-i}. \label{eq:jump-time}   
\end{align}
\end{lemma}
This is the probability that any \textit{one} of the remaining $D-(i-1)$ masked positions transitions at time $\tau_i$,\footnote{$\frac{\dot{\kappa}_{\tau_{i}}}{1-\kappa_{\tau_{i-1}}}=\Pm(\tau_i|\tau_i>t)$} given that none of them transitioned between $\tau_{i-1}$ and $\tau_i$ and all other positions have already been unmasked.

The core insight from the calculation below is that, in a masking process, the time to the next jump only depends on state through the total number of masked tokens: the rate of leaving the mask state is the same across positions, and once unmasked, the state is fixed.\footnote{Recall, there are exactly $D$ jumps because we are using the zero-stochasticity rates from \cite{gat2024discrete, campbell2024generative}.} As a result, once one specifies that we're sampling the \textit{$i$-th} jump time, \textit{we know exactly how many masked tokens are left} and there is no further dependence on the CTMC state.
\begin{proof}
\begin{align*}
\intertext{We start by rewriting the PDF of the holding time as the derivative of its CDF, }
  &\mathbb{P}_{\mathrm{MFM}}\left(\tau_i | \tau_{i-1}, x_{\tau_{i-1}}\right) \\
  = &\frac{\partial}{\partial \tau_i}\mathbb{P}_{\mathrm{MFM}}\left(T_i < \tau_i | \tau_{i-1}, x_{\tau_{i-1}} \right) \\
  \intertext{which, for CTMCs has the standard form}
  = &\frac{\partial}{\partial \tau_i} \left( 1 - \exp\left(\int_{\tau_{i-1}}^{\tau_i}R_s(x_{\tau_{i-1}}, x_{\tau_{i-1}})ds\right) \right).\\
  \intertext{The rates for two positions transitioning at once are zero \cite{campbell2024generative} so we can factor this expression for the sequence not changing into a product of the probabilities for each position not changing.} 
    = &\frac{\partial}{\partial \tau_i} \left( 1 - \prod_{i=1}^D \exp\left(\int_{\tau_{i-1}}^{\tau_i}R^{\sigma(i)}_s(x_{\tau_{i-1}}^{\sigma(i)}, x_{\tau_{i-1}}^{\sigma(i)})ds\right) \right).\\
  \intertext{We can also substitute in the explicit form of the rates (Eq.~\eqref{eq:rates-param}), taking care that these are the rates for staying at $x_{\tau_{i-1}}$}
  = &\frac{\partial}{\partial \tau_i} \left( 1 - 
      \prod_{i=1}^{D}
        \exp\left(
          \int_{\tau_{i-1}}^{\tau_i}
            \frac{\dot{\kappa}_s}{1-\kappa_s}\left(p^\theta_{1|s}(X_1^{\sigma(i)}=x_{\tau_{i-1}}^{\sigma(i)} | X_{s}^{-}=x_{\tau_{i-1}})-\delta_{x_{s}^{\sigma(i)}}(x_{\tau_{i-1}}^{\sigma(i)})\right)
          ds
        \right)
    \right)\\
  = &\frac{\partial}{\partial \tau_i} \left( 1 - 
      \prod_{i=1}^{D}
        \exp\left(
          \int_{\tau_{i-1}}^{\tau_i}
            -\frac{\dot{\kappa}_s}{1-\kappa_s}\left(1 - p^\theta_{1|s}(X_1^{\sigma(i)}=x_{\tau_{i-1}}^{\sigma(i)} | X_{s}^{-}=x_{\tau_{i-1}})\right)
          ds
        \right)
    \right)\\
  \intertext{Because these MFM only make the minimal number of transitions, if the current state at position $i$ is unmasked, we know it will not change. Therefore, $p^\theta_{1|t}(\cdot | X_t \neq M)=\delta_{X_t}(\cdot)$ and we can drop the terms corresponding to the unmasked positions to get}
   = &\frac{\partial}{\partial \tau_i} \left( 1 - 
      \prod_{j \in \sigma(\geq i)}
        \exp\left(
          \int_{\tau_{i-1}}^{\tau_i}
            -\frac{\dot{\kappa}_s}{1-\kappa_s}\left(1 - p^\theta_{1|s}(X_1^{j}=x_{\tau_{i-1}}^{j} | X_{s}^{-}=x_{\tau_{i-1}})\right)
          ds
        \right)
    \right).\\
\intertext{We can now drop the state dependence entirely by noting that $p^\theta_{1|t}(M | \cdot )=0$ because masked tokens are never in the denoised sequence $X_1$:}
  = &\frac{\partial}{\partial \tau_i} \left( 1 - 
      \prod_{j \in \sigma(\geq i)}
        \exp\left(
          \int_{\tau_{i-1}}^{\tau_i}
            -\frac{\dot{\kappa}_s}{1-\kappa_s}
          ds
        \right)
    \right).\\
\intertext{Simplifying, we find}
  = &\frac{\partial}{\partial \tau_i} \left( 1 - 
      \prod_{j \in \sigma(\geq i)}
        \exp\left(
          \int_{\tau_{i-1}}^{\tau_i}
            \frac{d}{ds}\ln(1-\kappa_s)
          ds
        \right)
    \right)\\
  = &\frac{\partial}{\partial \tau_i} \left( 1 - 
      \prod_{j \in \sigma(\geq i)}
        \frac{1-\kappa_{\tau_{i}}}{1-\kappa_{\tau_{i-1}}}
    \right)\\
  = &-\frac{\partial}{\partial \tau_i} 
        \left(\frac{1-\kappa_{\tau_{i}}}{1-\kappa_{\tau_{i-1}}} \right)^{D-(i-1)}\\
  = &(D-(i-1)) \frac{\dot{\kappa}_{\tau_{i}}}{1-\kappa_{\tau_{i-1}}} \left(\frac{1-\kappa_{\tau_{i}}}{1-\kappa_{\tau_{i-1}}} \right)^{D-i}  \\
  = &(D-(i-1)) P(\tau_i | \tau_i > \tau_{i-1})  \prod_{j \in \sigma(> i)} P(\tau_j > \tau_i | \tau_i > \tau_{i-1})
\end{align*}
as claimed.
\begin{align*}
\intertext{We start by rewriting the PDF of the holding time as the derivative of its CDF, }
  &\mathbb{P}_{\mathrm{MFM}}\left(\tau_i | \tau_{i-1}, x_{\tau_{i-1}}\right) \\
  = &\frac{\partial}{\partial \tau_i}\mathbb{P}_{\mathrm{MFM}}\left(T_i < \tau_i | \tau_{i-1}, x_{\tau_{i-1}} \right) \\
  \intertext{which, for CTMCs has the standard form}
  = &\frac{\partial}{\partial \tau_i} \left( 1 - \exp\left(\int_{\tau_{i-1}}^{\tau_i}R_s(x_{\tau_{i-1}}, x_{\tau_{i-1}})ds\right) \right).\\
  \intertext{The rates for two positions transitioning at once are zero \cite{campbell2024generative} so we can factor this expression for the sequence not changing into a product of the probabilities for each position not changing.} 
    = &\frac{\partial}{\partial \tau_i} \left( 1 - \prod_{i=1}^D \exp\left(\int_{\tau_{i-1}}^{\tau_i}R^{\sigma(i)}_s(x_{\tau_{i-1}}^{\sigma(i)}, x_{\tau_{i-1}}^{\sigma(i)})ds\right) \right).\\
  \intertext{We can also substitute in the explicit form of the rates (Eq.~\eqref{eq:rates-param}), taking care that these are the rates for staying at $x_{\tau_{i-1}}$}
  = &\frac{\partial}{\partial \tau_i} \left( 1 - 
      \prod_{i=1}^{D}
        \exp\left(
          \int_{\tau_{i-1}}^{\tau_i}
            \frac{\dot{\kappa}_s}{1-\kappa_s}\left(p^\theta_{1|s}(X_1^{\sigma(i)}=x_{\tau_{i-1}}^{\sigma(i)} | X_{s}^{-}=x_{\tau_{i-1}})-\delta_{x_{s}^{\sigma(i)}}(x_{\tau_{i-1}}^{\sigma(i)})\right)
          ds
        \right)
    \right)\\
  = &\frac{\partial}{\partial \tau_i} \left( 1 - 
      \prod_{i=1}^{D}
        \exp\left(
          \int_{\tau_{i-1}}^{\tau_i}
            -\frac{\dot{\kappa}_s}{1-\kappa_s}\left(1 - p^\theta_{1|s}(X_1^{\sigma(i)}=x_{\tau_{i-1}}^{\sigma(i)} | X_{s}^{-}=x_{\tau_{i-1}})\right)
          ds
        \right)
    \right)\\
    = &\frac{\partial}{\partial \tau_i} \left( 1 - 
      \prod_{i=1}^D
        \exp\left(
          \int_{\tau_{i-1}}^{\tau_i}
            -\frac{d}{ds}\ln(1-\kappa_s)\left(1 - p^\theta_{1|s}(X_1^{\sigma(i)}=x_{\tau_{i-1}}^{\sigma(i)} | X_{s}^{-}=x_{\tau_{i-1}})\right)
          ds
        \right)
    \right).\\
\intertext{We now note that, $p^\theta_{1|s}$ has no real dependence on time given the current sequence. It is actually a constant, fixed at the value it takes as $s=\tau_{i-1}$.}
  = &\frac{\partial}{\partial \tau_i} \left( 1 - 
      \prod_{i=1}^D
        \exp\left(
            \left(1 - p^\theta_{1|\tau_{i-1}}(X_1^{\sigma(i)}=x_{\tau_{i-1}}^{\sigma(i)} | X_{\tau_{i-1}}^{-}=x_{\tau_{i-1}})\right)
            \int_{\tau_{i-1}}^{\tau_i}
                -\frac{d}{ds}\ln(1-\kappa_s)
          ds
        \right)
    \right).\\
\intertext{The remaining integral also simplifies nicely:}
  = &\frac{\partial}{\partial \tau_i} \left( 1 - 
      \prod_{i=1}^D
        \exp\left(
            \left(1 - p^\theta_{1|\tau_{i-1}}(X_1^{\sigma(i)}=x_{\tau_{i-1}}^{\sigma(i)} | X_{\tau_{i-1}}^{-}=x_{\tau_{i-1}})\right)
            \ln\left(\frac{1-\kappa_{\tau_{i}}}{1-\kappa_{\tau_{i-1}}}\right)
          ds
        \right)
    \right).\\
  = &\frac{\partial}{\partial \tau_i} \left( 1 - 
      \prod_{i=1}^D
        \left(
            \frac{1-\kappa_{\tau_{i}}}{1-\kappa_{\tau_{i-1}}}
        \right)^{p^\theta_{1|\tau_{i-1}}\left(X_1^{\sigma(i)} \neq x_{\tau_{i-1}}^{\sigma(i)} | X_{\tau_{i-1}}^{-}=x_{\tau_{i-1}}\right)}
    \right)\\
  = &
    \sum_{j=1}^D
        {p^\theta_{1|\tau_{i-1}}\left(X_1^{\sigma(j)} \neq x_{\tau_{i-1}}^{\sigma(j)} | X_{\tau_{i-1}}^{-}=x_{\tau_{i-1}}\right)}
    \frac{\dot{\kappa}_{\tau_{i}}}{1-\kappa_{\tau_{i-1}}}
    \left(
        \frac{1-\kappa_{\tau_{i}}}{1-\kappa_{\tau_{i-1}}}
    \right)^{
        \sum_{k=1}^D
        {p^\theta_{1|\tau_{i-1}}\left(X_1^{\sigma(k)} \neq x_{\tau_{i-1}}^{\sigma(k)} | X_{\tau_{i-1}}^{-}=x_{\tau_{i-1}}\right)} - 1
    }\\
\intertext{Intuitively, this is a sum of the probabilities for each position $j$, that we observe all the positions stay constant until $j$ is the next to transition at time $\tau_i$. For the special case where we are using a masking process, we never transition to the masked state so we have $p^\theta_{1|\tau_{i-1}}(X_1^{\sigma(i)} \neq \cdot|X_{\tau_{i-1}}^{-}=M)=\delta_{M}(\cdot)$ and \hbox{$p^\theta_{1|\tau_{i-1}}(X_1^{\sigma(i)} \neq M |X_{\tau_{i-1}}^{-}=\cdot)=1$}. These observations let us only evaluate the sum over masked positions and set the exponent to the number number of masked positions minus one:}
  = &(D-(i-1)) \frac{\dot{\kappa}_{\tau_{i}}}{1-\kappa_{\tau_{i-1}}} \left(\frac{1-\kappa_{\tau_{i}}}{1-\kappa_{\tau_{i-1}}} \right)^{D-i}  \\
  = &(D-(i-1)) P(\tau_i | \tau_i > \tau_{i-1})  \prod_{j \in \sigma(> i)} P(\tau_j > \tau_i | \tau_i > \tau_{i-1})
\end{align*}
as claimed.
\end{proof}

\begin{lemma}\label{lemma:transitions-are-same}
The distribution over state transitions is independent of the jump time, and can be expressed in closed-form as follows:
\begin{align}
\Pm(x_{\tau_i} | \tau_i, X_{\tau_{i}}^{-}=x_{\tau_{i-1}}) 
&= \Pm(x_{\tau_i} | \tau_i, X_{\tau_{i}}^{-}=x_{\tau_{i-1}}) \notag \\
&= \frac{p_{1|\tau_i}^\theta(X_1^{\sigma(i)}=x_{\tau_i}^{\sigma(i)} | X_{\tau_i}^{-}=x_{\tau_{i-1}})}{D - (i-1)} \label{eq:state-transitions}
\end{align}
\end{lemma}
Note that the $t$ in $p^\theta_{1|t}$ is suggestive but is often not an explicit input to the model. See \citet{ou2025your,shi2024simplified,sahoo2024simple} for further discussion of the analogous results and implications of the lemma for MDMs. In particular, \citet{sahoo2024simple} provide some empirical results showing that inputting time to the model is unnecessary for MDMs.
\begin{proof}
There is a vanishingly small chance that two positions jump at the same time \cite{campbell2024generative}. Accordingly, we first examine only the conditional distribution for the position that changes. We can write it as
\begin{align*}
&\Pm\left(X_{\tau_i}^{\sigma(i)}=x_{\tau_i}^{\sigma(i)} | \tau_i, X_{\tau_i}^- = x_{\tau_{i-1}} \right) \\
= &\Pm\left(X_{\tau_i}^{\sigma(i)}=x_{\tau_i}^{\sigma(i)} | X_{\tau_i}^{\sigma(i)} \neq M, {X_{\tau_i}^{\sigma(i)}}^-= M, X_{\tau_i}^- = x_{\tau_{i-1}} \right) \\
= &\frac
{\Pm\left(X_{\tau_i}^{\sigma(i)}=x_{\tau_i}^{\sigma(i)} | X_{\tau_i}^- = x_{\tau_{i-1}} \right)}
{\Pm\left(X_{\tau_i}^{\sigma(i)} \neq M | X_{\tau_i}^- = x_{\tau_{i-1}} \right)} \\
= &\frac
{R_t^{\sigma(i)}\left(M, x_{\tau_i}^{\sigma(i)} | X_{\tau_i}^{-}=x_{\tau_{i-1}}\right) dt}
{\sum_{\tilde{x} \neq M} R_t^{\sigma(i)}\left(M, \tilde{x} | X_{\tau_i}^{-}=x_{\tau_{i-1}}\right) dt} \\
= &\frac
{R_t^{\sigma(i)}\left(M, x_{\tau_i}^{\sigma(i)} | X_{\tau_i}^{-}=x_{\tau_{i-1}}\right)}
{-R_t^{\sigma(i)}\left(M, M | X_{\tau_i}^{-}=x_{\tau_{i-1}}\right)} \\
= &\frac
{\frac{\dot{\kappa}_{\tau_i}}{1-\kappa_{\tau_i}} \left( p_{1|\tau_i}^\theta(X_1^{\sigma(i)}=x_{\tau_i}^{\sigma(i)} | X_{\tau_i}^{-}=x_{\tau_{i-1}}) - \delta_{x_{\tau_i}^{\sigma(i)}, M}\right)}
{-\frac{\dot{\kappa}_{\tau_i}}{1-\kappa_{\tau_i}} \left(p_{1|\tau_i}^\theta(X_1^{\sigma(i)}=M | X_{\tau_i}^{-}=x_{\tau_{i-1}}) - \delta_{M, M} \right)}. \\
\intertext{Notice that here we lose all time dependency,}
= &\frac
{p_{1|\tau_i}^\theta(X_1^{\sigma(i)}=x_{\tau_i}^{\sigma(i)} | X_{\tau_i}^{-}=x_{\tau_{i-1}})}
{\delta_{M, M}} \\
= &p_{1|\tau_i}^\theta(X_1^{\sigma(i)}=x_{\tau_i}^{\sigma(i)} | X_{\tau_i}^{-}=x_{\tau_{i-1}}).
\end{align*}

However, we still need to account for the fact that, out of all unmasked positions, $\sigma(i)$ was the one that changed.
\begin{align*}
&\Pm(x_{\tau_i} | \tau_i, X_{\tau_{i}}^{-}=x_{\tau_{i-1}}) \\
&= p_{1|\tau_i}^\theta(X_1^{\sigma(i)}=x_{\tau_i}^{\sigma(i)} | X_{\tau_i}^{-}=x_{\tau_{i-1}}) \frac{p_{1|\tau_i}^\theta(X_1^{\sigma(i)} \neq M | X_{\tau_i}^{-}=x_{\tau_{i-1}})}{\sum_{j \in \sigma(\geq i)} p_{1|\tau_i}^\theta(X_1^{\sigma(i)} \neq M | X_{\tau_i}^{-}=x_{\tau_{i-1}})} \\
&= \frac{p_{1|\tau_i}^\theta(X_1^{\sigma(i)}=x_{\tau_i}^{\sigma(i)} | X_{\tau_i}^{-}=x_{\tau_{i-1}})}{D-(i-1)}
\end{align*}
as claimed.
\end{proof}

Although we present a proof for the MFM case, the same calculation can be used to see that the general result for the common class of MFMs described in \cite{campbell2024generative, gat2024discrete} is:
\begin{align*}
&\Pm(x_{\tau_i}^{\sigma(i)} | \tau_i, X_{\tau_{i}}^{-}=x_{\tau_{i-1}}) \\
= &p^\theta_{1|\tau_i}(X_1^{\sigma(i)}=x_{\tau_i}^{\sigma(i)} | X_{\tau_i}^{-}=x_{\tau_{i-1}})
\frac{1 - p^\theta_{1|\tau_i}(X_1^{\sigma(i)}=x_{\tau_{i-1}}^{\sigma(i)} | X_{\tau_i}^{-}=x_{\tau_{i-1}})}{D - \sum_{j=1}^{D} p^\theta_{1|\tau_i}(X_1^{\sigma(j)}=x_{\tau_{i-1}}^{\sigma(j)} | X_{\tau_i}^{-}=x_{\tau_{i-1}})}\\
= &p^\theta_{1|\tau_i}(X_1^{\sigma(i)}=x_{\tau_i}^{\sigma(i)} | X_{\tau_i}^{-}=x_{\tau_{i-1}})
\frac{p^\theta_{1|\tau_i}(X_1^{\sigma(i)}\neq x_{\tau_{i-1}}^{\sigma(i)} | X_{\tau_i}^{-}=x_{\tau_{i-1}})}{\sum_{j=1}^{D} p^\theta_{1|\tau_i}(X_1^{\sigma(j)}\neq x_{\tau_{i-1}}^{\sigma(j)} | X_{\tau_i}^{-}=x_{\tau_{i-1}})}\\
\end{align*}

%% file: supplementary_figures.tex


\section{Supplementary Tables}


\newcolumntype{L}[1]{>{\raggedright\arraybackslash}p{#1}}

\begin{figure}[h!]
  \centering
  \captionof{table}{\textbf{Summary of design tasks in this work (\textit{in silico} and \textit{in vivo}).} The arrows in the ``Property'' column indicate the desired direction we wish to guide the generation towards. For folding stability, we desire enhanced stability, which is represented by lower $\Delta \Delta G$. The ``Multi-family'' column indicates whether multiple protein families are included in the dataset considered. Parentheses in the column ``\# variants (predictor)'' indicate the number of labeled variants used to train the guidance predictor when it differs from the full dataset size.``Mutational distance in dataset'' represents the range of hamming distances of the variants in a dataset relative to the wild-type sequence, when applicable. The ``Pre-trained model(s)'' column indicate the pre-trained model(s) used with \ourmethod.} 
  \label{tab:design_tasks}
  \footnotesize
  \setlength{\tabcolsep}{6pt} 
  \renewcommand{\arraystretch}{1.2}

  \begin{tabular}{@{}L{2.5cm} L{2.5cm} L{1.5cm} c L{1.8cm} L{1.8cm} L{1.5cm}@{}}
    \toprule
    \textbf{Task} &
    \textbf{Property} &
    \textbf{Design Setting} &
    \textbf{Multi-family?} &
    \textbf{\# variants (predictor)} &
    \textbf{Mutational distance in dataset} &
    \textbf{Pre-trained model(s)} \\
    \midrule

    Stability-guided inverse folding &
    Folding stability $\uparrow$~\citep{tsuboyama2023mega} &
    Interpolative &
    Yes &
    $\sim$669k total, $\sim$1-7k per protein &
    $\le$2 (per protein) &
    Protein MPNN~\citep{dauparas2022robust} \\

    \addlinespace
    CreiLOV assay-informed design &
    Fluorescence intensity $\uparrow$~\citep{chen2023deep} &
    Extrapolative &
    No &
    $\sim$167k (2k) &
    $\le$15 &
    ESM3~\citep{hayes2025simulating} \\

    UBE4B assay-informed design &
    Ubiquitination activity $\uparrow$~\citep{starita2013activity} &
    Extrapolative &
    No &
    $\sim$32k (2k) &
    $\le$6 &
    ESM3 \\

    PABP/Pab1 assay-informed design &
    Functional complementation $\uparrow$~\citep{melamed2013deep} &
    Extrapolative &
    No &
    $\sim$38k (2k) &
    $\le$2 &
    ESM3 \\

    \addlinespace
    PbrR Pareto extrapolative design &
    Pb binding $\uparrow$; Zn binding $\downarrow$~\citep{wang2025active} &
    Multi-property Pareto extrapolative &
    No &
    1060 &
    $\le$5 &
    ESM C~\citep{esm_cambrian_2024} \\

    \addlinespace
    Enzyme-class guidance &
    EC label (class)~\citep{uniprot2021uniprot} &
    Interpolative &
    Yes &
    $\sim$316k sequences, $\sim 10^2$ to $\sim 10^3$ sequences for common ECs, $\sim 10^5$ non-enzymes &
    N/A &
    ESM3 \\

    Fold-class guidance &
    CATH label (class)~\citep{dawson2017cath} &
    Interpolative &
    Yes &
    $\sim$200k structures total, $\sim10^2$ to $\sim10^4$ per class &
    N/A &
    ESM3 (backbone structure) \\

    \addlinespace
    \textit{In vivo}: ABE design &
    Base editing activity $\uparrow$ &
    Extrapolative &
    No &
    $\sim$2k &
    $\le$40&
    FMIF~\citep{nisonoff2025unlocking}, ESM3 \\

    \bottomrule
  \end{tabular}
\end{figure}

\clearpage


\begin{table}[h!]
\centering
\caption{\textbf{Multi-property, Pareto-extrapolative experiment: generation performance after post-hoc filtering with predictor.}}
\small
\begin{tabular}{lccc}
\toprule
\textbf{Method}
& \textbf{N Sequences}
& \textbf{Success Rate (Top-100)}
& \textbf{Diversity} \\
\midrule
Unguided
& 1000
& 16
& 33 \\
FT
& 1000
& 6
& 3 \\
Guidance
& 100
& 75
& 36 \\
\bottomrule
\end{tabular}
\label{tab:pbrr_filtered_results}
\end{table}

\begin{table}[h!]
\centering
\caption{\textbf{Multi-property, Pareto-extrapolative experiment: generation performance under equal wall-clock time without filtering.}}
\small
\begin{tabular}{lcccccc}  
\toprule
& & & \multicolumn{2}{c}{\textbf{Pb log-FC}} & \multicolumn{2}{c}{\textbf{Zn log-FC}} \\
\cmidrule(lr){4-5} \cmidrule(lr){6-7}
\textbf{Method}
& \textbf{Success Total}
& \textbf{Success / Sec}
& \textbf{Mean}
& \textbf{95th pct.}
& \textbf{Mean}
& \textbf{95th pct.} \\
\midrule
Unguided
& 19
& 0.1
& 1.9
& 4.9
& 4.9
& -2.9 \\
FT
& 3
& 0.0
& 0.8
& 2.4
& -0.4
& -2.2 \\
Guidance
& 75
& 0.4
& 2.8
& 4.5
& -0.7
& -2.8 \\
\bottomrule
\end{tabular}
\label{tab:pbrr_wallclock_matched}
\end{table}

\begin{figure}[h!]
    \centering
    \captionof{table}{\textbf{EC classes with enzyme names and associated keyword guidance.}}
    \resizebox{\textwidth}{!}{
    \begin{tabular}{@{}l l l@{}}
      \toprule
      \textbf{EC Number} & \textbf{Enzyme Name} & \textbf{Keywords} \\
      \midrule
      EC:2.3.2.27 & RING-type E3 ubiquitin transferase &
      ring type, e3 ubiquitin, ubiquitin transferase \\
      
      EC:2.7.11.1 & Non-specific serine threonine protein kinase &
      non specific, serine/threonine kinase \\
      
      EC:2.7.7.6 & DNA-directed RNA polymerase &
      dna directed, rna polymerase \\
      
      EC:3.1.26.4 & Ribonuclease H &
      ribonuclease H \\
      
      EC:3.6.4.12 & DNA helicase &
      dna helicase \\
      
      EC:3.6.4.13 & RNA helicase &
      rna helicase \\
      
      EC:4.2.1.33 & 3-isopropylmalate dehydratase &
      3-isopropylmalate, isopropylmalate dehydratase \\
      
      EC:5.2.1.8 & Peptidylprolyl isomerase &
      peptidyl prolyl, isomerase \\
      
      EC:7.1.1.2 & NADH:ubiquinone reductase (H$^+$-translocating) &
      NADH ubiquinone, ubiquinone reductase, proton, translocating \\
      
      EC:7.1.2.2 & H$^+$-transporting two-sector ATPase &
      proton, transporting, two sector, ATPase \\
      \bottomrule
    \end{tabular}
    }
    \label{tab:enzyme_keywords}
\end{figure}

\begin{figure}[h!]
    \centering
    \captionof{table}{\textbf{Enzyme-class guidance results for different EC sets at 80\% masking level.} Number of samples (N) per set are indicated. ``Mask Level \%" is the percent of positions masked, 100 \% is full masked. ``CLEAN \%'' is the percentage of samples that were correctly classified as the intended target class. ``pLDDT $> 0.8$'' is the proportion of samples correctly reclassified by CLEAN that have a median pLDDT across residues above 0.8.}
    \resizebox{\textwidth}{!}{
    \begin{tabular}{@{}l l r r r r r@{}}
      \toprule
      \textbf{Method} & \textbf{EC Set} & \textbf{CLEAN \% ($\uparrow$)} & \textbf{Mean pLDDT ($\uparrow$)} & \textbf{pLDDT $> 0.8$ ($\uparrow$)} & \textbf{Success \% ($\uparrow$)} & \textbf{N} \\
      \midrule
      ESM3 & Top-10 & 64\% & 0.84 & 86\% & 55\% & 100 \\
      & Random & 61\% & 0.83 & 95\% & 58\% & 100 \\
      & Min & 38\% & 0.82 & 94\% & 36\% & 86 \\
      \addlinespace
      +ProteinGuide & Top-10 & 76\% & 0.85 & 81\% & 62\% & 100 \\
      & Random & 73\% & 0.85 & 99\% & 72\% & 100 \\
      & Min & 41\% & 0.80 & 80\% & 33\% & 86 \\
      \bottomrule
    \end{tabular}
    }
    \label{tab:enzyme_guidance_results_ec_sets_combined}
\end{figure}

\newpage

\begin{figure}
    \centering
    \captionof{table}{
    \textbf{CATH labels used for A- and T-level fold-class conditioning baselines.}
    Each row includes the CATH code, its description, whether it belongs to the ``common'' or ``random'' group, the total number of samples in the dataset, and the fold-oracle test accuracy for that label.
    }
    \resizebox{0.8\textwidth}{!}{
    \begin{tabular}{llcrc}
    \toprule
    \textbf{Code} & \textbf{Description} & \textbf{Group} & \textbf{\# samples} & \textbf{Test Acc} \\
    \midrule
    \multicolumn{5}{l}{\textbf{A-level}}\\
    \midrule
    3.40  & 3-Layer(aba) Sandwich        & common & 84172 & 0.8163 \\
    2.60  & Sandwich                     & common & 34728 & 0.8724 \\
    2.40  & Beta Barrel                  & common & 24208 & 0.7961 \\
    3.20  & Alpha-Beta Barrel            & common & 15623 & 0.7755 \\
    3.10  & Roll                         & common & 13951 & 0.7230 \\
    2.30  & Roll                         & common &  6424 & 0.7406 \\
    1.25  & Alpha Horseshoe              & common &  2967 & 0.7316 \\
    2.80  & Trefoil                      & common &  1365 & 0.8224 \\
    2.130 & 7 Propeller                  & common &  1239 & 0.9423 \\
    2.120 & 6 Propeller                  & common &  1011 & 0.7374 \\
    2.100 & Aligned Prism                & random &   305 & 0.8485 \\
    3.80  & Alpha-Beta Horseshoe         & random &   756 & 0.8226 \\
    2.102 & 3-layer Sandwich             & random &   234 & 0.7812 \\
    2.50  & Clam                         & random &    66 & 0.9091 \\
    3.70  & Box                          & random &   297 & 0.9412 \\
    2.115 & 5 Propeller                  & random &   434 & 0.7778 \\
    2.90  & Orthogonal Prism             & random &   141 & 1.0000 \\
    2.160 & 3 Solenoid                   & random &   964 & 0.7396 \\
    3.65  & Alpha-beta prism             & random &   400 & 1.0000 \\
    2.140 & 8 Propeller                  & random &   398 & 0.8889 \\
    \midrule
    \multicolumn{5}{l}{\textbf{T-level}}\\
    \midrule
    3.40.50  & Rossmann fold                      & common & 50985 & 0.7885 \\
    2.60.40  & Immunoglobulin-like                & common & 23273 & 0.8811 \\
    3.20.20  & TIM Barrel                         & common & 14701 & 0.7779 \\
    3.30.70  & Alpha-Beta Plaits                  & common &  9911 & 0.7960 \\
    2.60.120 & Jelly Rolls                        & common &  8713 & 0.8591 \\
    2.40.10  & Thrombin, subunit H                & common &  7012 & 0.9036 \\
    3.40.190 & D-Maltodextrin-Binding Protein     & common &  5938 & 0.7025 \\
    2.40.70  & Cathepsin D, subunit A             & common &  4683 & 1.0000 \\
    3.40.30  & Glutaredoxin                       & common &  3702 & 0.9466 \\
    1.10.490 & Globin-like                        & common &  2682 & 0.9792 \\
    2.30.110 & Pnp Oxidase                        & random &   367 & 0.9487 \\
    1.25.40  & Serine Threonine Protein Phosp...  & random &  2155 & 0.7453 \\
    3.30.429 & Macrophage Migration Inhibitor...  & random &   489 & 1.0000 \\
    3.10.310 & Diaminopimelate Epimerase          & random &   285 & 0.8000 \\
    2.160.10 & UDP N-Acetylglucosamine Acyltr...  & random &   544 & 0.7547 \\
    3.10.129 & Thiol Ester Dehydrase              & random &   973 & 0.9024 \\
    1.20.91  & Influenza Virus Matrix Protein     & random &    76 & 0.7273 \\
    3.40.605 & Aldehyde Dehydrogenase             & random &  1113 & 1.0000 \\
    2.40.155 & Green Fluorescent Protein          & random &  1100 & 1.0000 \\
    3.40.225 & L-fuculose-1-phosphate Aldolase    & random &    90 & 1.0000 \\
    \bottomrule
    \end{tabular}
    }
    \label{tab:cath_labels_a_t_common_random}
\end{figure}

\newpage

\begin{figure}[t]
  \centering
  \captionof{table}{\textbf{Illumina PCR1 primers used in the base editor experiment.}}
  \resizebox{0.8\textwidth}{!}{
  \begin{tabular}{lp{8cm}}
    \toprule
    \textbf{Type} & \textbf{Sequence} \\
    \midrule
    Forward primers & \texttt{GCTCTTCCGATCTNCGACGCCACGTTGTAT} \\
    & \texttt{GCTCTTCCGATCTNNCGACGCCACGTTGTAT} \\
    & \texttt{GCTCTTCCGATCTNNNCGACGCCACGTTGTAT} \\
    \midrule
    Reverse primers & \texttt{GCTCTTCCGATCTNCGCTAGATCCTCCGGA}  \\
    & \texttt{GCTCTTCCGATCTNNCGCTAGATCCTCCGGA} \\
    & \texttt{GCTCTTCCGATCTNNNCGCTAGATCCTCCGGA} \\
    \bottomrule
  \end{tabular}
  }
  \label{tab:illumina_primers}
\end{figure}

\clearpage


\section{Supplementary Figures}

\begin{figure*}[h!]
  \centering
  \includegraphics[width=1.0\textwidth]{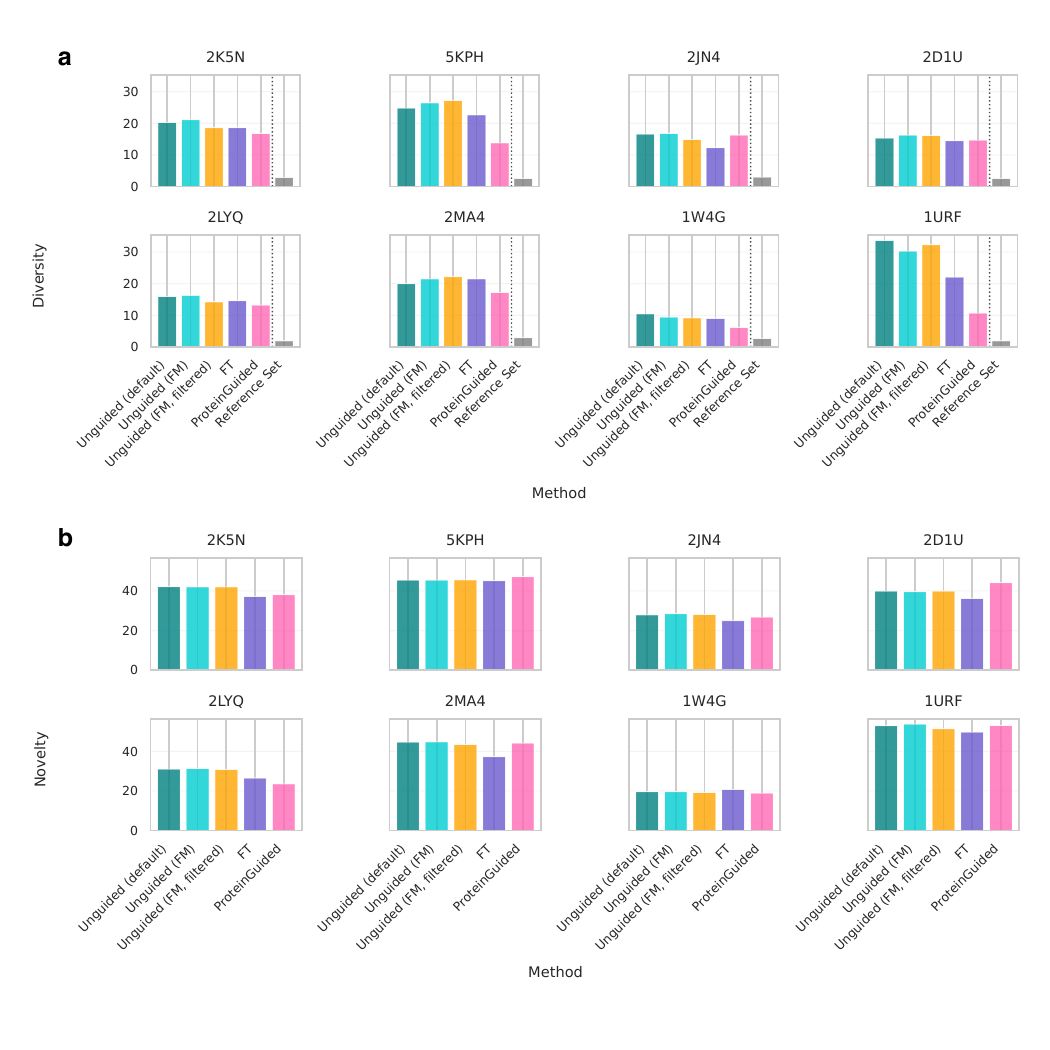}
  \caption[A]{
    \textbf{Diversity and novelty of generated sequences for all methods in the ProteinMPNN stability experiment.}
    \textbf{a}) Diversity of the generated sequences. For each method, diversity is evaluated on 100 sequences. Diversity is computed as the averaged pairwise hamming distance among the generated sequences. For each protein, the ``Reference Set'' diversity is computed on 500 randomly sampled sequences from the labeled dataset for each protein.
    \textbf{b}) Novelty of the generated sequences. For each method, novelty is evaluated on 100 sequences. For each sequence, novelty is computed as its minimum hamming distance with sequences in the labeled dataset. For each method, the mean novelty among the generated sequences is shown.
  }
  \label{si_fig:stability_diversity_novelty}
\end{figure*}

\newpage


\begin{figure*}[h!]
  \centering
  \includegraphics[width=0.9\textwidth]{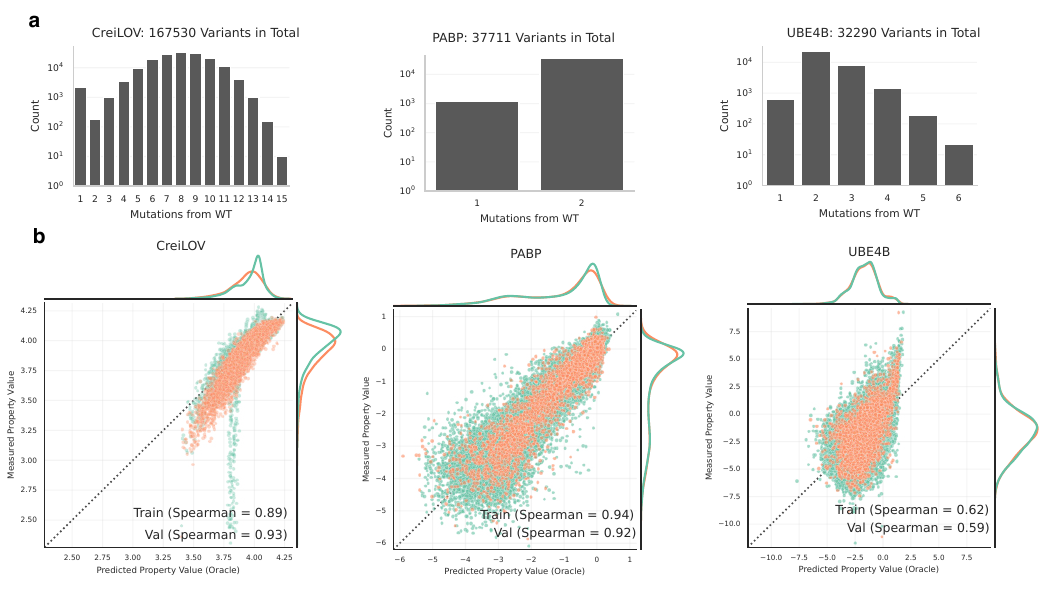}
  \caption[A]{
    \textbf{Datasets and oracle models used for the extrapolative experiment of guiding ESM3 with assay-informed predictive models.}
    \textbf{a}) Distribution of variant counts in each dataset stratified by the hamming distance from the wild-type sequence.
    \textbf{b}) Scatterplot of the labeled vs. oracle predicted property value in each dataset, for both the training and validation set.

  }
  \label{si_fig:fitness_dataset}
\end{figure*}

\begin{figure*}[h!]
  \centering
  \includegraphics[width=0.9\textwidth]{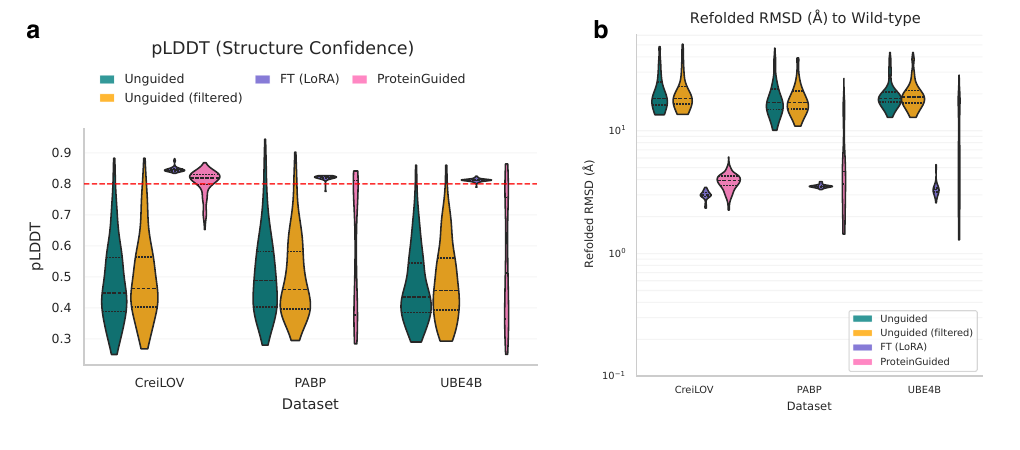}
  \caption[A]{
    \textbf{Structure prediction-based metrics for generated sequences from guiding ESM3 with assay-informed predictive models.}
    \textbf{a}) Distribution of the pLDDT of the generated sequences from each method (higher is better). pLDDT is based on ESMFold prediction. Red dotted line indicates the threshold above with a generated sequence is considered a success (pLDDT $\geq80$).   
    \textbf{b}) Distribution of the refolded RMSD (\AA) of the generated sequences from each method (lower is better). RMSD is computed between the predicted structure of the wild-type sequence and that of the generated sequence.
  }
  \label{si_fig:fitness_structure}
\end{figure*}

\begin{figure*}[h!]
\centering
\includegraphics[width=1.0\textwidth]{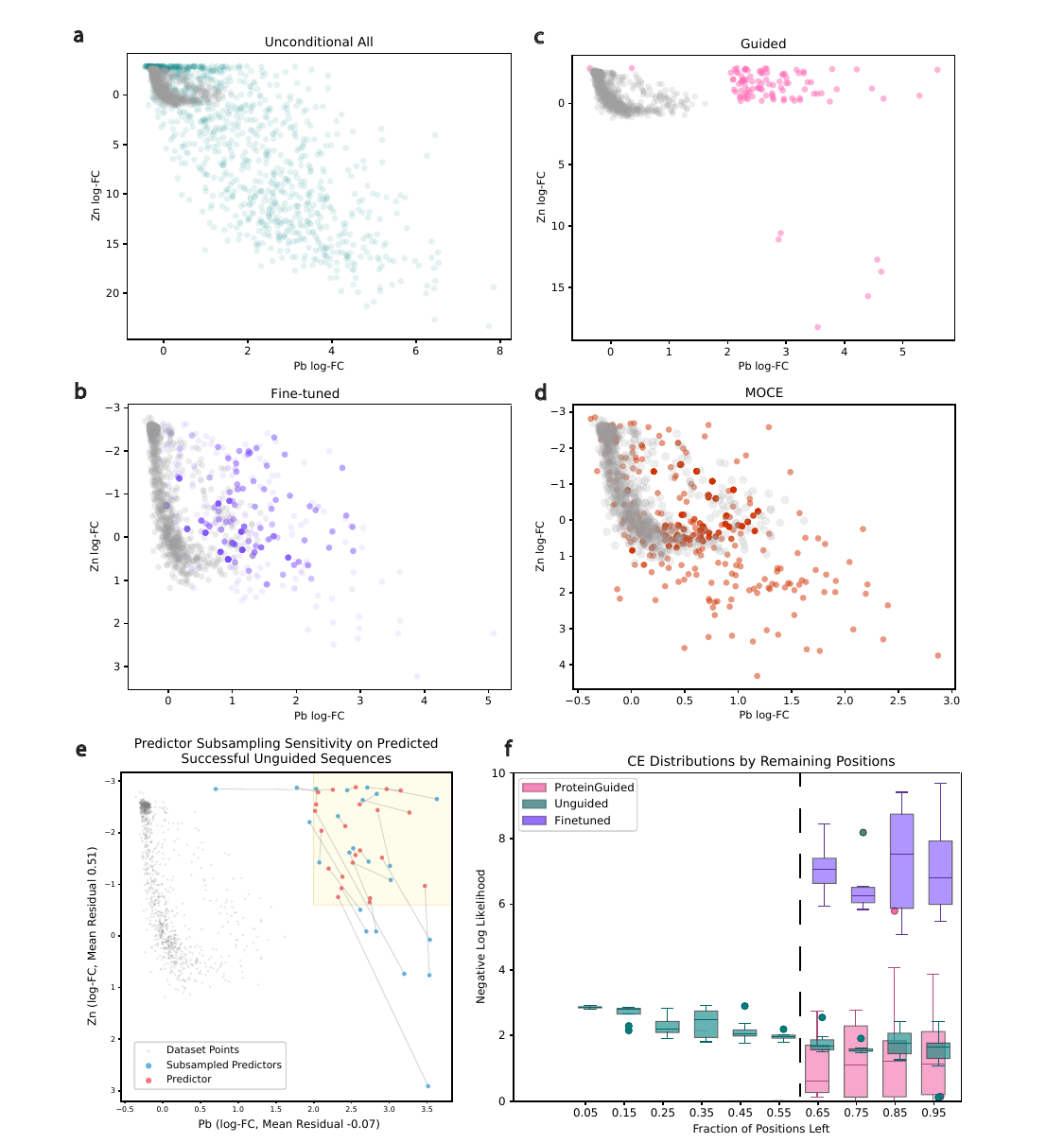}
\caption{
\textbf{Multi-property, Pareto-extrapolative experiment: unfiltered generation results, heuristic constraint adjustment, and negative log-likelihood curves.}
    \textbf{a-d}) Show the unfiltered compute-matched sequences generated by the three methods benchmarked in the main text and MOCE, the method used by~\citet{wang2025active}.
    \textbf{e}) Example of training-data-subsampled predictors (blue) against the predictor used for guidance (red). The residuals are ``Predictor'' minus ``Subsampled''. Residuals were calculated on randomly sampled sequences of up to 30 mutations that were predicted to be in the target region for the Pareto-extrapolative task (\fref{fig:fitness_multi}b, yellow box). 
    \textbf{f}) Time-series of negative log likelihoods for each method on the set of sequences predicted by the oracle to be in the target region of the Pareto-extrapolative task (\fref{fig:fitness_multi}b, yellow box). Dashed line shows the time point from which generation started. A cross entropy of 3 on the y-axis corresponds to the negative log-likelihood on the data for a uniform distribution over amino acids.
    }
\label{si_fig:multi_main}
\end{figure*}

\begin{figure*}[h!]
\centering
\includegraphics[width=1.0\textwidth]{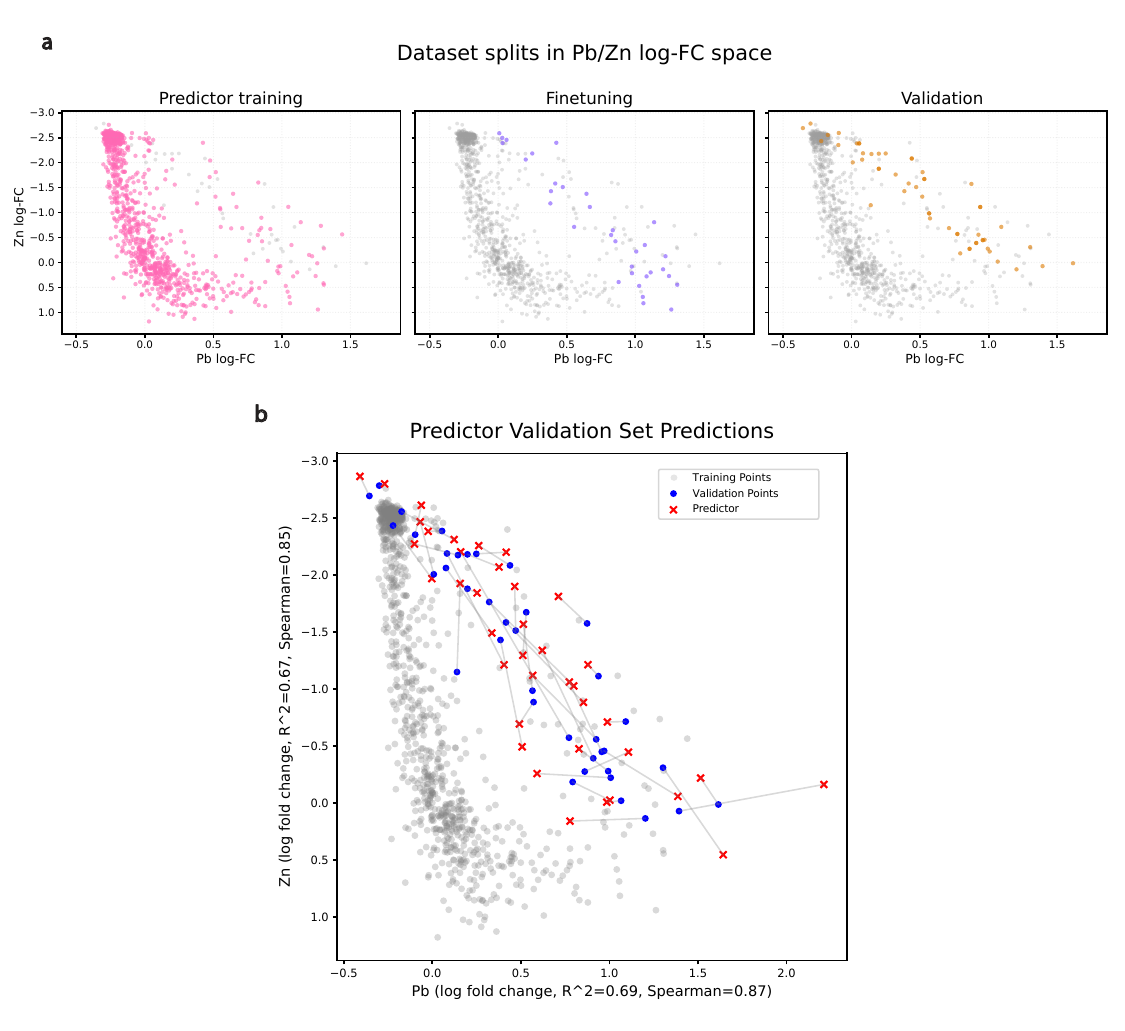}
\caption{
\textbf{Data splits used in the multi-property, Pareto-extrapolative PbrR experiment.} 
    \textbf{a.} Summary of training and validation splits used for the PbrR experiment. Gray points show full dataset points not included in that split. The validation split was common between the predictor and finetuned model. The finetuning training data was the portion of the predictor training data that was above the dashed on-/off-target threshold line shown in the center subpanel.
    \textbf{b.} Predictions from the regression model used for \ourmethod compared with the true assay value on the validation set.
}
\label{si_fig:multi_data}
\end{figure*}
\newpage

\begin{figure*}[h!]
\centering
\includegraphics[width=1.0\textwidth]{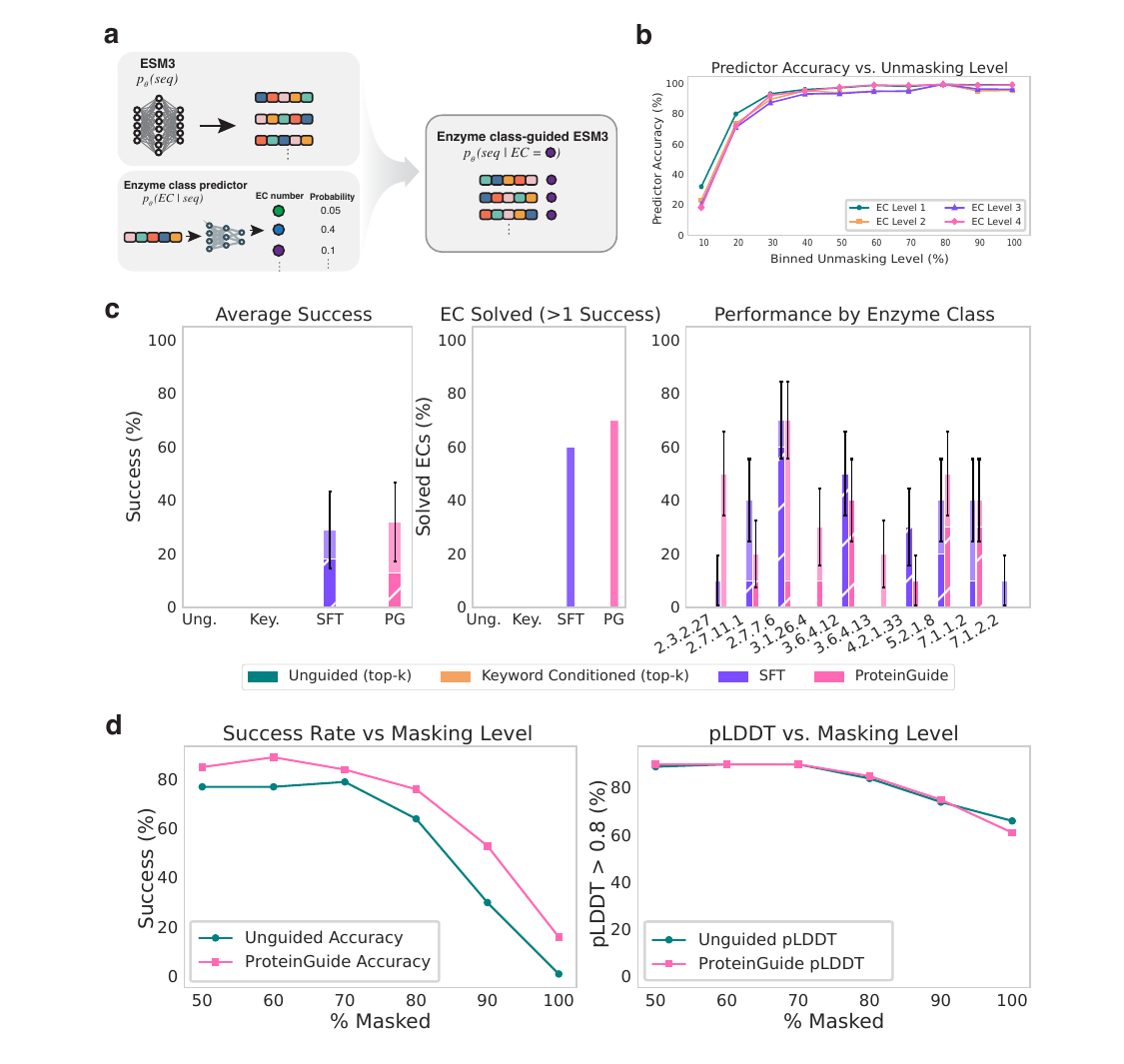}
\caption{
\textbf{Guiding ESM3 with enzyme annotations.} 
    \textbf{a}) Schematic for using \ourmethod to guide ESM3 with an enzyme class predictor.
    \textbf{b}) Classification accuracy of the predictor on SwissProt test sequences as an increasing number of positions are unmasked. All accuracies are class balanced.
    \textbf{c}) Four generative methods (in legend) are assessed across 10 of the most common enzyme classes in SwissProt (horizontal axis of rightmost sub-figure). Left panel shows average success rate across all 10 common classes. Center shows the number of ECs for which each method produced at least one success. Right panel shows success rate by enzyme class for the 10 most common ECs. Top-k experiments generated 210 samples total. Out of these, the top 10 sequences rated by the predictor as most likely to be in the target EC are kept for evaluation. The height of each bar represents the percentage of sequences predicted by the oracle to be of the target EC. The hatched sub-portion are the fraction of sequences in the target EC that are also predicted to fold with pLDDT $\leq$ 0.8. Error bars are standard errors for proportions.
    \textbf{d}) Guidance performance when redesigning test set enzyme sequences. The x-axis of each subplot is the percent of positions in the original enzyme that were masked. The left subplot y-axis shows the percent of sequences generated (10 total) that were reclassified to be in the target enzyme class. The right subplot y-axis shows the percent of these 10 sequences that had a pLDDT greater than 0.8.
}
\label{si_fig:ec}
\end{figure*}
\newpage

\begin{figure*}[h!]
  \centering
  \includegraphics[width=1.0\textwidth]{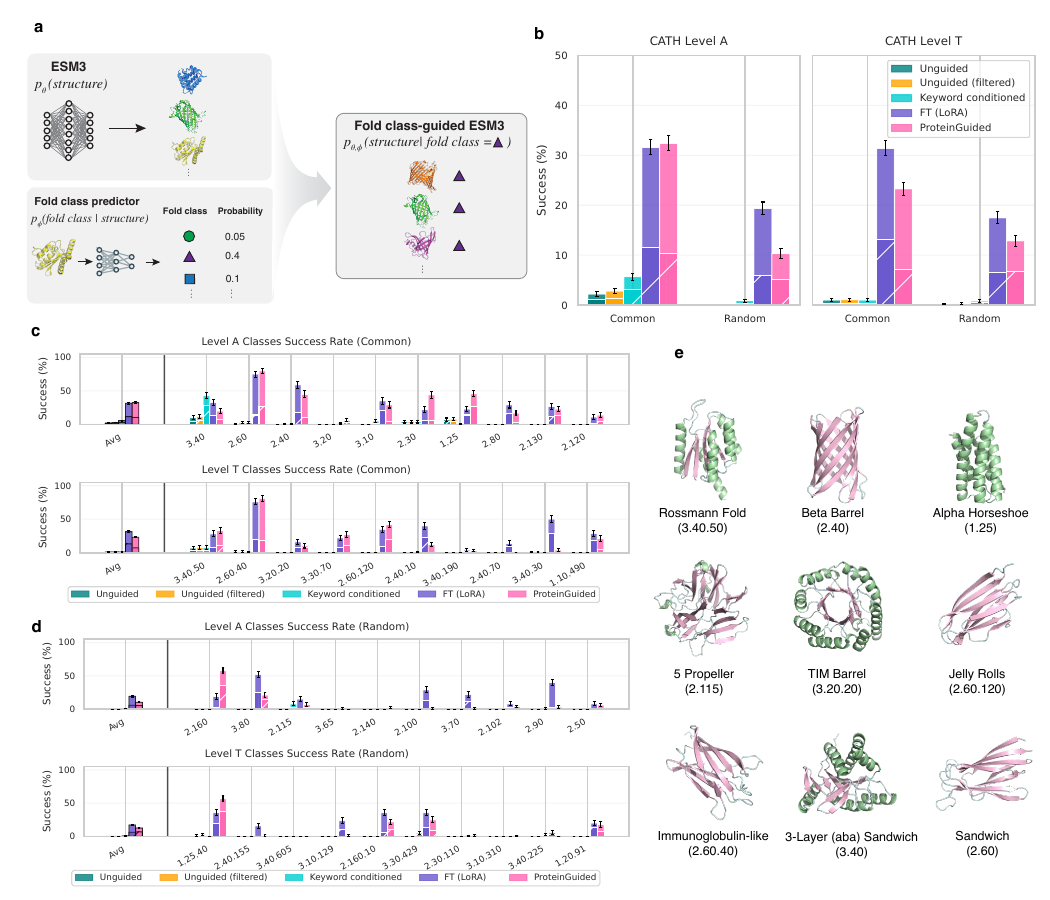}
  \caption[A]{
    \textbf{Guiding ESM3 to generate backbone structures with CATH fold class annotations.}
    \textbf{a}) Schematic for using \ourmethod to guide ESM3 with a fold class predictor (trained on CATH) to generate structures predicted to belong to specific fold class.
    \textbf{b}) Aggregated success rates of different mehods (in legends) across different CATH hierarchy levels. The height of each bar shows the percentage of structures correctly classified to their target fold class, while the hatched sub-portion indicates the subset that are also designable (self-consistency RMSD $\leq 2$\r{A}). Results are grouped by CATH level, and the classes for each level were grouped by their frequency of occurrence in the training dataset of the fold class oracle. Metrics were aggregated within each group following \citet{geffner2025proteina}. ``Common'' consists of the 10 most common classes for each level, and ``Random'' consists of 10 randomly selected classes for each level. 
    \textbf{c}) Per-class success rates of different methods for the Common classes for both A level (top) and T level (bottom). 
    \textbf{d}) Per-class success rates of different methods for the Random classes for both A level (top) and T level (bottom). 
    \textbf{e}) Examples of generated backbone structures by guiding ESM3 with some example CATH classes. The shown examples are both classified correctly to the target fold class by the fold classifier and are designable.
  }
  \label{si_fig:fold_class}
\end{figure*}

\newpage


\begin{figure*}[h!]
  \centering
  \includegraphics[width=1.0\textwidth]{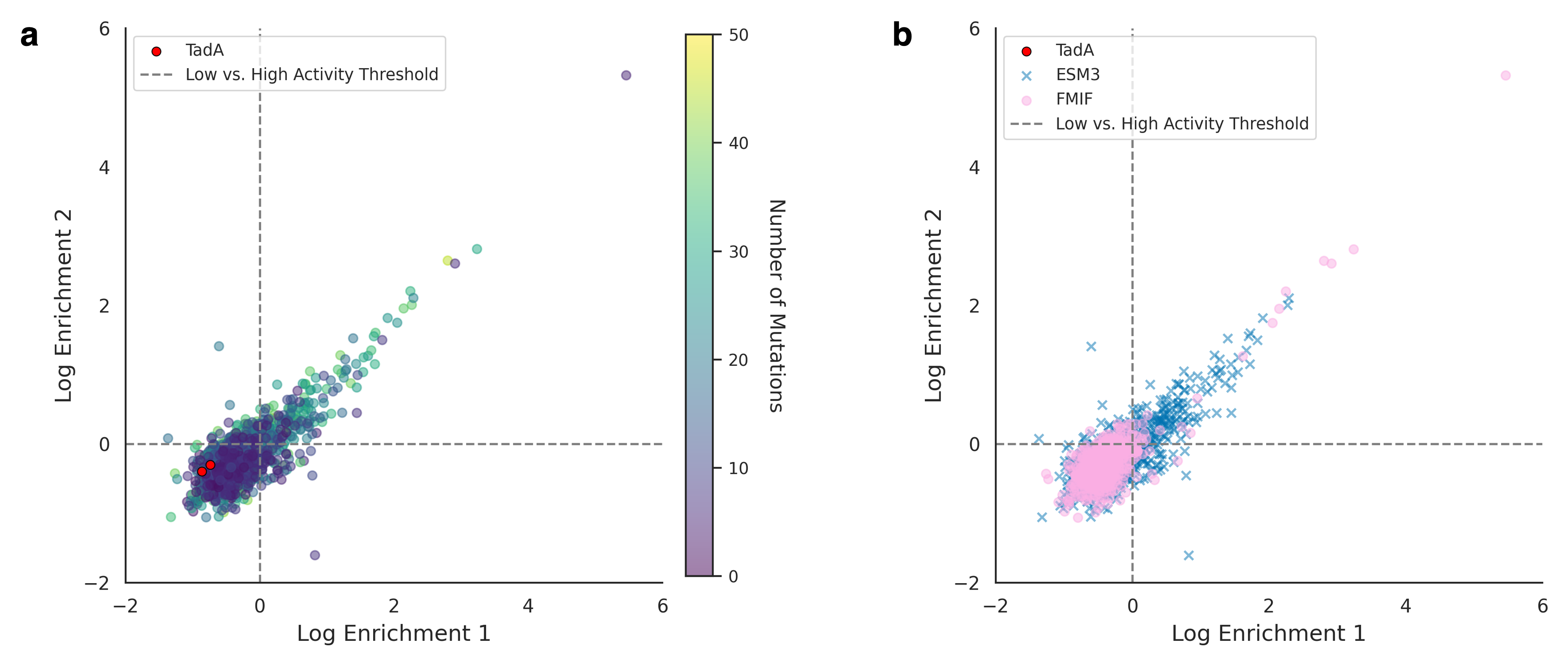}
  \caption[A]{
    \textbf{Experimental characterization of sequences sampled without \ourmethodacro.}
    \textbf{a})
    Log enrichment scores of the $2,000$ sequences sampled from ESM3 and FMIF for two biological replicates. Each sequence is colored by the number of mutations away from TadA (red). The class boundary used to train a classifier is shown as a dotted grey line.
    \textbf{b})
    The same log enrichment scores as in (\textbf{a.}) but different colors and markers are used to denote which model the sequences were generated from.
  }
  \label{si_fig:round1scatter}
\end{figure*}

\begin{figure*}[h!]
  \centering
  \includegraphics[width=1.0\textwidth]{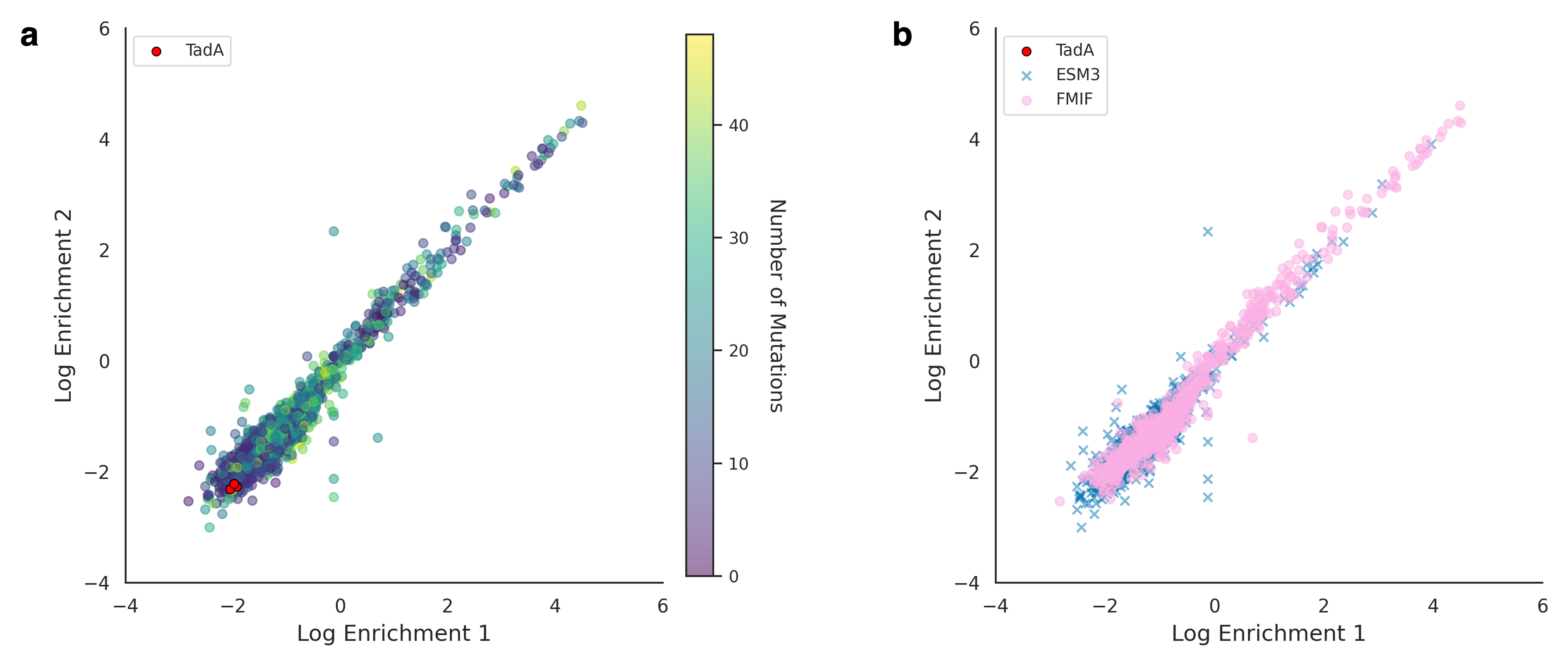}
  \caption[A]{
    \textbf{Experimental characterization of sequences sampled with \ourmethodacro.}
    \textbf{a})
    Log enrichment scores of the $1,888$ sequences sampled from ESM3 and FMIF with \ourmethod~for two biological replicates. Each sequence is colored by the number of mutations away from TadA (red).
    \textbf{b})
    The same log enrichment scores as in \textbf{a}) but different colors and markers are used to denote which model the sequences were generated from.
  }
  \label{si_fig:round2scatter}
\end{figure*}

\begin{figure*}[h!]
  \centering
  \includegraphics[width=1.0\textwidth]{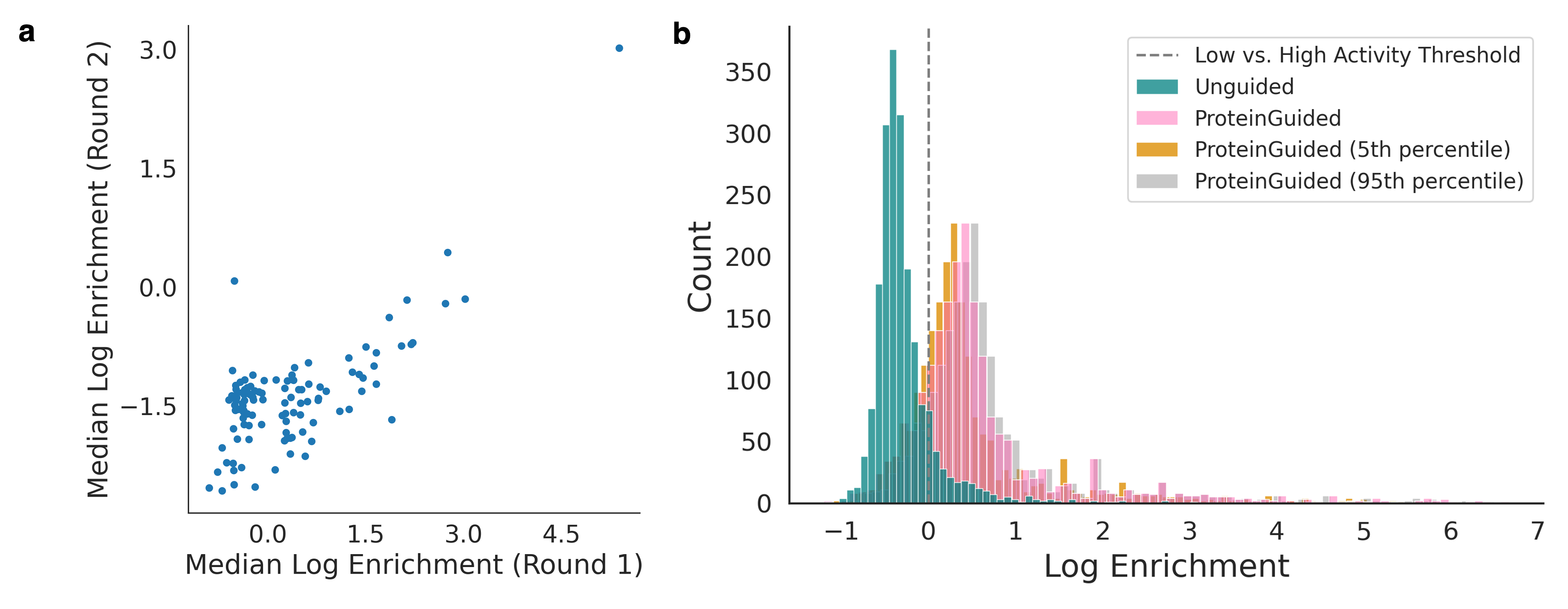}
  \caption[A]{
    \textbf{Normalizing log enrichment scores between round 1 and round 2.}
    \textbf{a})
    Log enrichment scores of the $112$ control sequences that were tested in both round 1 and round 2. The x-axis shows the median log enrichment scores for the replicates from round 1. The y-axis shows the median log enrichment scores for the replicates from round 2.    %
    \textbf{b})
    Histogram of log enrichment values comparing unconditional sequences (teal) and guided sequences (pink) after calibration between rounds. Bootstrap analysis (1,000 iterations) of the 112 control sequences yielded a mean median log enrichment of 0.56 for the guided library with a 90\% confidence interval of [0.38, 0.75]. The 5th and 95th percentile mappings (gold and gray histograms, respectively) represent conservative and optimistic transformations based on bootstrap sampling, visualizing the uncertainty range in our calibration procedure. The gray dashed line at a log enrichment value of 0 represents the low vs. high activity classification threshold. Our guidance strategy was designed to generate sequences with log enrichment values greater than this threshold.

  }
  \label{si_fig:tadA_mapping}
\end{figure*}

\begin{figure*}[h!]
    \centering
    \includegraphics[width=1.0\textwidth]{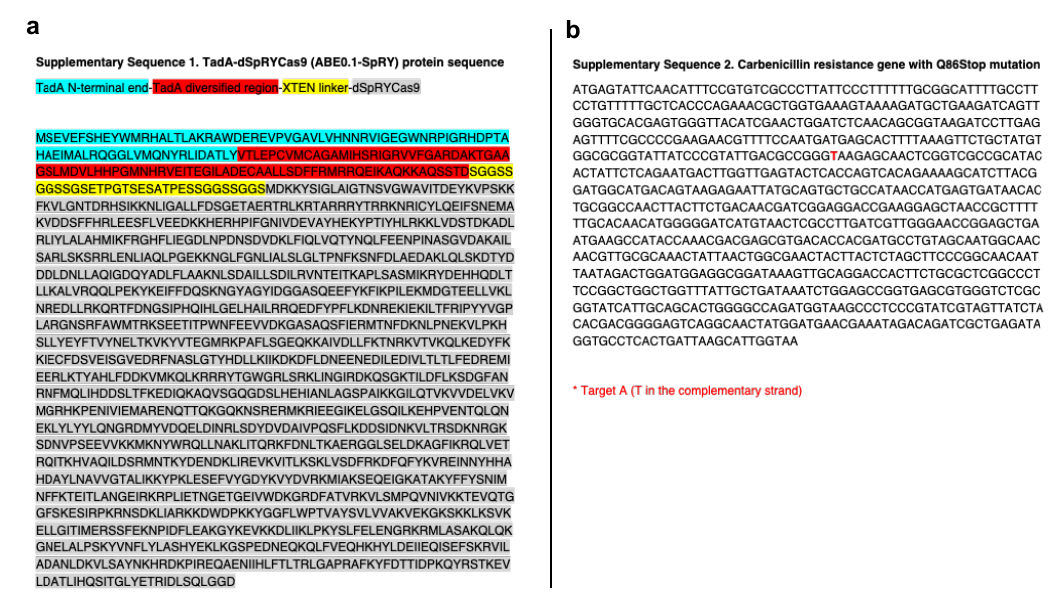}
    \caption[A]{
    \textbf{Supplementary sequences from the base editor experiment.}
    \textbf{a}) Base editor (starting variant) protein sequence.
    \textbf{b}) Reporter gene DNA sequence. 
    }
  \label{si_fig:supp_seq}
\end{figure*}

